\newif\ifabstract
\abstracttrue
 \abstractfalse %Uncommenting this line gives the full version
\newif\iffull
\ifabstract \fullfalse \else \fulltrue \fi
\documentclass[11pt]{article}
\usepackage{amsfonts}
\usepackage{amssymb,bm}
\usepackage{amstext}
\usepackage{wrapfig}
\usepackage{booktabs} % for professional tables
\usepackage{amsmath}
\usepackage{xspace}
\usepackage{theorem}
\usepackage{graphicx}
\usepackage{url}
\usepackage{graphics}
\usepackage{colordvi}
\usepackage{colordvi}
\usepackage{subfigure}
\usepackage{xr-hyper}
\usepackage{hyperref}

\advance \oddsidemargin by -0.8in
\newcommand{\myparskip}{3pt}
\parskip \myparskip
\newcommand{\sgn}{\text{sgn}}
\newcommand{\loss}{\hat{L}_{n}(\bm{\bm{\theta}}^{*};p)}
\newcommand{\lossc}{\hat{L}_{n}(\bm{\bm{\theta}}^{*})}
\newcommand{\error}{\hat{R}_{n}(\bm{\bm{\theta}}^{*})}
\newcommand{\thres}{\sigma_{\text{th}}}

\newtheorem{theorem}{Theorem}
\newtheorem{lemma}{Lemma}
\newtheorem{example}{Example}

\newtheorem{claim}{Claim}
\newtheorem{proposition}{Proposition}
\newtheorem{assumption}{Assumption}

\newcommand{\footremember}[2]{%
   \footnote{#2}
    \newcounter{#1}
    \setcounter{#1}{\value{footnote}}%
}
\newcommand{\footrecall}[1]{%
    \footnotemark[\value{#1}]%
} 

\newenvironment{proof}{\par \smallskip{\bf Proof:}}{\hfill\stopproof}
\def\stopproof{\square}
\def\square{\vbox{\hrule height.2pt\hbox{\vrule width.2pt height5pt \kern5pt
\vrule width.2pt} \hrule height.2pt}}

\begin{document}

\title{Understanding the Loss Surface of Neural Networks for Binary Classification}
\author{Shiyu Liang\footremember{uiuc}{University of Illinois at Urbana-Champaign}\\ sliang26@illinois.edu
            \and Ruoyu Sun\footrecall{uiuc}\\ ruoyus@illinois.edu
            \and Yixuan Li\footremember{cornell}{Facebook Research }\\yixuanl@fb.com
            \and R. Srikant\footrecall{uiuc}\\rsrikant@illinois.edu
            }

\date{}
\maketitle

\begin{abstract}
It is widely conjectured that the reason that training algorithms for neural networks are successful because all local minima lead to similar performance; for example, see \cite{lecun2015deep, choromanska2015loss,dauphin2014identifying}. Performance is typically measured in terms of two metrics: training performance and generalization performance. Here we focus on the training performance of neural networks for binary classification, and provide conditions under which the training error is zero at all local minima of appropriately chosen surrogate loss functions. Our conditions are roughly in the following form: the neurons have to be increasing and strictly convex, the neural network should either be single-layered or is multi-layered with a shortcut-like connection, and the surrogate loss function should be a smooth version of hinge loss. We also provide counterexamples to show that, when these conditions are relaxed, the result may not hold.
\end{abstract}

\section{Introduction}
Local search algorithms like stochastic gradient descent~\cite{bottou2010large} or variants have gained huge success in training deep neural networks 
%in various areas such as image recognition, speech recognition, natural language process and others 
(see,  \cite{krizhevsky2012imagenet}; \cite{goodfellow2013maxout}; \cite{wan2013regularization}, for example). 
%Deep neural networks typically converge to one of the many local minima -- the number of which can grow exponentially with the number of parameters~\cite{kawaguchi2016deep}.  
Despite the spurious saddle points and local minima on the loss surface~\cite{dauphin2014identifying}, it has been widely conjectured that all local minima of the empirical loss lead to similar training performance~\cite{lecun2015deep,choromanska2015loss}. For example, \cite{li2015convergent} empirically showed that neural networks with identical architectures but different initialization points can converge to local minima with similar classification performance.  However, it still remains a challenge to characterize the theoretical properties of the loss surface for neural networks.
%To explain this empirical success, 

In the setting of regression problems, theoretical justifications has been established to support the conjecture that all local minima  lead to similar training performance. For shallow models, \cite{andoni2014learning, sedghi2014provable, janzamin2015beating, haeffele2015global, gautier2016globally, brutzkus2017globally, soltanolkotabi2017learning, soudry2017exponentially, goel2017learning, du2017convolutional, zhong2017recovery, li2017convergence} provide conditions under which the local search algorithms are guaranteed to converge to the globally optimal solution for the regression problem. 
 %In regression problems where the outputs of neural network are required to exactly match the target output, many interesting theoretical results have been established \cite{andoni2014learning, sedghi2014provable, janzamin2015beating, haeffele2015global, gautier2016globally, brutzkus2017globally, soltanolkotabi2017learning, soudry2017exponentially, goel2017learning, du2017convolutional, zhong2017recovery, li2017convergence}, providing conditions under which the local search algorithms are guaranteed to converge to the globally optimal solution for the regression problem. 
For deep linear networks, it has been shown that every local minimum of the empirical loss is a global minimum~\cite{baldi1989neural, kawaguchi2016deep, freeman2016topology, hardt2016identity, yun2017global}. In order to characterize the loss surface of more general deep networks for regression tasks, \cite{choromanska2015loss} have proposed an interesting approach. Based on certain constructions on network models and additional assumptions, they relate the loss function to a spin glass model and show that the almost all local minima have similar empirical loss and the number of bad local minima decreases quickly with the distance to the global optimum. Despite the interesting results, it remains a concern to properly justify their assumptions. 
More recently, it has been shown ~\cite{nguyen2017loss1,nguyen2017loss2} that, when the dataset satisfies certain conditions, if one layer in the multilayer network has more neurons than the number of training samples, then a subset of local minima are global minima.   

Although the loss surfaces in regression tasks have been well studied, the theoretical understanding of loss surfaces in classification tasks is still limited. \cite{nguyen2017loss2,boob2017theoretical,soltanolkotabi2017theoretical} treat the classification problem as the regression problem by using quadratic loss, and show that (almost) all local minima are global minima. However, the global minimum of the quadratic loss does not necessarily have zero misclassification error even in the simplest cases (e.g., every global minimum of quadratic loss can have non-zero misclassification error even when the dataset is linearly separable and the network is a linear network). This issue was mentioned in \cite{nguyen2017loss1} and a different loss function was used, but their result only studied the linearly separable case and a subset of the critical points.
% but they either require a special architecture or only consider a subset of the critical points.

In view of the prior work, the context and contributions of our paper are as follows:

\begin{itemize}
\item Prior work on quadratic and related loss functions suggest that one can achieve zero misclassification error at all local minima by overparameterizing the neural network. The reason for over-parameterization is that the quadratic loss function tries to match the output of the neural network to the label of each training sample.
\item On the other hand, hinge loss-type functions only try to match the sign of the outputs with the labels. So it may be possible to achieve zero misclassification error without over-parametrization. We provide conditions under which the misclassification error of neural networks is zero at all local minima for hinge-loss functions. 
\item Our conditions are roughly in the following form: the neurons have to be increasing and strictly convex, the neural network should either be single-layered or is multi-layered with a shortcut-like connection and the surrogate loss function should be a smooth version of the hinge loss function.
\item We also provide counterexamples to show that when these conditions are relaxed, the result may not hold. 
\item We establish our results under the assumption that either the dataset is linearly separable or the positively and negatively labeled samples are located on different subspaces. Whether this assumption is necessary is an open problem, except in the case of certain special neurons.
\end{itemize}

The outline of this paper is as follows. In Section~\ref{sec::prelim}, we present the necessary definitions. In Section~\ref{sec::main-results}, we present the main results and we discuss each condition in Section~\ref{sec::discussions}. Conclusions are presented in Section~\ref{sec::conclusions}. All proofs are provided in Appendix.

\section{Preliminaries}\label{sec::prelim}
\textbf{Network models.} Given an input vector $x$ of dimension $d$, we consider a neural network with $L$ layers for binary classification. We denote by $M_{l}$  the number of neurons on the $l$-th layer (note that $M_{0}=d$ and $M_{L}=1$). We denote the neuron activation function by  $\sigma$. 
Let $\bm{W}_{l}\in\mathbb{R}^{M_{l-1}\times M_{l}}$ denote the weight matrix connecting the $(l-1)$-th layer and the $l$-th layer and $\bm{b}_{l}\in\mathbb{R}^{M_{l}}$ denote the bias vector for the neurons in the $l$-th layer. Therefore, the output of the  network $f:\mathbb{R}^{d}\rightarrow \mathbb{R}$ can be expressed by
\begin{equation*}
f(x;\bm{\theta})=\bm{W}^{\top}_{L}\sigma\left(...\sigma(\bm{W}^{\top}_{1}x+\bm{b}_{1})+\bm{b}_{L-1}\right)+\bm{b}_{L}, 
\end{equation*}
where $\bm{\theta}$ %=(\bm{W}_{1},...,\bm{W}_{L},\bm{b}_{1},...,\bm{b}_{L})$ 
denotes all parameters in the neural network. 
%In this paper, we allow the number of layers and the number of neurons in the multilayer neural network to be zero. In this case, the  the multilayer network outputs a constant equal to zero for any input.   

\textbf{Data distribution.} In this paper, we consider binary classification tasks where each sample $(\bm{X},Y)\in\mathbb{R}^{d}\times \{-1,1\}$ is drawn from an underlying data distribution $\mathbb{P}_{\bm{X}\times Y}$
defined on $\mathbb{R}^{d}\times \{-1,1\}$. The sample $(\bm{X},Y)$ is considered positive if $Y=1$, and negative otherwise. 
Let $\mathcal{E}=\{\bm{e}_{1},...,\bm{e}_{d}\}$ denote a set of orthonormal basis on the space $\mathbb{R}^{d}$. Let  $\mathcal{U}_{+}$ and $\mathcal{U}_{-}$ denote two subsets of $\mathcal{E}$ such that all positive and negative samples are located on the linear span of the set $\mathcal{U}_{+}$ and $\mathcal{U}_{-}$, respectively,   
i.e., $\mathbb{P}_{\bm{X}|Y}(\bm{X}\in\text{Span}(\mathcal{U}_{+})|Y=1)=1$ and $\mathbb{P}_{\bm{X}|Y}(\bm{X}\in\text{Span}(\mathcal{U}_{-})|Y=-1)=1$. Let $r$ denote the size of the set $\mathcal{U}_{+}\cup\mathcal{U}_{-}$, $r_{+}$ denote the size of the set $\mathcal{U}_{+}$ and $r_{-}$ denote the size of the set $\mathcal{U}_{-}$, respectively.

\textbf{Loss and error.} Let $\mathcal{D}=\{(x_{i},y_{i})\}_{i=1}^{n}$ denote a dataset with $n$ samples, each independently drawn from the distribution $\mathbb{P}_{\bm{X}\times Y}$. Given a neural network $f(x;\bm{\theta})$ parameterized by $\bm{\theta}$ and a loss function $\ell:\mathbb{R}\rightarrow\mathbb{R},$
in binary classification tasks\footnote{We note that, in regression tasks, the empirical loss is usually defined as $\hat{L}_{n}(\bm{\theta})=\frac{1}{n}\sum_{i=1}^{n}\ell(y_{i}-f(x_{i};\bm{\theta}))$.}, we define the \textbf{empirical loss} $\hat{L}_{n}(\bm{\theta})$ as the average loss of the network $f$ on a  sample in the dataset {$\mathcal{D}$, i.e.,}
$$
\hat{L}_{n}(\bm{\theta})=\frac{1}{n}\sum_{i=1}^{n}\ell(-y_{i}f(x_{i};\bm{\theta})).
$$
Furthermore, for a neural network $f$, we define a binary classifier $g_{f}:\mathbb{R}^{d}\rightarrow \{-1,1\}$ of the form $g_{f}=\sgn(f)$, where the sign function $\sgn(z)=1$, if $z\ge 0$, and $\sgn(z)=0$ otherwise. 
We define the \textbf{training error} (also called the \textbf{misclassification error}) $\hat{R}_{n}(\bm{\theta})$ as the misclassification rate of the neural network $f(x;\bm{\theta})$ on the dataset $\mathcal{D}$, i.e., 
\begin{equation*}
\hat{R}_{n}(\bm{\theta})=\frac{1}{n}\sum_{i=1}^{n}\mathbb{I}\{y_{i}\neq \sgn(f(x_{i};\bm{\theta}))\},
\end{equation*}
where $\mathbb{I}\{\cdot\}$ is the indicator function. 
%We define the \textbf{generalization error} ${R}(\bm{\theta})$ as the probability that  the neural network $f$  makes an incorrect prediction on a data sample $(\bm{X},Y)$ randomly drawn from the distribution $\mathbb{P}_{\bm{X}\times Y}$, i.e., 
%$$R(\bm{\theta})=\mathbb{P}_{\bm{X}\times Y}\left[Y\neq\sgn(f(\bm{X};\bm{\theta}))\right].$$
The training error $\hat{R}_{n}$ measures the classification performance of the network $f$ on the finite samples in the dataset $\mathcal{D}$.
%, while the generalization error $R$ measures the classification performance of the network on the overall data distribution $\mathbb{P}_{\bm{X}\times Y}$.

\section{Main Results}\label{sec::main-results}

In this section, we present the main results. We first introduce several important conditions in order to derive the main results, and we will provide further discussions on these conditions in the next section. 

\subsection{Conditions}
To fully specify the problem, we need to specify our assumptions on several components of the model, including: (1) the loss function, (2) the data distribution, (3) the network architecture and (4) the neuron activation function.  
%\textbf{Loss function.} In this paper, our main results can be applied to the following class of loss function. 
\begin{assumption}[Loss function] \label{assump::loss}
Let $\ell_{p}:\mathbb{R}\rightarrow \mathbb{R}$ denote a loss function satisfying the following conditions: (1) $\ell_{p}$ is a surrogate loss function, i.e., $\ell_{p}(z)\ge \mathbb{I}\{z\ge0\}$ for all $z\in\mathbb{R}$, where $\mathbb{I}(\cdot)$ denotes the indicator function; (2) $\ell_{p}$ has continuous derivatives up to order $p$ on $\mathbb{R}$; (3) $\ell_{p}$ is non-decreasing (i.e., $\ell'_{p}(z)\ge 0$ for all $z\in\mathbb{R}$) and there exists  a positive constant $z_{0}$ such that $\ell'_{p}(z)=0$ iff $z\le -z_{0}$. 
\end{assumption}
%\cite{} definition
The first condition in Assumption~\ref{assump::loss} ensures that the training error $\hat{R}_{n}$ is always upper bounded by the empirical loss $\hat{L}_{n}$, i.e., $\hat{R}_{n}\le \hat{L}_{n}$. This guarantees that the neural network can correctly classify all samples in the dataset (i.e., $\hat{R}_{n}=0$), when the neural network achieves zero empirical loss (i.e., $\hat{L}_{n}=0$).
The second condition ensures that the empirical loss $\hat{L}_{n}$ has continuous derivatives with respect to the parameters up to a sufficiently high order. The third condition ensures that the loss function is non-decreasing and $\ell'_{p}(z)=0$ is achievable if and only if $z\le -z_{0}$. Here, we provide a simple example of the loss function satisfying all conditions in Assumption~\ref{assump::loss}: the polynomial hinge loss, i.e., $\ell_{p}(z)=[\max\{z+1,0\}]^{p+1}$. We note that, in this paper, we use $\hat{L}_{n}(\bm{\theta};p)$ to denote the empirical loss when the loss function is $\ell_{p}$ and the network is parametrized by a set of parameters $\bm{\theta}$. Further results on the impact of loss functions are presented in Section~\ref{sec::discussions}.

%\textbf{Data distribution.} In this paper, we focus on a class of data distribution $\mathbb{P}_{\bm{X}\times Y}$ satisfying the following two assumptions. 
\begin{assumption}[Data distribution] \label{assump::full-rank}
Assume that for random vectors $\bm{X}_{1},...,\bm{X}_{r_{+}}$ independently drawn from the distribution $\mathbb{P}_{\bm{X}|Y=1}$ and $\bm{Z}_{1},...,\bm{Z}_{r_{-}}$ independently drawn from the distribution $\mathbb{P}_{\bm{X}|Y=-1}$, matrices $\left(\bm{X}_{1},...,\bm{X}_{r_{+}}\right)\in\mathbb{R}^{r_{+}\times d}$ and $\left(\bm{Z}_{1},...,\bm{Z}_{r_{-}}\right)\in\mathbb{R}^{r_{-}\times d}$ are full rank matrices with probability one. 
\end{assumption}
Assumption ~\ref{assump::full-rank} states that support of the conditional distribution $\mathbb{P}_{\bm{X}|Y=1}$ is sufficiently rich so that $r_+$ samples drawn from it will be linearly independent. In other words, by stating this assumption, we are avoiding trivial cases where all the positively labeled points are located in a very small subset of the linear span of $\mathcal{U}_+.$ Similarly for the negatively labeled samples.
\begin{assumption}[Data distribution] \label{assump::different-subspaces}
Assume $|\mathcal{U}_{+}\cup\mathcal{U}_{-}|>\max\{|\mathcal{U}_{+}|,|\mathcal{U}_{-}|\}$, i.e., $r>\max\{r_{+},r_{-}\}$.
\end{assumption}

Assumption~\ref{assump::different-subspaces} assumes that the positive and negative samples are not located on the same linear subspace. Previous works \cite{belhumeur1997eigenfaces, chennubhotla2001sparse, cootes2001active,belhumeur1997eigenfaces} have observed that some classes of natural images (e.g., images of faces, handwritten digits, etc) can be reconstructed from lower-dimensional representations. For example, using dimensionality reduction methods such as PCA, one can approximately reconstruct the original image from only a small number of principal components~\cite{belhumeur1997eigenfaces, chennubhotla2001sparse}. %Based on this observation, a class of techniques (e.g., PCA, CCA, etc) are proposed to generate natural images, where the images are generated from a linear subspace model added with some random noises. 
Here, Assumption~\ref{assump::different-subspaces} states that both the positively and negatively labeled samples have lower-dimensional representations, and they do not exist in the same lower-dimensional subspace. We provide additional analysis in Section~\ref{sec::discussions}, showing how our main results generalize to other data distributions. 
\begin{wrapfigure}{R}{0.45\linewidth}
\centering
\includegraphics[width=1\linewidth]{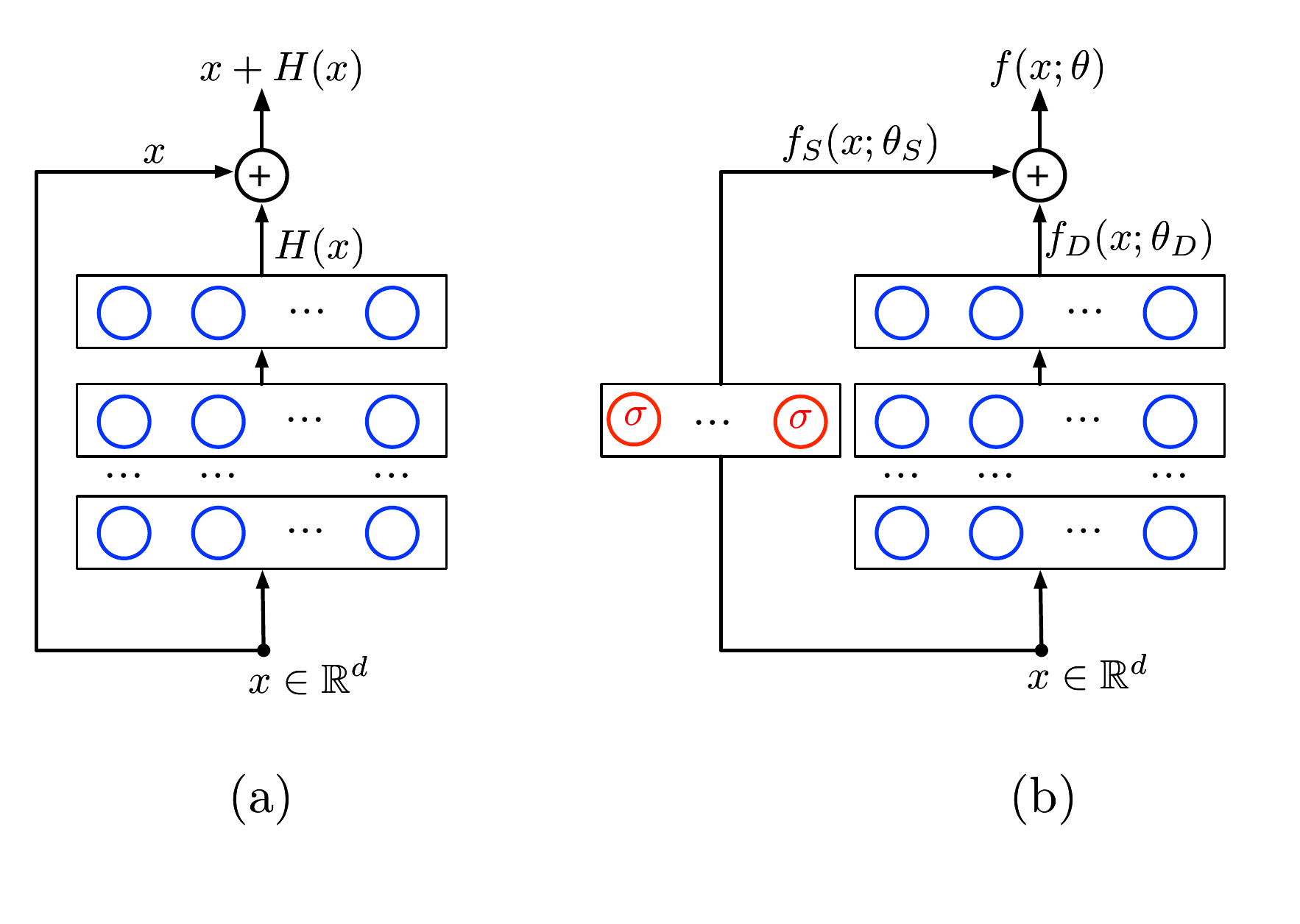}
\caption{(a) The identity shortcut connection adopted in the residual network~\cite{he2016deep}. (b) The shortcut-like connection adopted in this paper. }
\label{fig::network}
\end{wrapfigure}
%Furthermore, we note here that the assumption~\ref{assump::different-subspaces}, together with assumption~\ref{assump::full-rank}, shows that both positive and negative samples can only locate on the subspace Span$(\mathcal{U}_{+}\cap \mathcal{U}_{-})$ with probability zero.
\begin{assumption}[Network architecture]\label{assump::shortcut-connection}
Assume that the  neural network $f$ is a single-layered neural network, or more generally, 
 has shortcut-like connections shown in Fig~\ref{fig::network}~(b), where $f_{S}$ is a single layer network and $f_{D}$ is a feedforward network. 
\end{assumption}
 
Shortcut connections are widely used in the modern network architectures (e.g., Highway Networks~\cite{srivastava2015highway}, ResNet~\cite{he2016deep}, DenseNet~\cite{huang2017densely}, etc.), where the skip connections allow the deep layers to have direct access to the outputs of shallow layers.  For instance, in the residual network, each residual block has a identity shortcut connection, shown in Fig~\ref{fig::network}~(a), where the output of each residual block is the vector sum of its input and the output of a network $H$. 

Instead of using the identity shortcut connection, in this paper, we first pass the input through a single layer network $f_{S}(x;\bm{\theta}_{S})=a_{0}+\bm{a}^{\top}\sigma\left(\bm{W}^{\top}x\right)$, where  vector $\bm{a}$ denotes the weight vector, matrix $\bm{W}$ denotes the weight matrix and vector $\bm{\theta}_{S}$ denotes the vector containing all parameters in $f_{S}$. We next add the output of this network  to a network $f_{D}$ and use the addition as the output of the whole network, i.e., 
$f(x;\bm{\theta})=f_{S}(x;\bm{\theta}_{S})+f_{D}(x;\bm{\theta}_{D}),$
where vector  $\bm{\theta}_{D}$ and $\bm{\theta}$ denote the vector containing all parameters in the network $f_{D}$ and the whole network $f$, respectively.
We note here that, in this paper, we do not restrict the number of layers and neurons in the network $f_{D}$ and this means that the network $f_{D}$ can be a feedforward network introduced in Section~\ref{sec::prelim} or a single layer network or even a constant. In fact, when the network $f_{D}$ is a single layer network or a constant, the whole network $f$ becomes a single layer network. 
Furthermore, we note that, in Section~\ref{sec::discussions}, we will show  that if we remove this connection or replace this shortcut-like connection with the identity shortcut connection, the main result does not hold.

%\textbf{Neurons.} Finally, we introduce the neurons. We make the following assumptions on the types of neurons in the single layer network $f_{S}$ and the network $f_{D}$, respectively. 

\begin{assumption}[Neuron activation] \label{assump::neurons}
Assume that neurons $\sigma(z)$ in the network $f_{S}$ are real analytic and  satisfy $\sigma''(z)>0$ for all $z\in\mathbb{R}$. Assume that neurons in the network $f_{D}$ are real functions on $\mathbb{R}$.

\end{assumption}
In Assumption~\ref{assump::neurons}, we  assume that neurons in the network $f_{S}$ are infinitely differentiable and have  positive second order derivatives on $\mathbb{R}$, while neurons in the network $f_{D}$ are real functions. 
We make the above assumptions to ensure that the loss function $\hat{L}_{n}(\bm{\theta}_{S},\bm{\theta}_{D};p)$ is partially differentiable w.r.t. the parameters $\bm{\theta}_{S}$ in the network $f_{S}$ up to a sufficiently high order and allow us to use Taylor expansion  in the analysis. Here, we list a few neurons which can be used in the network $f_{S}$:  softplus neuron, i.e., $\sigma(z)=\log_{2}(1+e^{z})$, quadratic neuron, i.e, $\sigma(z)=z^{2}$, etc.
 We note that neurons in the network $f_{S}$ and $f_{D}$ do not need to be of the same type and this means that a more general class of neurons can be used in the network $f_{D}$, e.g., threshold neuron, i.e., $\sigma(z)=\mathbb{I}\{z\ge 0\}$, rectified linear unit $\sigma(z)=\max\{z,0\}$, sigmoid neuron $\sigma(z)=\frac{1}{1+e^{-z}}$, etc. Further discussion on the effects of neurons on the main results are provided in Section~\ref{sec::discussions}.

\subsection{Main Results}
%\ref{assump::loss}, \ref{assump::full-rank}, \ref{assump::different-subspaces} and~\ref{assump::neurons}
Now we present the following theorem to show that when assumptions~\ref{assump::loss}-\ref{assump::neurons} are satisfied, every local minimum of the empirical loss function has zero training error  if the number of neurons in  the network $f_{S}$ are chosen appropriately.  

\begin{theorem}[Linear subspace data]\label{thm::convex-finite-deep}
Suppose that assumptions \ref{assump::loss}-\ref{assump::neurons} are satisfied.
%Assume that the loss function $\ell_{p}$ satisfies  assumption~\ref{assump::loss}, the distribution $\mathbb{P}_{\bm{X}\times Y}$ satisfies assumption~\ref{assump::full-rank} and \ref{assump::different-subspaces} and neurons in the network satisfy assumption~\ref{assump::neurons}.
Assume that  samples in the dataset $\mathcal{D}=\{(x_{i},y_{i})\}_{i=1}^{n}, n\ge 1$ are independently drawn from the distribution $\mathbb{P}_{\bm{X}\times Y}$.  Assume that the number of neurons $M$ in the network $f_{S}$ satisfies $M\ge 2\max\{\frac{n}{\Delta r},r_{+},r_{-}\}$, where $\Delta r=r-\max\{r_{+},r_{-}\}$. If $\bm{\theta}^{*}=(\bm{\theta}_{S}^{*},\bm{\theta}_{D}^{*})$ is a local minimum of the loss function $\hat{L}_{n}(\bm{\theta}_{S},\bm{\theta}_{D};p)$ and $p\ge 6$, then  $\hat{R}_{n}(\bm{\theta}^{*}_{S},\bm{\theta}^{*}_{D})=0$ holds with probability one.
\end{theorem}

\textbf{Remark:} (i) By setting the network $f_{D}$ to a constant, it directly follows from Theorem~\ref{thm::convex-finite-deep} that if a single layer network $f_{S}(x;\bm{\theta}_{S})$ consisting of neurons satisfying Assumption~\ref{assump::neurons} and all other conditions in Theorem~\ref{thm::convex-finite-deep}  are satisfied, then every local minimum of the empirical loss $\hat{L}_{n}(\bm{\theta}_{S};p)$ has zero training error.  (ii) The positiveness of $\Delta r$ is guaranteed by Assumption~\ref{assump::different-subspaces}. In the worst case (e.g., $\Delta r=1$ and $\Delta r=2$), the number of neurons needs to be at least greater than the number of samples, i.e., $M\ge n$. However, when the two orthonormal basis sets $\mathcal{U}_{+}$ and $\mathcal{U}_{-}$ differ significantly (i.e., $\Delta r\gg1$), the number of neurons required by Theorem~\ref{thm::convex-finite-deep} can be significantly smaller than the number of samples (i.e., $n\gg {2n}/{\Delta r}$). In fact, we can show that, when the neuron has quadratic activation function $\sigma(z)=z^{2}$, the assumption $M\ge 2n/\Delta r$ can be further relaxed such that the number of neurons is independent of the number of samples. We discuss this in  the following proposition. 

%In Theorem~\ref{thm::convex-finite-deep}, we assume that the $f_{D}$ is a real function on $\mathbb{R}^{d}$. This indicates that we do not make assumptions on the depth, size and type of neurons in the network $f_{D}$ but we require that $f_{D}(x;\bm{\theta})$ needs to be a real function on $\mathbb{R}^{d}$ under all parameters $\bm{\theta}_{D}$. In other words, the network $f_{D}$ can be a constant or a single layer network or a multilayer network. In Theorem~\ref{thm::convex-finite-deep}, we further assume that the activation function $\sigma$ is real analytic and has a positive second order derivative on $\mathbb{R}$. We will show in the next section that if we replace with some other types of neurons (e.g., ReLUs, Leakly ReLUs, sigmoid, etc), Theorem~\ref{thm::convex-finite-deep}  does not hold. In addition, in Theorem~\ref{thm::convex-finite-deep}, we assume that the number of neurons $M$ is in scale with the number of neurons. This does not means we require the model to be over-specified, i.e., $M\ge n$. Specifically, when the subspaces where positive and negative samples differ significantly (i.e., $\Delta r$ is large), then the number of neurons required in the single layer network $f_{S}$, i.e., $2n/\Delta r$ can be far smaller than the number of samples. In fact, when the neuron is the quadratic neuron, i.e., $\sigma(z)=z^{2}$, the number of neurons does not rely on the number of samples.   

\begin{proposition}\label{prop::results-quadratic}
Assume that assumptions~\ref{assump::loss}-\ref{assump::neurons} are satisfied. 
%Assume that the loss function $\ell_{p}$ satisfies  assumption~\ref{assump::loss}, the distribution $\mathbb{P}_{\bm{X}\times Y}$ satisfies assumption~\ref{assump::full-rank} and \ref{assump::different-subspaces} and neurons in the network satisfy assumption~\ref{assump::neurons}.
Assume that  samples in the dataset $\mathcal{D}=\{(x_{i},y_{i})\}_{i=1}^{n}, n\ge 1$ are independently drawn from the distribution $\mathbb{P}_{\bm{X}\times Y}$. Assume that neurons in the network $f_{S}$ satisfy $\sigma(z)=z^{2}$ and the number of neurons in the network $f_{S}$ satisfies $M>r$. If $\bm{\theta}^{*}=(\bm{\theta}_{S}^{*},\bm{\theta}_{D}^{*})$ is a local minimum of the loss function $\hat{L}_{n}(\bm{\theta}_{S},\bm{\theta}_{D};p)$ and $p\ge 6$, then $\hat{R}_{n}(\bm{\theta}^{*}_{S},\bm{\theta}_{D})=0$ holds with probability one.
\end{proposition}

\textbf{Remark:} Proposition~\ref{prop::results-quadratic} shows that if the number of neuron $M$ is greater than the dimension of the subspace, i.e., $M>r$, then every local minimum of the empirical loss function has zero training error. We note here that although the result is stronger with  quadratic neurons, it does not imply that the quadratic neuron has advantages over the other types of neurons (e.g., softplus neuron, etc). This is due to the fact that when the neuron has positive derivatives on $\mathbb{R}$, the result in Theorem~\ref{thm::convex-finite-deep} holds for the dataset where positive and negative samples are linearly separable. We provide the formal statement of this result in Theorem~\ref{thm::linear-sep-deep}. However, when the neuron has quadratic activation function, the result in Theorem~\ref{thm::convex-finite-deep} may not hold for linearly separable dataset and we will illustrate this by providing a counterexample in the next section. 
%As we will show later, every local minimum of the empirical loss has zero training error, if the dataset is linear separable and the neuron 

%by using the softplus neuron, the results in Theorem~\ref{thm::convex-finite-deep} can be generalized to the linearly separable distributions, while the results do not hold for the quadratic neuron.   
As shown in Theorem~\ref{thm::convex-finite-deep}, when the data distribution satisfies Assumption~\ref{assump::full-rank} and~\ref{assump::different-subspaces}, every local minimum of the empirical loss has zero training error. However, we can easily see that distributions satisfying these two assumptions may not be linearly separable. Therefore, to provide a complementary result to Theorem~\ref{thm::convex-finite-deep},  we consider the case where the data distribution is linearly separable.
Before presenting the result, we first present the following assumption on the data distribution.  
\begin{assumption}[Linear separability]\label{assump::linear-sep}
Assume that there exists a vector $\bm{w}\in\mathbb{R}^{d}$ such that  the data distribution satisfies {$\mathbb{P}_{\bm{X}\times Y}(Y\bm{w}^{\top}X>0)=1$}.
\end{assumption}
In Theorem~\ref{thm::linear-sep-deep}, we will show that when the samples drawn from the data distribution are linearly separable, and the network has a shortcut-like connection shown in Figure~\ref{fig::network}, all local minima of the empirical loss function have zero training errors if the type of the neuron in  the network $f_{S}$ are chosen appropriately.

\begin{theorem}[Linearly separable data]\label{thm::linear-sep-deep}
Suppose that the loss function $\ell_{p}$ satisfies  Assumption~\ref{assump::loss} and the network architecture satisfies Assumption~\ref{assump::shortcut-connection}.
 Assume that  samples in the dataset $\mathcal{D}=\{(x_{i},y_{i})\}_{i=1}^{n}, n\ge 1$ are independently drawn from a distribution satisfying Assumption~\ref{assump::linear-sep}. Assume that the single layer network $f_{S}$ has $M\ge1$ neurons and neurons $\sigma$ in the network $f_{S}$ are twice differentiable and satisfy $\sigma'(z)>0$ for all $z\in\mathbb{R}$. If $\bm{\theta}^{*}=(\bm{\theta}^{*}_{S},\bm{\theta}^{*}_{D})$ is a local minimum of the loss function $\hat{L}_{n}(\bm{\theta}^{}_{S},\bm{\theta}^{}_{D};p)$, $p\ge 3$, then  $\hat{R}_{n}(\bm{\theta}^{*}_{S},\bm{\theta}^{*}_{D})=0$ holds with probability one.
\end{theorem}

\textbf{Remark:} Similar to Proposition~\ref{prop::results-quadratic}, Theorem~\ref{thm::linear-sep-deep} does not require the number of neurons to be in scale with the number of samples. In fact, we make a   weaker assumption here: the single layer network $f_{S}$ only needs to have at least one neuron, in contrast to at least $r$ neurons required by Proposition~\ref{prop::results-quadratic}. Furthermore, we note here that, in Theorem~\ref{thm::linear-sep-deep}, we assume that neurons in the network $f_{S}$  have positive derivatives on $\mathbb{R}$. This implies that Theorem~\ref{thm::linear-sep-deep} may not hold for a subset of neurons considered in Theorem~\ref{thm::convex-finite-deep} (e.g., quadratic neuron, etc). We will provide further discussions on the effects of neurons in the next section.  

So far, we have provided results showing that under certain constraints on the (1) neuron activation function, (2) network architecture, (3) loss function and (4) data distribution, every local minimum of the empirical loss function has zero training error. In the next section, we will discuss the implications of these conditions on our main results.

\section{Discussions}\label{sec::discussions}
In this section, we discuss the effects of the (1) neuron activation, (2) shortcut-like connections, (3) loss function and (4) data distribution on the main results, respectively. We show that the result may not hold if these assumptions are relaxed.   
\subsection{Neuron Activations}\label{sec::neuron}
To begin with, we discuss whether the results in Theorem~\ref{thm::convex-finite-deep} and~\ref{thm::linear-sep-deep} still hold if we vary the neuron activation function in the single layer network $f_{S}$. Specifically, we consider the following five classes of neurons: (1) softplus class, (2) rectified linear unit (ReLU) class, (3) leaky rectified linear unit (Leaky ReLU) class, (4) quadratic class and (5) sigmoid class. In the following, for each class of neurons, we show whether the main results hold and provide counterexamples if certain conditions in the main results are violated. We summarize our findings in Table~\ref{tab:neuron-family}. We visualize some neurons activation functions from these five classes in Fig.~\ref{fig::loss}(a). 

\begin{figure}[t]
\centering
\includegraphics[width=0.75\linewidth]{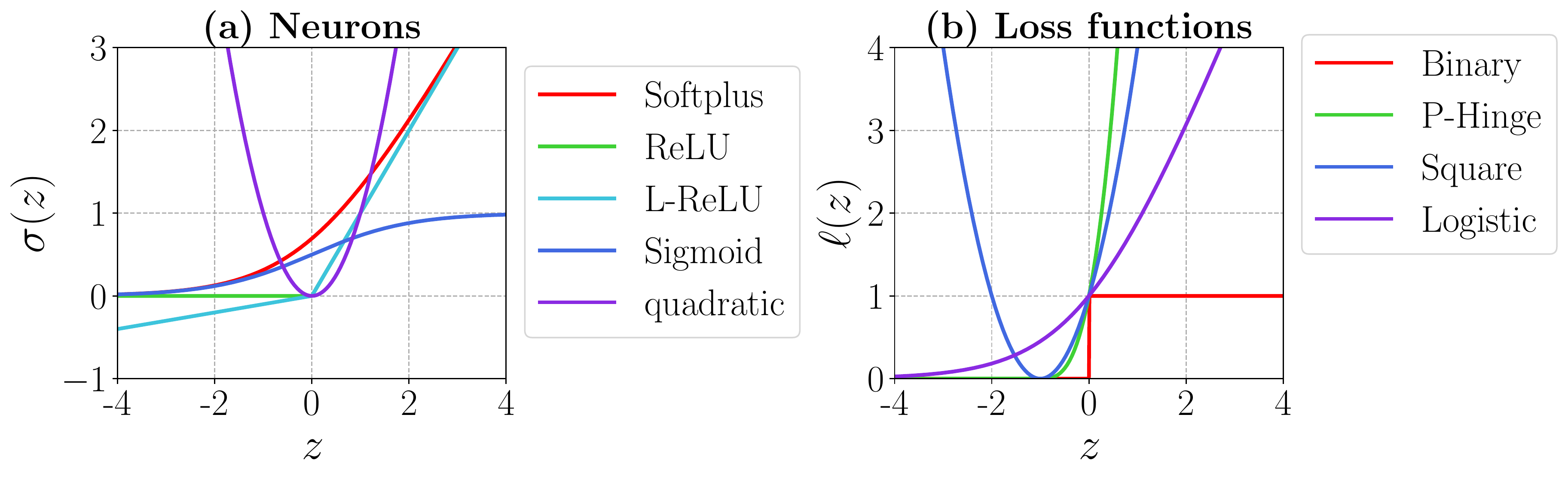}
\caption{(a) Five types of neuron activations, including  softplus neuron, ReLU, Leaky-ReLU, sigmoid neuron, quadratic neuron. (b) Four types of surrogate loss functions, including binary loss (i.e., $\ell(z)=\mathbb{I}\{z\ge 0\}$), polynomial hinge loss (i.e., $\ell(z)=[\max\{z+1,0\}]^{p+1}$), square loss (i.e., $\ell(z)=(1+z)^2$) and logistic loss (i.e., $\ell(z)=\log_{2}(1+e^z)$). Definitions of all neurons can be found in Section~\ref{sec::neuron}. }
\label{fig::loss}
\end{figure}

\textbf{Softplus class} contains neurons with real analytic activation functions $\sigma$, where $\sigma'(z)>0$, $\sigma''(z)>0$ for all $z\in\mathbb{R}$. A widely used neuron in this class is the softplus neuron, i.e., $\sigma(z)=\log_{2}(1+e^{z})$, which is a smooth approximation of ReLU. We can see that neurons in this class satisfy assumptions in both Theorem~\ref{thm::convex-finite-deep} and~\ref{thm::linear-sep-deep} and this  indicates that both theorems hold for the neurons in this class.

\textbf{ReLU class} contains neurons with  $\sigma(z)=0$ for all $z\le 0$ and $\sigma(z)$ is piece-wise continuous on $\mathbb{R}$. Some commonly adopted neurons in this class include: threshold units, i.e., $\mathbb{I}\{z\ge0\}$, rectified linear units (ReLU), i.e., $\max\{z,0\}$ and rectified quadratic units (ReQU), i.e., $\left[\max\{z, 0\}\right]^{2}$.  We can see that neurons in this class do not satisfy neither assumptions in Theorem~\ref{thm::convex-finite-deep} nor~\ref{thm::linear-sep-deep}.
In proposition~\ref{prop::relus}, we show that when the single layer network $f_{S}$ consists of neurons in the ReLU class, even if all other conditions in Theorem~\ref{thm::convex-finite-deep} or~\ref{thm::linear-sep-deep} are satisfied, the empirical loss function can have a local minimum with non-zero training error.

\begin{proposition}\label{prop::relus}
Suppose that assumptions~\ref{assump::loss} and~\ref{assump::shortcut-connection} are satisfed.
%Assume that the network $f_{S}(x;\bm{\theta}_{S})$ has $M\ge 1$ neurons.  
Assume that neurons in the network $f_{S}$ satisfy that $\sigma(z)=0$ for all $z\le 0$ and $\sigma(z)$ is piece-wise continuous on $\mathbb{R}$. Then there exists a  network architecture $f_{D}$ and a distribution satisfying assumptions in Theorem~\ref{thm::convex-finite-deep} or~\ref{thm::linear-sep-deep} such that with probability one, the empirical loss $\hat{L}_{n}(\bm{\theta};p),p\ge2$ has a local minima $\bm{\theta}^{*}=(\bm{\theta}_{S}^{*},\bm{\theta}^{*}_{D})$ 
satisfying $\hat{R}_{n}(\bm{\theta}^{*})\ge\frac{\min\{n_{+},n_{-}\}}{n}$, where $n_{+}$ and $n_{-}$ are the number of positive and negative samples, respectively. 
\end{proposition}
% if 
\textbf{Remark:} 
(i) We note here that the above result holds in the over-parametrized case, where the number of neurons in the network $f_{S}$ is larger than the number of samples in the dataset. In addition, all counterexamples shown in Section~\ref{sec::neuron} hold in the over-parametrized case.  
(ii) We note here that applying the same analysis, we can generalize the above result to a larger class of neurons satisfying the following condition: there exists a scalar $z_{1}$ such that $\sigma(z)=$ constant for all $z\le z_{1}$ and $\sigma(z)$ is  piece-wise continuous on $\mathbb{R}$. (iii) We note that the training error is strictly non-zero when the dataset has both positive and negative samples and this can happen with probability at least $1-e^{-\Omega(n)}$.

\begin{table}[t]\label{tab:neuron-family}
\centering
\small
\begin{tabular}{cccccc}
\toprule
{\bf Theorem} & {\bf Softplus} & {\bf ReLU} 	& {\bf Leaky-ReLU} 	& {\bf Sigmoid} 	& {\bf Quadratic}\\
\midrule
{\bf \ref{thm::convex-finite-deep}}& Yes & No & No & No & Yes\\
{\bf \ref{thm::linear-sep-deep}} & Yes & No  &No & No & No\\
\bottomrule
\end{tabular}
\caption{\small The result whether Theorem~\ref{thm::convex-finite-deep} or~\ref{thm::linear-sep-deep} hold for all neurons in each class. The definition of each class can be found in Section~\ref{sec::neuron}.}
\end{table}

\textbf{Leaky-ReLU class}  contains neurons with $\sigma(z)=z$ for all $z\ge 0$ and $\sigma(z)$ is piece-wise continuous on $\mathbb{R}$. Some commonly used neurons in this class include ReLU, i.e., $\max\{z, 0\}$, leaky rectified linear unit (Leaky-ReLU), i.e., $\sigma(z)=z$ for $z\ge 0$, $\sigma = \alpha z$ for $z\le 0$ and some constant $\alpha\in(0,1)$, exponential linear unit (ELU), i.e., $\sigma(z)=z$ for $z\ge 0$, $\sigma(z)=\alpha(\exp(z)-1)$ for $z\le 0$ and some constant $\alpha<0$. We can see that all neurons in this class do not satisfy assumptions in Theorem~\ref{thm::convex-finite-deep}, while some neurons in this class satisfy the condition in Theorem~\ref{thm::linear-sep-deep} (e.g., linear neuron, $\sigma(z)=z$) and some neurons do not (e.g., ReLU). In Proposition~\ref{prop::relus}, we have provided a counterexample showing that Theorem~\ref{thm::linear-sep-deep} does not hold for some neurons in this class (e.g., ReLU).
Next, we will present the following proposition to show that when the network $f_{S}$ consists of neurons in the Leaky-ReLU class, even if all other conditions in Theorem~\ref{thm::convex-finite-deep} are satisfied, the empirical loss function is likely to have a local minimum with non-zero training error with high probability.  

\begin{proposition}\label{prop::piecewise}
%Assume that the loss function $\ell_{p}$ satisfies  assumption~\ref{assump::loss}.
%Assume that the network $f_{S}(x;\bm{\theta}_{S})$ has $M\ge 1$ neurons.  
Suppose that Assumption~\ref{assump::loss} and~\ref{assump::shortcut-connection} are satisfied.
Assume that neurons in the network $f_{S}$ satisfy that $\sigma(z)=z$ for all $z\ge 0$ and $\sigma(z)$ is piece-wise continuous on $\mathbb{R}$. Then there exists a network architecture $f_{D}$ and a distribution satisfying assumptions in Theorem~\ref{thm::convex-finite-deep} such that, with probability at least $1-e^{-\Omega(n)}$, the empirical loss $\hat{L}_{n}(\bm{\theta};p),p\ge 2$ has  a local minima $\bm{\theta}^{*}=(\bm{\theta}_{S}^{*},\bm{\theta}^{*}_{D})$ with non-zero training error.
\end{proposition}
\textbf{Remark:}  We note that applying the same proof, we can generalize the above result to  a larger class of neurons, i.e., neurons satisfying the condition that there exists two scalars $z_{1}$ and $\alpha$ such that $\sigma(z)=\alpha(z-z_{1})$ for all $z\ge 0$ and $\sigma$ is piece-wise continuous on $\mathbb{R}$. In addition, we note that the ReLU neuron (but not all neurons in the ReLU class) satisfies the definition of both  ReLU class and Leaky-ReLU class, and therefore both Proposition~\ref{prop::relus} and~\ref{prop::piecewise} hold for the ReLU neuron.

%Proposition~\ref{prop::piecewise} shows that if the positive and negative samples are not linearly separable and the multilayer neural network is consisted of neurons in the Leaky-ReLU family, then the empirical loss function always has a local minima with non-zero training error. As we have shown earlier, it is likely that the data samples drawn from the  distribution considered in Theorem~\ref{thm::convex-finite-deep} are not linearly separable. Therefore, combining Theorem~\ref{thm::convex-finite-deep} and Proposition~\ref{prop::piecewise}, we can see that when the single layer network and multilayer network are consisted of neurons in the Leaky-ReLU family, even if all other conditions are satisfied, the empirical loss function is likely to have a local minima with a non-zero training error. This means that Theorem~\ref{thm::convex-finite-deep} does not hold for the neurons in the Leaky-ReLU family. 

\textbf{Sigmoid class} contains neurons with $\sigma(z)+\sigma(-z)\equiv$ constant on $\mathbb{R}$. We list a few commonly adopted neurons in this family: sigmoid neuron, i.e.,  $\sigma(z)=\frac{1}{1+e^{-z}}$, hyperbolic tangent neuron, i.e., $\sigma(z)=\frac{e^{z}-1}{e^{z}+1}$, arctangent neuron, i.e., $\sigma(z)=\tan^{-1}(z)$ and softsign neuron, i.e., $\sigma(z)=\frac{z}{1+|z|}$. We note that all real odd functions\footnote{A real function $f:\mathbb{R}\rightarrow\mathbb{R}$ is an odd function, if $f(x)+f(-x)\equiv0$ for all $x\in\mathbb{R}$.} satisfy the conditions of the sigmoid class.
We can see that none of the above neurons satisfy assumptions in Theorem~\ref{thm::convex-finite-deep}, since neurons in this class satisfy either $\sigma''(z)+\sigma''(-z)\equiv0$ for all $z\in\mathbb{R}$ or $\sigma(z)$ is not twice differentiable.
 For Theorem~\ref{thm::linear-sep-deep}, we can see that some neurons in this class satisfy the condition in Theorem~\ref{thm::linear-sep-deep} (e.g., sigmoid neuron) and some neurons do not (e.g., constant neuron $\sigma(z)\equiv0$ for all $z\in\mathbb{R}$).
In Proposition~\ref{prop::relus}, we provided a counterexample showing that Theorem~\ref{thm::linear-sep-deep} does not hold for some neurons in this class (e.g., constant neuron).
Next, we present the following proposition showing that when the network $f_{S}$ consists of neurons in the sigmoid class, then there always exists a data distribution satisfying the assumptions in Theorem~\ref{thm::convex-finite-deep} such that, with a positive probability, the empirical loss has a local minima with non-zero training error.

\begin{proposition}\label{prop::logistic}
%Assume that the loss function $\ell_{p}$ satisfies  assumption~\ref{assump::loss}.
%Assume that the network $f_{S}(x;\bm{\theta}_{S})$ has $M\ge 1$ neurons.  
Suppose that assumptions~\ref{assump::loss} and~\ref{assump::shortcut-connection} are satisfed.
Assume that  there exists a constant $c\in\mathbb{R}$ such that neurons in the network $f_{S}$ satisfy $\sigma(z)+\sigma(-z)\equiv c$ for all $z\in\mathbb{R}$. Assume that the dataset $\mathcal{D}$ has $2n$ samples. There exists a network architecture $f_{D}$ and a distribution satisfying assumptions in Theorem~\ref{thm::convex-finite-deep} such that, with a positive probability, the empirical loss function $\hat{L}_{2n}(\bm{\theta};p), p\ge2$ has a local minimum $\bm{\theta}^{*}=(\bm{\theta}_{S}^{*},\bm{\theta}^{*}_{D})$ satisfying $\hat{R}_{2n}(\bm{\theta}^{*})\ge\frac{\min\{n_{-},n_{+}\}}{2n}$, where $n_{+}$ and $n_{-}$ denote the number of positive and negative samples in the dataset, respectively.
\end{proposition}
\textbf{Remark:}  Proposition~\ref{prop::logistic} shows that when the network $f_{S}$ consists of neurons in the sigmoid class,  even if all other conditions are satisfied, the results in Theorem~\ref{thm::convex-finite-deep} does not  hold with a positive probability. %This indicates that Theorem~\ref{thm::convex-finite-deep} does not apply to the neurons in the sigmoid class. 

\textbf{Quadratic family} contains neurons where $\sigma(z)$ is real analytic and strongly convex on $\mathbb{R}$ and has a global minimum at the point $z=0$. A simple example of neuron in this family is the quadratic neuron, i.e., $\sigma(z)=z^{2}$. It is easy to check that all neurons in this class satisfy the conditions in Theorem~\ref{thm::convex-finite-deep} but not in Theorem~\ref{thm::linear-sep-deep}. 
For Theorem~\ref{thm::linear-sep-deep}, we present a counterexample and show that, when the network $f_{S}$ consists of neurons in the quadratic class, even if positive and negative samples are linearly separable,  the empirical loss can have a local minimum with non-zero training error.

\begin{proposition}\label{prop::quadratic}
%Assume that the single layer network $f_{S}(x;\bm{\theta}_{S})$ has $M\ge 1$ neurons.  
Suppose that Assumption~\ref{assump::loss} and~\ref{assump::shortcut-connection} are satisfied.
Assume that neurons in $f_{S}$ satisfy that $\sigma$ is strongly convex and twice differentiable on $\mathbb{R}$ and has a global minimum at $z=0$. There exists a network architecture $f_{D}$ and a distribution satisfying assumptions in Theorem~\ref{thm::linear-sep-deep} such that with probability one, the empirical loss $\hat{L}_{n}(\bm{\theta};p),p\ge2$ has a local minima $\bm{\theta}^{*}=(\bm{\theta}_{S}^{*},\bm{\theta}^{*}_{D})$ 
satisfying $\hat{R}_{n}(\bm{\theta}^{*})\ge\frac{\min\{n_{+},n_{-}\}}{n}$, where $n_{+}$ and $n_{-}$ denote the number of positive and negative samples in the dataset, respectively.
\end{proposition}

%\textbf{Remark:} We provide the proof of Proposition~\ref{prop::quadratic} in Appendix~\ref{appendix::prop-quadratic}. 
%First, it is obvious that the dataset presented in Example~\ref{example::1} satisfies the condition that the matrix $\frac{1}{n_{+}}\sum_{i:y_{i}=1}x_{i}x_{i}^{\top}-\frac{1}{n_{-}}\sum_{i:y_{i}=-1}x_{i}x_{i}^{\top}$ is positive definite. 
%Second, it is easy to check that all positive and negative samples in the dataset $\mathcal{D}$ are linearly separable. Third, it is not difficult to see that the empirical loss $\hat{L}_{n}$ has a global optimal point $\bm{\theta}_{\text{opt}}$ where the single layer network $f(x;\bm{\theta}_{\text{opt}})$ has a zero training error. Thus, by Proposition~\ref{prop::quadratic}, we have that when the neurons in the single layer network are from the quadratic family, the empirical loss can have a bad local minimum even if the dataset is linearly separable. This indicates that Theorem~\ref{thm::linear-sep-deep} does not hold for the quadratic family. 

\subsection{Shortcut-like Connections}\label{sec::short-cut}
In this subsection, we discuss whether the main results still hold if we remove the shortcut-like connections or replace them with the identity shortcut connections used in the residual network~\cite{he2016deep}.  
Specifically, we provide two counterexamples and show that the main results do not hold if the shortcut-like connections are removed or replaced with the identity shortcut connections.

\textbf{Feed-forward networks.} When the shortcut-like connections (i.e., the network $f_{S}$ in Figure~\ref{fig::network}(b)) are removed, the network architecture can be viewed as  a standard feedforward neural network. %In Theorem~\ref{thm::convex-finite-deep} and~\ref{thm::linear-sep-deep}, we assume that the neurons in the feedforward network $f_{D}$ are real functions defined on $\mathbb{R}$. 
We provide a counterexample to show that, for a feedforward network with ReLU neurons, even if the other conditions  in Theorem~\ref{thm::convex-finite-deep} or ~\ref{thm::linear-sep-deep} are satisfied, the empirical loss functions is likely to have a local minimum with non-zero training error. In other words, neither Theorem~\ref{thm::convex-finite-deep} nor ~\ref{thm::linear-sep-deep} holds when the shortcut-like connections are removed. 

\begin{proposition}\label{prop::feedforward}
Suppose that assumption \ref{assump::loss} is satisfied.
Assume that the feedforward network $f(x;\bm{\theta})$ has at least one hidden layer and at least one neuron in each hidden layer. If neurons in the network $f$ satisfy that $\sigma(z)=0$ for all $z\le 0$ and $\sigma(z)$ is continuous on $\mathbb{R}$, then for any dataset $\mathcal{D}$ with $n$ samples,  the empirical loss $\hat{L}_{n}(\bm{\theta};p),p\ge 2$ has a local minima $\bm{\theta}^{*}$ with $\hat{R}_{n}(\bm{\theta}^{*})\ge\frac{\min\{n_{+},n_{-}\}}{n}$, where $n_{+}$ and $n_{-}$ are the number of positive and negative samples in the dataset, respectively. 
\end{proposition}
\textbf{Remark:}  The result holds for ReLUs, since it is easy to check that the ReLU neuron satisfies the above assumptions.  %(ii) As implied by Proposition~\ref{prop::feedforward},  the empirical loss has a local minimum with non-zero training error, when the dataset contain both positive and negative samples.  This can happen with probability at least $1-e^{-\Omega(n)}$ if both positive and negative samples have positive probabilities appearing in the dataset.

%For example, when all neurons in a multilayer network are ReLUs, even if the other conditions are satisfied, with probability at least $1-e^{-\Omega(n)}$, the empirical loss function has a local minima with non-zero training error. However, if we add single layer network consisting of sufficient number of softplus neurons with a short-cut connection to this multilayer network, then by Theorem~\ref{thm::convex-finite-deep} and~\ref{thm::linear-sep-deep}, every local minimum of the empirical loss has zero training error. This indicates that the short-cut connections can improve the  landscape the loss surface under certain conditions. 

\textbf{Identity shortcut connections.} As we stated earlier, adding shortcut-like connections to a  network can improve the loss surface. However, the shortcut-like connections shown in Fig~\ref{fig::network}(b) are different from some popular shortcut connections used in the real-world applications, e.g., the identity shortcut connections in the residual network.   
Thus, a natural question arises: do the main results still hold if we use the identity shortcut connections? To address the question, we provide the following counterexample to show that, when we replace the shortcut-like connections with the identity shortcut connections, even if the other conditions in Theorem~\ref{thm::convex-finite-deep} are satisfied, the empirical loss function is likely to have a local minimum with non-zero training error. In other words, Theorem~\ref{thm::convex-finite-deep} does not hold for the identity shortcut connections.
% In addition, we provide a counterexample in Appendix~\ref{appendix::residual}, and show that Theorem~\ref{thm::linear-sep-deep} does not hold for the identity shortcut connections either. 

\begin{proposition}\label{prop::short-cut}
Assume that $H:\mathbb{R}^{d}\rightarrow\mathbb{R}^{d}$ is a feedforward neural network parameterized by $\bm{\theta}$ and all neurons in $H$ are ReLUs. Define a network $f:\mathbb{R}^{d}\rightarrow \mathbb{R}$ with identity shortcut connections as $f(x;\bm{a},\bm{\theta},b)=\bm{a}^{\top}(x+H(x;\bm{\theta}))+b$, $\bm{a}\in\mathbb{R}^{d}, b\in\mathbb{R}$. Then there exists a distribution $\mathbb{P}_{\bm{X}\times Y}$ satisfying the assumptions in Theorem~\ref{thm::convex-finite-deep} such that with probability at least $1-e^{-\Omega(n)}$, the empirical loss $\hat{L}_{n}(\bm{a},\bm{\theta},b;p)=\frac{1}{n}\sum_{i=1}^{n}\ell(-y_{i}f(x_{i};\bm{\theta});p), p\ge 2$ has a local minimum with non-zero training error. 
\end{proposition}

%In summary, when the shortcut-like connections shown in Fig~\ref{fig::network}(b) are removed or replaced by the identity shortcut connections, the main results may no longer hold. 

\subsection{Loss Functions}\label{sec::loss}
In this subsection, we discuss whether the main results still hold if we change the loss function. 
We mainly focus on the following two types of surrogate loss functions: quadratic loss and logistic loss. We will show that if the loss function is replaced with the quadratic loss or logistic loss, then neither Theorem~\ref{thm::convex-finite-deep} nor~\ref{thm::linear-sep-deep} holds. In addition, we show that when the loss function is the logistic loss and the network is a feedforward neural network, there are no local minima with zero training error in the real parameter space. In Fig.~\ref{fig::loss}(b), we visualize some surrogate loss functions discussed in this subsection.

%first show that if the loss function is replaced with the quadratic loss, then neither Theorem~\ref{thm::convex-finite-deep} nor~\ref{thm::linear-sep-deep} hold. Next, we show that when the loss function is replaced with the logistic loss, even if the other conditions in Theorem~\ref{thm::convex-finite-deep} or~\ref{thm::linear-sep-deep} are satisfied, every critical point of the empirical loss is a saddle point. Furthermore, we present a general result showing that, for any dataset containing both positive and negative samples and for any multilayer neural architectures, if the loss function is the logistic loss,  then every critical point of the empirical loss has non-zero training error. 

%distinguish between quadratic loss and the other two types of loss function by showing that the global minima of the empirical loss consisted of quadratic loss can have zero-training errors in the cases where the other two types of loss do not. We next distinguish between the logistic loss and the surrogate loss considered in this paper by showing that the empirical loss function consisted of logistic loss may not have a local minimum in the real parameter space. 

\textbf{Quadratic loss.}  The  quadratic loss $\ell(z)=(1+z)^{2}$ has been well-studied in prior works. It has been shown that when the loss function is quadratic, under certain assumptions, all local minima of the empirical loss are global minima. However, the global minimum of the quadratic loss does not necessarily have  zero misclassification error, even in the realizable case (i.e., the case where there exists a set of parameters such that the network achieves zero misclassification error on the dataset or the data distriubtion). To illustrate this, we provide a simple example where the network is a simplified linear network and the data distribution is linearly separable. 
\begin{example}\label{prop::quadratic-loss-linsep}
Let the distribution $\mathbb{P}_{\bm{X}\times Y}$ satisfy that $\mathbb{P}(Y=1)=\mathbb{P}(Y=-1)=0.5$, $\mathbb{P}(X=5/4|Y=1)=1$ and $\mathbb{P}_{X|Y=-1}$ is a uniform distribution on the interval $[0,1]$.  For a linear model $f(x;a,b)=ax+b,$ $a,b\in\mathbb{R}$, every global minimum $(a^{*},b^{*})$ of the population loss ${L}(a,b)=\mathbb{E}_{X\times Y}[(1-Yf(X;a,b))^{2}]$ satisfies $\mathbb{P}_{\bm{X}\times Y}[Y\neq \sgn(f(X;a^{*},b^{*}))]\ge 1/16$.
\end{example}
\textbf{Remark:} The proof of the above result in Appendix~\ref{appendix::prop-quadratic-loss-linsep} is very straightforward. We have only provided it there since we are unable to find a reference which explicitly states such a result, but we will not be surprised if this result has been known to others. This example shows that every global minimum of the quadratic loss has non-zero misclassification error, although the linear model is able to achieve zero misclassification error on this data distribution. Similarly, one can easily find datasets under which all global minima of the quadratic loss have non-zero training error.

In addition, we provide two examples in Appendix~\ref{appendix::prop-quadratic-loss} and show that, when the loss function is replaced with the quadratic loss, even if the other conditions in Theorem~\ref{thm::convex-finite-deep} or \ref{thm::linear-sep-deep} are satisfied, every global minimum of the empirical loss has a training error larger than $1/8$ with a positive probability. 
In other words, our main results do hold for the quadratic loss. 

The following observation may be of independent interest.
Different from the quadratic loss, the loss functions conditioned in Assumption~\ref{assump::loss} have the following two properties: (i) the minimum empirical loss is zero if and only if there exists a set of parameters achieving zero training error; (ii) every global minimum of the empirical loss has zero training error in the realizable case.

\begin{proposition}\label{prop::loss}
Let $f:\mathbb{R}^{d}\rightarrow\mathbb{R}$ denote a feedforward network parameterized by $\bm{\theta}$ and let the dataset have $n$ samples. When the loss function $\ell_{p}$ satisfies Assumption~\ref{assump::loss} and $p\ge1$, we have $\min_{\bm{\theta}}\hat{L}_{n}(\bm{\theta};p)=0$ if and only if $\min_{\bm{\theta}}\hat{R}_{n}(\bm{\theta})=0$. Furthermore, if $\min_{\bm{\theta}}\hat{R}_{n}(\bm{\theta})=0$, every global minimum $\bm{\theta}^{*}$ of the empirical loss $\hat{L}_{n}(\bm{\theta};p)$ has zero training error, i.e., $\hat{R}_{n}(\bm{\theta}^{*})=0$. 
\end{proposition}
\textbf{Remark:}  We note that the network does not need to be a feedforward network. In fact, the same results hold for a large class of network architectures, including both architectures shown in Fig~\ref{fig::network}. We provide additional analysis in Appendix~\ref{appendix::prop-loss}.

\textbf{Logistic loss.}  
The logistic loss $\ell(z)=\log_{2}\left(1+e^{z}\right)$ is different from the loss functions conditioned in Assumption~\ref{assump::loss}, since the logistic loss does not have a global minimum on $\mathbb{R}$. Here, for the logistic loss function, we show that even if the remaining assumptions in Theorem~\ref{thm::convex-finite-deep} hold, every critical point is a saddle point. In other words, Theorem~\ref{thm::convex-finite-deep} does not hold for logistic loss. Additional analysis on Theorem~\ref{thm::linear-sep-deep} are provided in Appendix~\ref{appendix::logit-linear-sep}.  
\begin{proposition}\label{prop::logit-convex-finite}
Assume that the loss function is the logistic loss, i.e., $\ell(z)=\log_{2}(1+e^{z})$.
Assume that assumptions~\ref{assump::full-rank}-\ref{assump::neurons} are satisfied. 
Assume that  samples in the dataset $\mathcal{D}=\{(x_{i},y_{i})\}_{i=1}^{n}, n\ge 1$ are independently drawn from the distribution $\mathbb{P}_{\bm{X}\times Y}$.  Assume that the number of neurons $M$ in the network $f_{S}$ satisfies $M\ge 2\max\{\frac{n}{\Delta r},r_{+},r_{-}\}$, where $\Delta r=r-\max\{r_{+},r_{-}\}$. If $\bm{\theta}^{*}$ denotes a critical point of the empirical loss $\hat{L}_{n}(\bm{\theta})$, then $\bm{\theta}^*$ is a saddle point. In particular, there are no local minima.
\end{proposition}
\textbf{Remark: }  We note here that the result can be generalized to every loss function $\ell$ which is real analytic and has a positive derivative on $\mathbb{R}$.

Furthermore, we provide the following result to show that when the dataset contains both positive and negative samples, if the loss is the logistic loss, then every critical point of the empirical loss function has non-zero training error. 

\begin{proposition}\label{prop::logit-general}
Assume the dataset ${\mathcal{D}=\{(x_{i},y_{i})\}_{i=1}^{n}}$ consists of both positive and negative samples. Assume that $f(x;\bm{\theta})$ is a feedforward network parameterized by $\bm{\theta}$. Assume that the loss function is logistic, i.e., $\ell(z)=\log_{2}\left(1+e^{z}\right)$. If the real parameters $\bm{\theta}^{*}$ denote a critical point of the empirical loss $\hat{L}_{n}(\bm{\theta}^{*})$, then $\hat{R}_{n}(\bm{\theta}^{*})>0$.
\end{proposition}
\textbf{Remark: } We provide the proof in Appendix~\ref{appendix::prop-logit-general}. The above proposition implies every critical point is either a local minimum with non-zero training error or is a saddle point (also with non-zero training error). We note here that, similar to Proposition~\ref{prop::logit-convex-finite}, the result can be generalized to every loss function $\ell$  that is differentiable and has a positive derivative on $\mathbb{R}$.

%In summary, when the loss function is replaced with the quadratic loss and logistic loss, then the main results may not hold.

\subsection{Open Problem: Datasets}\label{sec::necc}
In this paper, we have mainly considered a class of non-linearly separable distribution where positive and negative samples are located on different subspaces. We show that if the samples are drawn from such  a distribution, under certain additional conditions, all local minima of the empirical loss have zero training errors. However, one may ask: how well does the result generalize to other non-linearly separable distributions or datasets? Here, we partially answer this question by presenting the following necessary condition on the dataset so that Theorem~\ref{thm::convex-finite-deep} can hold. 
\begin{proposition}\label{prop::convex-necc}

Suppose that assumptions~\ref{assump::loss},~\ref{assump::shortcut-connection} and~\ref{assump::neurons} are satisfied. For any feedforward architecture $f_{D}(x;\bm{\theta}_{D})$, every local minimum $\bm{\theta}^{*}=(\bm{\theta}_{S}^{*},\bm{\theta}_{D}^{*})$ of the empirical loss function $\hat{L}_{n}(\bm{\theta}_{S},\bm{\theta}_{D};p)$, $p\ge 6$ satisfies $\hat{R}_{n}(\bm{\theta}^{*})=0$ \textbf{only if}  the matrix $\sum_{i=1}^{n}\lambda_{i}y_{i}x_{i}x_{i}^{\top}$ is neither positive nor negative definite for all sequences $\{\lambda_{i}\ge0\}_{i=1}^{n}$ satisfying $\sum_{i:y_{i}=1}\lambda_{i}=\sum_{i:y_{i}=-1}\lambda_{i}>0$ and $\|\sum_{i=1}^{n}\lambda_{i}y_{i}x_{i}\|_{2}=0$.
\end{proposition}
\textbf{Remark:} The proposition implies that when the dataset does not meet this necessary condition, there exists a feedforward architecture $f_{D}$ such that the empirical loss function has a local minimum with a non-zero training error. We use this implication to prove the counterexamples provided in Appendix~\ref{appendix::exam-dataset} when Assumption~\ref{assump::full-rank} or \ref{assump::different-subspaces} on the dataset is not satisfied. Therefore, Theorem~\ref{thm::convex-finite-deep} no longer holds when Assumption~\ref{assump::full-rank} or  \ref{assump::different-subspaces} is removed. 
We note that the necessary condition shown here is not equivalent to Assumption~\ref{assump::full-rank} and  \ref{assump::different-subspaces}. Now we present the following result to show the sufficient and necessary condition that the dataset should satisfy so that Proposition~\ref{prop::results-quadratic} can hold.

\begin{proposition}\label{prop::quadratic-sn}
Suppose that the loss function $\ell_{p}$ satisfies Assumption~\ref{assump::loss} and neurons in the network satisfy Assumption~\ref{assump::neurons}.
Assume that the single layer network $f_{S}(x;\bm{\theta}_{S})$ has $M> d$ neurons and
assume that neurons in $f_{S}$ are quadratic neurons, i.e., $\sigma(z)=z^{2}$.
For any network architecture $f_{D}(x;\bm{\theta}_{D})$, every local minimum $\bm{\theta}^{*}=(\bm{\theta}_{S}^{*},\bm{\theta}_{D}^{*})$ of the empirical loss function $\hat{L}_{n}(\bm{\theta}_{S},\bm{\theta}_{D};p)$, $p\ge 6$ satisfies $\hat{R}_{n}(\bm{\theta}^{*})=0$ \textbf{ if and only if} the matrix $\sum_{i=1}^{n}\lambda_{i}y_{i}x_{i}x_{i}^{\top}$ is indefinite for all sequences $\{\lambda_{i}\ge0\}_{i=1}^{n}$ satisfying $\sum_{i:y_{i}=1}\lambda_{i}=\sum_{i:y_{i}=-1}\lambda_{i}>0$.

\end{proposition}
\textbf{Remark:} (i) This sufficient and necessary condition implies that for any network architecture $f_{D}$, there exists a set of parameters $\bm{\theta}=(\bm{\theta}_{S},\bm{\theta}_{D})$ such that the network $f(x;\bm{\theta})=f_{S}(x;\bm{\theta}_{S})+f_{D}(x;\bm{\theta}_{D})$ can correctly classify all samples in the dataset. This also indicates the existence of a set of parameters achieving zero training error, regardless of the network architecture of $f_{D}$. 
We provide the proof in Appendix~\ref{appendix::lemma-suff-necc}.
(ii) We note that Proposition~\ref{prop::quadratic-sn} only holds for the quadratic neuron. The problem of finding the sufficient and necessary conditions for the other types of neurons is open.

\section{Conclusions}\label{sec::conclusions}\vspace{-0.2cm}
In this paper, we studied the surface of a smooth version of the hinge loss function in binary classification problems. We provided conditions under which the neural network has zero misclassification error at all local minima and also provide counterexamples to show that when some of these assumptions are relaxed, the result may not hold. Further work involves exploiting our results to design efficient training algorithms classification tasks using  neural networks.  

\vspace{-0.2cm}
\bibliography{icml2018}
\bibliographystyle{unsrt}

%Theorem 2, linearly separability, change the constant c to zero. 
%proposition 2, continuous changes to piece-wise continuous.
%proposition 2, check whether the result holds for both theorems.
%proposition 3, check piece-wise continuous1``````

\begin{appendix}

\section{Additional Results in Section~\ref{sec::main-results}}

\subsection{Proof of Lemma~\ref{lemma::nec-single}}

\begin{lemma}[Necessary condition.]\label{lemma::nec-single} Assume that neurons $\sigma$ in the network $f_{S}$ are twice differentiable and the loss function $\ell:\mathbb{R}\rightarrow \mathbb{R}$ has a continuous derivative on $\mathbb{R}$ up to the third order. If  $n\ge 1$ and parameters $\bm{\theta}^{*}=(\bm{\theta}^{*}_{S},\bm{\theta}^{*}_{D})$ denote a local minimum of the loss function $\hat{L}_{n}(\bm{\theta})$, then for any $j=1,...,M$,
\begin{align*}
&\sum_{i=1}^{n}\ell'(-y_{i}f(x_{i};\bm{\theta}^{*}))y_{i}\sigma'({\bm{w}^{*}_{j}}^{\top}x_{i})x_{i}=\bm{0}_{d}.
\end{align*}
\end{lemma}

\begin{proof}
We first recall some notations defined in the paper. The output of the neural network is 
$$f(x;\bm{\theta})=f_{S}(x;\bm{\theta}_{S})+f_{D}(x;\bm{\theta}_{D}),$$
where $f_{S}(x;\bm{\theta}_{S})$ is the single layer neural network parameterized by $\bm{\theta}_{S}$, i.e., 
$$f_{S}(x;\bm{\theta}_{S})=a_{0}+\sum_{j=1}^{M}a_{j}\sigma\left(\bm{w}_{j}^{\top}x\right),$$
and $f_{D}(x;\bm{\theta}_{D})$ is a deep neural network parameterized by $\bm{\theta}_{D}$. 
The empirical loss function is given by
$$\hat{L}_{n}(\bm{\theta})=\hat{L}_{n}(\bm{\theta}_{S},\bm{\theta}_{D})=\frac{1}{n}\sum_{i=1}^{n}\ell(-y_{i}f(x_{i};\bm{\theta})).$$
Since the loss function $\ell$ has a continuous derivative on $\mathbb{R}$ up to the third order, neurons $\sigma$ in the network $f_{S}$ are twice differentiable, then the gradient vector $\nabla_{\bm{\theta}_{S}}\hat{L}_{n}(\bm{\theta}^{*}_{S},\bm{\theta}^{*}_{D})$ and the Hessian matrix $\nabla^{2}_{\bm{\theta}_{S}}\hat{L}_{n}(\bm{\theta}^{*}_{S},\bm{\theta}^{*}_{D})$ exists. 
Furthermore, by the assumption that  $\bm{\theta}^{*}=(\bm{\theta}^{*}_{S},\bm{\theta}_{D}^{*})$ is a local minima of the loss function $\hat{L}_{n}(\bm{\theta})$, then we should have for $j=1,..., M$,
\begin{align}
\bm{0}_{d}=\nabla_{\bm{w}_{j}}L_{n}(\bm{\theta}^{*})&=\sum_{i=1}^{n}\ell'(-y_{i}f(x_{i};\bm{\theta}^{*}))(-y_{i}\nabla_{\bm{w}_{j}}f(x_{i};\bm{\theta}^{*}))\notag\\
&=\sum_{i=1}^{n}\ell'(-y_{i}f(x_{i};\bm{\theta}^{*}))(-y_{i}a^{*}_{j}\sigma'({\bm{w}_{j}^{*}}^{\top}x_{i})x_{i})\notag\\
&=-a^{*}_{j}\sum_{i=1}^{n}\ell'(-y_{i}f(x_{i};\bm{\theta}^{*}))y_{i}\sigma'({\bm{w}_{j}^{*}}^{\top}x_{i})x_{i}.\label{eq::lemma1-critical}
\end{align}
Now we need to prove that if $\bm{\theta}^{*}$ is a local minima, then
\begin{equation*}
\forall j\in[M], \quad\left\|\sum_{i=1}^{n}\ell'(-y_{i}f(x_{i};\bm{\theta}^{*}))y_{i}\sigma'({\bm{w}_{j}^{*}}^{\top}x_{i})x_{i}\right\|_{2}= 0.
\end{equation*}
We prove it by contradiction. Assume that there exists $j\in[M]$ such that 
\begin{equation*}
\left\|\sum_{i=1}^{n}\ell'(-y_{i}f(x_{i};\bm{\theta}^{*}))y_{i}\sigma'({\bm{w}_{j}^{*}}^{\top}x_{i})x_{i}\right\|_{2}\neq 0.
\end{equation*}
Then by equation~\eqref{eq::lemma1-critical}, we have $a^{*}_j=0$. Now, we consider the following Hessian matrix $H(a_{j},\bm{w}_{j})$. Since $\bm{\theta}^{*}$ is a local minima of the loss function $\hat{L}_{n}(\bm{\theta})$, then the matrix $H(a_{j},\bm{w}_{j})$ should be positive semidefinite at $(a_{j}^{*},\bm{w}_{j}^{*})$. By $a^{*}_{j}=0$, we have 
\begin{align*}
 \nabla_{\bm{w}_{j}}^{2}L_{n}(\bm{\theta}^{*})&=-a^{*}_{j}\nabla_{\bm{w}_{j}}\left[\sum_{i=1}^{n}\ell'(-y_{i}f(x_{i};\bm{\theta}^{*}))y_{i}\sigma'({\bm{w}_{j}^{*}}^{\top}x_{i})x_{i}\right]=\bm{0}_{d\times d},\\
\frac{\partial \left[\nabla_{w_{j}}L_{n}(\bm{\theta}^{*})\right]}{\partial a_{j}}&=-\sum_{i=1}^{n}\ell'(-y_{i}f(x_{i};\bm{\theta}^{*}))y_{i}\sigma'({\bm{w}_{j}^{*}}^{\top}x_{i})x_{i}\\
&\quad -a^{*}_{j}\frac{\partial}{\partial a_{j}}\left[\sum_{i=1}^{n}\ell'(-y_{i}f(x_{i};\bm{\theta}^{*}))y_{i}\sigma'({\bm{w}_{j}^{*}}^{\top}x_{i})x_{i}\right]\\
&=-\sum_{i=1}^{n}\ell'(-y_{i}f(x_{i};\bm{\theta}^{*}))y_{i}\sigma'({\bm{w}_{j}^{*}}^{\top}x_{i})x_{i}.
\end{align*}
In addition,  we have
\begin{align*}
\frac{\partial^{2} L_{n}(\bm{\theta}^{*})}{\partial a_{j}^{2}}&=\frac{\partial}{\partial a_{j}}\left[\sum_{i=1}^{n}\ell'(-y_{i}f(x_{i};\bm{\theta}^{*}))(-y_{i}\sigma({\bm{w}_{j}^{*}}^{\top}x_{i}))\right]\\
&=\sum_{i=1}^{n}\ell''(-y_{i}f(x_{i};\bm{\theta}^{*}))\sigma^{2}({\bm{w}_{j}^{*}}^{\top}x_{i}).
\end{align*}
Since the matrix $H(a_{j}^{*},\bm{w}_{j}^{*})$ is positive semidefinite, then for any $\alpha\in\mathbb{R}$ and $\bm{\omega}\in\mathbb{R}^{d}$, 
\begin{align*}
\left(\begin{matrix}\alpha &\bm{\omega}^{\top}\end{matrix}\right)H(a_{j}^{*},\bm{w}_{j}^{*})\left(\begin{matrix}\alpha \\\bm{\omega}\end{matrix}\right)\ge 0.
\end{align*}
Since
\begin{align*}
\left(\begin{matrix}\alpha &\bm{\omega}^{\top}\end{matrix}\right)H(a_{j}^{*},\bm{w}_{j}^{*})\left(\begin{matrix}\alpha \\ \bm{\omega}\end{matrix}\right)&=\alpha^{2}\sum_{i=1}^{n}\ell''(-y_{i}f(x_{i};\bm{\theta}^{*}))\sigma^{2}({\bm{w}_{j}^{*}}^{\top}x_{i})\\
&\quad-\alpha \bm{\omega}^{\top}\sum_{i=1}^{n}\ell'(-y_{i}f(x_{i};\bm{\theta}^{*}))y_{i}\sigma'({\bm{w}_{j}^{*}}^{\top}x_{i})x_{i},
\end{align*}
and by setting $$\bm{\omega}=\sum_{i=1}^{n}\ell'(-y_{i}f(x_{i};\bm{\theta}^{*}))y_{i}\sigma'({\bm{w}_{j}^{*}}^{\top}x_{i})x_{i},$$
then
\begin{align*}
\left(\begin{matrix}\alpha &\omega^{\top}\end{matrix}\right)H(a_{j}^{*},\bm{w}_{j}^{*})\left(\begin{matrix}\alpha \\\omega\end{matrix}\right)&=\alpha^{2}\sum_{i=1}^{n}\ell''(-y_{i}f(x_{i};\bm{\theta}^{*}))\sigma^{2}({\bm{w}_{j}^{*}}^{\top}x_{i})\\
&\quad-\alpha \left\|\sum_{i=1}^{n}\ell'(-y_{i}f(x_{i};\bm{\theta}^{*}))y_{i}\sigma'({\bm{w}_{j}^{*}}^{\top}x_{i})x_{i}\right\|^{2}_{2}.
\end{align*}
Furthermore, since we assume that
$$\left\|\sum_{i=1}^{n}\ell'(-y_{i}f(x_{i};\bm{\theta}^{*}))y_{i}\sigma'({\bm{w}_{j}^{*}}^{\top}x_{i})x_{i}\right\|^{2}_{2}>0,$$
then clearly, there exists $\alpha$ such that 
$$\left(\begin{matrix}\alpha &\bm{\omega}^{\top}\end{matrix}\right)H(a_{j}^{*},\bm{w}_{j}^{*})\left(\begin{matrix}\alpha \\\bm{\omega}\end{matrix}\right)<0.$$
and this leads to the contradiction. Thus, we proved the lemma. 

\end{proof}

\clearpage

\subsection{Proof of Theorem~\ref{thm::convex-finite-deep}}\label{sec::thm-convex}
\begin{theorem}
Assume that the loss function $\ell_{p}$ satisfies  assumption~\ref{assump::loss}, the distribution $\mathbb{P}_{\bm{X}\times Y}$ satisfies assumption~\ref{assump::full-rank} and \ref{assump::different-subspaces}, the network architecture satisfies assumption~\ref{assump::shortcut-connection} and neurons in the network satisfy assumption~\ref{assump::neurons}.
Assume that  samples in the dataset $\mathcal{D}=\{(x_{i},y_{i})\}_{i=1}^{n}, n\ge 1$ are independently drawn from the distribution $\mathbb{P}_{\bm{X}\times Y}$.  Assume that the number of neurons $M$ in the network $f_{S}$ satisfies $M\ge 2\max\{\frac{n}{\Delta r},r_{+},r_{-}\}$, where $\Delta r=r-\max\{r_{+},r_{-}\}$. If the real parameters $\bm{\theta}^{*}=(\bm{\theta}_{S}^{*},\bm{\theta}_{D}^{*})$ denote a local minimum of the loss function $\hat{L}_{n}(\bm{\theta}_{S},\bm{\theta}_{D};p)$ and $p\ge 6$, then  $\hat{R}_{n}(\bm{\theta}^{*}_{S},\bm{\theta}^{*}_{D})=0$ holds with probability one.
\end{theorem}
%Sketch of Proof 
%Differentiability
\begin{proof}
We first present some  notations used in this proof. The output of the neural network is 
$$f(x;\bm{\theta})=f_{S}(x;\bm{\theta}_{S})+f_{D}(x;\bm{\theta}_{D}),$$
where $f_{S}(x;\bm{\theta}_{S})$ is the single layer neural network parameterized by $\bm{\theta}_{S}$, i.e., 
$$f_{S}(x;\bm{\theta}_{S})=a_{0}+\sum_{j=1}^{M}a_{j}\sigma\left(\bm{w}_{j}^{\top}x\right),$$
and $f_{D}(x;\bm{\theta}_{D})$ is a deep neural network parameterized by $\bm{\theta}_{D}$. 
The empirical loss function is given by
$$\hat{L}_{n}(\bm{\theta};p)=\hat{L}_{n}(\bm{\theta}_{S},\bm{\theta}_{D};p)=\frac{1}{n}\sum_{i=1}^{n}\ell_{p}(-y_{i}f(x_{i};\bm{\theta}))$$
We first assume that the real parameters $\bm{\theta}^{*}=(\bm{\theta}_{S}^{*},\bm{\theta}_{D}^{*})$ denote a local minima of the loss function $\hat{L}_{n}(\bm{\theta};p)$.
Next, we prove the following two claims: 

\textbf{Claim 1:} If $\bm{\theta}^{*}=(\bm{\theta}_{S}^{*},\bm{\theta}_{D}^{*})$ is a local minima and there exists $j\in[M]$ such that $a^{*}_{j}=0$, then $\error =0$. 

\textbf{Claim 2:} If $\bm{\theta}^{*}=(\bm{\theta}_{S}^{*},\bm{\theta}_{D}^{*})$ is a local minima and $a^{*}_{j}\neq 0$ for all $j\in [M]$, then $\error =0$.

\textbf{(a) Proof of claim 1.} We prove that  if $\bm{\theta}^{*}=(\bm{\theta}_{S}^{*},\bm{\theta}_{D}^{*})$ is a local minima of the loss function $\hat{L}_{n}(\bm{\theta};p)$ and there exists $j\in[M]$ such that $a^{*}_{j}=0$, then $\error =0$. Without loss of generality, we assume that $a_{1}^{*}=0$. Since $\bm{\theta}^{*}=(\bm{\theta}_{S}^{*},\bm{\theta}_{D}^{*})$ is a local minima, then there exists $\varepsilon_{0}>0$ such that for all small perturbations $\Delta{a}_{1}$, $\Delta \bm{w}_{1}$ on the parameters $a^{*}_{1}$ and $ \bm{w}^{*}_{1}$, i.e., $|\Delta a_{1}|^{2}+\|\Delta\bm{w}_{1}\|_{2}^{2}\le \varepsilon_{0}^{2}$, we have 
$$\hat{L}_{n}(\tilde{\bm{\theta}}_{S},\bm{\theta}_{D}^{*};p)\ge \hat{L}_{n}(\bm{\theta}^{*}_{S},\bm{\theta}_{D}^{*};p),$$
where  $\tilde{\bm{\theta}}_{S}=(\tilde{a}_{0}, \tilde{a}_{1},...,\tilde{a}_{M},\tilde{\bm{w}}_{1},...,\tilde{\bm{w}}_{M})$, $\tilde{a}_{1}=a^{*}_{1}+\Delta a_{1}$, $\tilde{\bm{w}}_{1}=\bm{w}_{1}^{*}+\Delta \bm{w}_{1}$ and $\tilde{a}_{j}=a^{*}_{j}$, $\tilde{\bm{w}}_{j}=\bm{w}^{*}_{j}$ for $j\neq 1$.  Now we consider the Taylor expansion of $\hat{L}_{n}(\tilde{\bm{\theta}}_{S},\bm{\theta}_{D}^{*};p)$ at the point $\bm{\theta}^{*}=(\bm{\theta}_{S}^{*},\bm{\theta}_{D}^{*})$. We note here that the Taylor expansion of $\hat{L}(\bm{\theta}_{S},\bm{\theta}_{D}^{*};p)$ on $\bm{\theta}_{S}$ always exists, since the empirical loss function $\hat{L}_{n}$ has continuous derivatives with respect to $f_{S}$ up to the $p$-th order and the output of the neural network $f(x;\bm{\theta}_{S})$ is infinitely differentiable with respect to $\bm{\theta}_{S}$ due to the fact that neuron activation function $\sigma$ is real analytic.

We first calculate the first order derivatives at the point $\bm{\theta}^{*}$,
\begin{align*}
\frac{d\hat{L}_{n}(\bm{\theta}^{*};p)}{da_{1}}&=\frac{1}{n}\sum_{i=1}^{n}\ell_{p}'(-y_{i}f(x_{i};\bm{\theta}^{*}))(-y_{i})\sigma\left({\bm{w}_{1}^{*}}^{\top}x_{i}\right)=0,&& \text{$\bm{\theta}^{*}$ is a critical point,}\\
\nabla_{\bm{w}_{1}}\hat{L}_{n}(\bm{\theta}^{*};p)&=\frac{a^{*}_{1}}{n}\sum_{i=1}^{n}\ell_{p}'(-y_{i}f(x_{i};\bm{\theta}^{*}))(-y_{i})\sigma'\left({\bm{w}_{1}^{*}}^{\top}x_{i}\right)x_{i}=\bm{0}_{d},&& \text{$\bm{\theta}^{*}$ is a critical point.}
\end{align*}
Next, we calculate the second order derivatives at the point $\bm{\theta}^{*}$,
\begin{align*}
\frac{d^{2}\hat{L}_{n}(\bm{\theta}^{*};p)}{da_{1}^{2}}&=\frac{1}{n}\sum_{i=1}^{N}\ell''_{p}(-y_{i}f(x_{i};\bm{\theta}^{*}))\sigma^{2}\left({\bm{w}_{1}^{*}}^{\top}x_{i}\right)\ge 0,\\
\frac{d}{da_{1}}(\nabla_{\bm{w}_{1}}\hat{L}_{n}(\bm{\theta}^{*};p))&=\frac{1}{n}\sum_{i=1}^{n}\ell_{p}'(-y_{i}f(x_{i};\bm{\theta}^{*}))(-y_{i})\sigma'\left({\bm{w}_{1}^{*}}^{\top}x_{i}\right)x_{i}\\
&\quad+\frac{a^{*}_{1}}{n}\sum_{i=1}^{n}\ell''_{p}(-y_{i}f(x_{i};\bm{\theta}^{*}))\sigma\left({\bm{w}_{1}^{*}}^{\top}x_{i}\right)\sigma'\left({\bm{w}_{1}^{*}}^{\top}x_{i}\right)x_{i}\\
&=\bm{0}_{d},
\end{align*}
where the first term equals to the zero vector by  the necessary condition for a local minima presented in Lemma~\ref{lemma::nec-single} and the second term equals to the zero vector by the assumption that $a^{*}_{1}=0$. Furthermore, by the assumption that $a^{*}_{1}=0$, we have 
\begin{equation*}
\nabla^{2}_{\bm{w}_{1}}\hat{L}_{n}(\bm{\theta}^{*};p)=\frac{a_{1}^{*}}{n}\nabla_{w_{1}}\left[\sum_{i=1}^{n}\ell_{p}'(-y_{i}f(x_{i};\bm{\theta}^{*}))(-y_{i})\sigma'\left({\bm{w}_{1}^{*}}^{\top}x_{i}\right)x_{i}\right]=\bm{0}_{d\times d}.
\end{equation*}
Now, we further calculate the third order derivatives 
\begin{align*}
\frac{d}{da_{1}}\left[\nabla_{\bm{w}_{1}}^{2}{ \hat{L}_{n}(\bm{\theta}^{*};p)}\right]&=\frac{1}{n}\frac{d}{da_{1}}\left[a_{1}^{*}\nabla_{\bm{w}_{1}}\left[\sum_{i=1}^{n}\ell_{p}'(-y_{i}f(x_{i};\bm{\theta}^{*}))(-y_{i})\sigma'\left({\bm{w}_{1}^{*}}^{\top}x_{i}\right)x_{i}\right]\right]\\
&=\nabla_{\bm{w}_{1}}\left[\frac{1}{n}\sum_{i=1}^{n}\ell_{p}'(-y_{i}f(x_{i};\bm{\theta}^{*}))(-y_{i})\sigma'\left({\bm{w}_{1}^{*}}^{\top}x_{i}\right)x_{i}\right]+\bm{0}_{d\times d}&& \text{by $a_{1}^{*}=0$}\\
&=\frac{1}{n}\sum_{i=1}^{n}\ell_{p}'(-y_{i}f(x_{i};\bm{\theta}^{*}))(-y_{i})\sigma''\left({\bm{w}_{1}^{*}}^{\top}x_{i}\right)x_{i}x_{i}^{\top}\\
&\quad+\frac{a^{*}_{1}}{n}\sum_{i=1}^{n}\ell_{p}''(-y_{i}f(x_{i};\bm{\theta}^{*}))\left[\sigma'\left({\bm{w}_{1}^{*}}^{\top}x_{i}\right)\right]^{2}x_{i}x_{i}^{\top}\\
&=\frac{1}{n}\sum_{i=1}^{n}\ell_{p}'(-y_{i}f(x_{i};\bm{\theta}^{*}))(-y_{i})\sigma''\left({\bm{w}_{1}^{*}}^{\top}x_{i}\right)x_{i}x_{i}^{\top}&& \text{by $a_{1}^{*}=0$}
\end{align*}
and
$$\nabla^{3}_{\bm{w}_{1}}\hat{L}_{n}(\bm{\theta}^{*};p)=\frac{a^{*}_{1}}{n}\nabla^{2}_{\bm{w}_{1}}\left[\sum_{i=1}^{n}\ell_{p}'(-y_{i}f(x_{i};\bm{\theta}^{*}))(-y_{i})\sigma'\left({\bm{w}_{1}^{*}}^{\top}x_{i}\right)x_{i}\right]=\bm{0}_{d\times d\times d}.$$
In fact, it is easy to show that for any $2\le k\le p$, 
$$\nabla^{k}_{\bm{w}_{1}}\hat{L}_{n}(\bm{\theta}^{*};p)=\frac{a^{*}_{1}}{n}\nabla^{k-1}_{\bm{w}_{1}}\left[\sum_{i=1}^{n}\ell_{p}'(-y_{i}f(x_{i};\bm{\theta}^{*}))(-y_{i})\sigma'\left({\bm{w}_{1}^{*}}^{\top}x_{i}\right)x_{i}\right]=\bm{0}_{\underbrace{d\times d\times  ...\times d}_{ \text{$k$ times}}}.$$
Let $\varepsilon>0$, $|\Delta a_{1}|=\varepsilon^{9/4}$ and $\Delta \bm{w}_{1}=\varepsilon \bm{u}_{1}$ for $\bm{u}_{1}:\|\bm{u}_{1}\|_{2}=1$.  Clearly, when $\varepsilon\rightarrow 0$, $\Delta a_{1}=o(\|\Delta \bm{w}_{1}\|_{2})$, $\Delta a_{1}=o(1)$ and $\|\Delta \bm{w}_{1}\|=o(1)$. Then we expand $\hat{L}_{n}(\tilde{\bm{\theta}};p)$ at the point $\bm{\theta}^{*}$ up to the sixth order  and thus as $\varepsilon\rightarrow 0$,
\begin{align*}
\hat{L}_{n}(\tilde{\bm{\theta}};p)&=\hat{L}_{n}({\bm{\theta}}^{*};p)+\frac{1}{2!}\frac{d^{2}\hat{L}_{n}(\bm{\theta}^{*};p)}{d^{2}a_{1}}(\Delta a_{1})^{2}\\
&\quad+\frac{1}{2}\Delta a_{1}\Delta \bm{w}_{1}^{\top}\frac{d}{da_{1}}\left[\nabla_{\bm{w}_{1}}^{2}{ \hat{L}_{n}(\bm{\theta}^{*};p)}\right]\Delta \bm{w}_{1} + o(|\Delta a_{1}|^{2})+o(|\Delta a_{1}|\|\Delta \bm{w}_{1}\|^{2}_{2})+o(\|\Delta \bm{w}_{1}\|_{2}^{5})\\
&=\hat{L}_{n}({\bm{\theta}}^{*})+\frac{1}{2!}\frac{d^{2}\hat{L}_{n}(\bm{\theta}^{*};p)}{d^{2}a_{1}}\varepsilon^{9/2}\\
&\quad+\frac{1}{2n}\text{sgn}(\Delta a_{1}) \varepsilon^{9/4+2}\sum_{i=1}^{n}\ell_{p}'(-y_{i}f(x_{i};\bm{\theta}^{*}))(-y_{i})\sigma''\left({\bm{w}_{1}^{*}}^{\top}x_{i}\right)(\bm{u}_{1}^{\top}x_{i})^{2}\\
&\quad+o(\varepsilon^{9/2})+o(\varepsilon^{9/4+2})+o(\varepsilon^{5})\\
&=\hat{L}_{n}({\bm{\theta}}^{*})+\frac{1}{2n}\text{sgn}(\Delta a_{1})\varepsilon^{17/4}\sum_{i=1}^{n}\ell_{p}'(-y_{i}f(x_{i};\bm{\theta}^{*}))(-y_{i})\sigma''\left({\bm{w}_{1}^{*}}^{\top}x_{i}\right)(\bm{u}_{1}^{\top}x_{i})^{2}+o(\varepsilon^{17/4}).
\end{align*}
Since $\varepsilon>0$ and $\hat{L}_{n}(\tilde{\bm{\theta}};p)\ge \loss$ holds for any $\bm{u}_{1}:\|\bm{u}_{1}\|_{2}=1$ and any $\sgn(\Delta a_{1})\in\{-1, 1\}$, then 
\begin{equation}\label{eq::thm2-part1-cond}\sum_{i=1}^{n}\ell_{p}'(-y_{i}f(x_{i};\bm{\theta}^{*}))(-y_{i})\sigma''\left({\bm{w}_{1}^{*}}^{\top}x_{i}\right)(\bm{u}^{\top}x_{i})^{2}=0, \quad\text{for any } \bm{u}\in\mathbb{R}^{d}.\end{equation}
Therefore, 
\begin{equation*}
\sum_{i=1}^{n}\ell_{p}'(-y_{i}f(x_{i};\bm{\theta}^{*}))(-y_{i})\sigma''\left({\bm{w}_{1}^{*}}^{\top}x_{i}\right)x_{i}x_{i}^{\top}=\bm{0}_{d\times d}.
\end{equation*}
By assumption that there exists a set of orthogonal basis $\mathcal{E}=\{\bm{e}_{1},...,\bm{e}_{d}\}$ in $\mathbb{R}^{d}$ and a subset $\mathcal{U}_{+}\subseteq \mathcal{E}$  such that $\mathbb{P}_{\bm{X}|Y}(\bm{X}\in\text{Span}(\mathcal{U}_{1})|Y=1)=1$ and by assumption that $r=|\mathcal{U}_{+}\cup \mathcal{U}_{-}|>\max\{r_{+},r_{-}\}=\max\{|\mathcal{U}_{+}|,|\mathcal{U}_{-}|\}$, then the set $\mathcal{U}_{+}\backslash \mathcal{U}_{-}$ is not an empty set. It is easy to show that for any vector $\bm{v}\in\mathcal{U}_{+}\backslash\mathcal{U}_{-}$, $\mathbb{P}_{\bm{X}\times Y}(\bm{v}^{\top}\bm{X}=0|Y=1)=0$. We prove it by contradiction. If we assume $p=\mathbb{P}_{\bm{X}\times Y}(\bm{v}^{\top}\bm{X}=0|Y=1)>0$, then for random vectors $\bm{X}_{1},...,\bm{X}_{|\mathcal{U}_{+}|}$ independently drawn from the conditional distribution $\mathbb{P}_{\bm{X}| Y=1}$, 
\begin{align*}
\mathbb{P}_{\bm{X} |Y=1}\left(\bigcup_{i=1}^{|\mathcal{U}_{+}|}\left\{\bm{v}^{\top}\bm{X}_{i}=0\right\}\Bigg|Y=1\right)&=\prod_{i=1}^{|\mathcal{U}_{+}|}\mathbb{P}_{\bm{X}|Y=1}\left(\bm{v}^{\top}\bm{X}_{i}=0|Y=1\right)=p^{|\mathcal{U}_{+}|}>0.
\end{align*} 
Furthermore, since $\bm{X}_{1},...,\bm{X}_{|\mathcal{U}_{+}|}\in\text{Span}(\mathcal{U}_{+})$, $\bm{v}^{\top}\bm{X}_{i}=0$, $i=1,...,|\mathcal{U}_{+}|$ and $\bm{v}\in\mathcal{U}_{+}$, then the rank of the matrix $\left(\bm{X}_{1},...,\bm{X}_{|\mathcal{U}_{+}|}\right)$ is at most $|\mathcal{U}_{+}|-1$ and this indicates that the matrix is not a full rank matrix with probability $p^{|\mathcal{U}_{+}|}>0$. This leads to the contradiction with the Assumption~\ref{assump::full-rank}. Thus, with probability 1, $\bm{v}^{\top}x_{i}\neq 0$ for all $i:y_{i}=1$ and $\bm{v}^{\top}x_{i}= 0$ for all $i:y_{i}=-1$.

Therefore, by setting $\bm{u}=\bm{v}$ in Equation~\eqref{eq::thm2-part1-cond}, we have 
\begin{align*}
0=-\sum_{i:y_{i}=1}\ell'_{p}(-y_{i}f(x_{i};\bm{\theta}^{*}))\sigma''({\bm{w}_{1}^{*}}^{\top}x_{i})(\bm{v}^{\top}x_{i})^{2}\le 0,
\end{align*}
where the equality holds if and only if $\forall i: y_{i}=1$, $\ell_{p}'(-y_{i}f(x_{i};\bm{\theta}^{*}))=0$ and this further indicates that $\forall i: y_{i}=1$, $y_{i}f(x_{i};\bm{\theta}^{*})\ge z_{0}>0$. 
Furthermore, since $\bm{\theta}^{*}$ is a critical point and thus 
\begin{align*}
0=\frac{d\loss}{da_{0}}&=\frac{1}{n}\sum_{i=1}^{n}\ell_{p}'(-y_{i}f(x_{i};\bm{\theta}^{*}))(-y_{i})=-\frac{1}{n}\sum_{i:y_{i}=1}\ell_{p}'(-y_{i}f(x_{i};\bm{\theta}^{*}))+\frac{1}{n}\sum_{i:y_{i}=-1}\ell_{p}'(-y_{i}f(x_{i};\bm{\theta}^{*}))\\
&=\frac{1}{n}\sum_{i:y_{i}=-1}\ell_{p}'(-y_{i}f(x_{i};\bm{\theta}^{*})).
\end{align*}
Therefore, $\forall i: y_{i}=-1$, $y_{i}f(x_{i};\bm{\theta}^{*})\ge z_{0}>0$ and this indicates that $\hat{R}_{n}(\bm{\theta}^{*})=0.$

\textbf{Proof of Claim 2:}  First, we define $M_{0}=\lceil M/2\rceil$, then  $$M_{0}\ge\max\{r_{+},r_{-}\}.$$
In addition, since $r=|\mathcal{U}_{+}\cup \mathcal{U}_{-}|,$ then
 $\max\{r_{+},r_{-}\}+\min\{r_{+},r_{-}\}\ge r$. Therefore, 
$$2M_{0}\ge2\max\{r_{+},r_{-}\}>2r-r_{+}-r_{-}\ge 2\min\{r-r_{+},r-r_{-}\}\triangleq 2K,$$
where we define $K=\min\{r-r_{+},r-r_{-}\}$.
Since in claim 2, we assume that $a^{*}_{j}\neq 0$ for all $j\in[M]$, then there exists $a_{i_{1}},..., a_{i_{M_{0}}}$, $i_{1}<i_{2}<...<i_{M_{0}}$ having the same sign, i.e., 
 $$\sgn(a_{i_{1}})=...=\sgn(a_{i_{M_{0}}}).$$
Without loss of generality, we assume that $\sgn(a_{1})=...=\sgn(a_{M_{0}})=+1$.

Now we prove the claim 2. 
First, we consider the Hessian matrix $H(\bm{w}_{1}^{*},...,\bm{w}_{M_{0}}^{*})$. Since $\bm{\theta}^{*}$ is a local minima with $\error>0$, then the inequality
\begin{equation*}
F(\bm{u}_{1},...,\bm{u}_{M_{0}})=\sum_{j=1}^{M_{0}}\sum_{k=1}^{M_{0}}\bm{u}_{j}^{\top}\nabla^{2}_{\bm{w}_{j},\bm{w}_{k}}\loss \bm{u}_{k}\ge 0
\end{equation*}
holds for all vectors $\bm{u}_{1},...,\bm{u}_{M_{0}}\in\mathbb{R}^{d}$. 
Since 
\begin{align*}
\nabla_{\bm{w}_{j}}^{2}\loss&=\frac{a_{j}^{*}}{n}\sum_{i=1}^{n}\ell_{p}'(-y_{i}f(x_{i};\bm{\theta}^{*}))(-y_{i})\sigma''\left({\bm{w}_{j}^{*}}^{\top}x_{i}\right)x_{i}x_{i}^{\top}\\
&\quad+\frac{{a_{j}^{*}}^{2}}{n}\sum_{i=1}^{n}\ell_{p}''(-y_{i}f(x_{i};\bm{\theta}^{*}))\left[\sigma'\left({\bm{w}_{j}^{*}}^{\top}x_{i}\right)\right]^{2}x_{i}x_{i}^{\top},
\end{align*}
and 
\begin{align*}
\nabla_{\bm{w}_{j},\bm{w}_{k}}^{2}\loss&=\frac{{a_{j}^{*}}a_{k}^{*}}{n}\sum_{i=1}^{n}\ell_{p}''(-y_{i}f(x_{i};\bm{\theta}^{*}))\left[\sigma'\left({\bm{w}_{j}^{*}}^{\top}x_{i}\right)\right]\left[\sigma'\left({\bm{w}_{k}^{*}}^{\top}x_{i}+b_{k}^{*}\right)\right]x_{i}x_{i}^{\top}.
\end{align*}
Thus, we have for any $\bm{u}_{1},...,\bm{u}_{M_{0}}\in\mathbb{R}^{d}$,
\begin{align*}
F(\bm{u}_{1},...,\bm{u}_{M_{0}})&=-\frac{1}{n}\sum_{j=1}^{M_{0}}\left[a_{j}^{*}\sum_{i=1}^{n}\ell_{p}'(-y_{i}f(x_{i};\bm{\theta}^{*}))y_{i}\sigma''\left({\bm{w}_{j}^{*}}^{\top}x_{i}\right)\left(\bm{u}_{j}^{\top}x_{i}\right)^{2}\right]\\
&\quad +\frac{1}{n}\sum_{j=1}^{M_{0}}\sum_{k=1}^{M_{0}}\left[{a_{j}^{*}}a_{k}^{*}\sum_{i=1}^{n}\ell_{p}''(-y_{i}f(x_{i};\bm{\theta}^{*}))\sigma'\left({\bm{w}_{j}^{*}}^{\top}x_{i}\right)\sigma'\left({\bm{w}_{k}^{*}}^{\top}x_{i}+b_{k}^{*}\right)\left(\bm{u}_{j}^{\top}x_{i}\right)\left(\bm{u}_{k}^{\top}x_{i}\right)\right]\\
&=-\frac{1}{n}\sum_{i=1}^{n}\left[\ell_{p}'(-y_{i}f(x_{i};\bm{\theta}^{*}))y_{i}\sum_{j=1}^{M_{0}}\left[a_{j}^{*}\sigma''\left({\bm{w}_{j}^{*}}^{\top}x_{i}\right)\left(\bm{u}_{j}^{\top}x_{i}\right)^{2}\right]\right]\\
&\quad +\frac{1}{n}\sum_{i=1}^{n}\left[\ell_{p}''(-y_{i}f(x_{i};\bm{\theta}^{*}))\left(\sum_{j=1}^{M_{0}}a_{j}^{*}\sigma'\left({\bm{w}_{j}^{*}}^{\top}x_{i}\right)\left(\bm{u}_{j}^{\top}x_{i}\right)\right)^{2}\right].
\end{align*}
Now we  find  some coefficients $\alpha_{1},...,\alpha_{M_{0}}$, not all zero, and vectors $\bm{u}_{1},...,\bm{u}_{M_{0}}$, not all zero vector, satisfying
$$\sum_{j=1}^{M_{0}}\alpha_{j}\sigma'\left({\bm{w}_{j}^{*}}^{\top}x_{i}\right)\bm{u}_{j}^{\top}x_{i}=0,\quad \forall i\in[n],$$
and 
$$\forall i:y_{i}=-1 \text{ and }\forall j\in[M_{0}],\quad \bm{u}_{j}^{\top}x_{i}=0.$$
We note here that if $\sgn(a_{1})=...=\sgn(a_{M_{0}})=-1$, then we need to find  coefficients $\alpha_{1},...,\alpha_{M_{0}}$, not all zero, and vectors $\bm{u}_{1},...,\bm{u}_{M_{0}}$, not all zero vector, satisfying
$$\sum_{j=1}^{M_{0}}\alpha_{j}\sigma'\left({\bm{w}_{j}^{*}}^{\top}x_{i}\right)\bm{u}_{j}^{\top}x_{i}=0,\quad \forall i\in[n],$$
and 
$$\forall i:y_{i}=1 \text{ and }\forall j\in[M_{0}],\quad \bm{u}_{j}^{\top}x_{i}=0.$$ 
Since $\bm{\bm{\theta}}^{*}$ is a local minima, then by Lemma~\ref{lemma::nec-single}, we have 
\begin{equation}\label{eq::thm2-1}
\sum_{i=1}^{n}\ell_{p}'(-y_{i}f(x_{i};\bm{\theta}^{*}))y_{i}\sigma'({\bm{w}^{*}_{j}}^{\top}x_{i})x_{i}=\bm{0}_{d}.
\end{equation}
Furthermore, by the  assumption that $K=r-\max\{r_{+},r_{-}\}>0$, then the set $\mathcal{U}_{+}\backslash\mathcal{U}_{-}$ is not an empty set. Thus, for  $\forall\bm{v}\in\mathcal{U}_{+}\backslash\mathcal{U}_{-}\subset\mathcal{E}$, with probability 1, $\forall i:y_{i}=-1$, $\bm{v}^{\top}x_{i}=0$. In addition, by the analysis presented in the proof of claim 1, we have that with probability 1,  $\bm{v}^{\top}x_{i}\neq 0$ for all $i:y_{i}=1$. Since $$K=r-\max\{r_{+},r_{-}\}=|\mathcal{U}_{+}\cup\mathcal{U}_{-}|-\max\{|\mathcal{U}_{+}|,|\mathcal{U}_{-}|\}=|\mathcal{U}_{+}\backslash\mathcal{U}_{-}|+|\mathcal{U}_{-}|-\max\{|\mathcal{U}_{+}|,|\mathcal{U}_{-}|\}\le |\mathcal{U}_{+}\backslash\mathcal{U}_{-}|,$$ then without loss of generality, we assume that $\{\bm{e}_{1},...,\bm{e}_{K}\}\subseteq\mathcal{U}_{+}\backslash\mathcal{U}_{-}$ and $\mathcal{U}_{+}=\{\bm{e}_{1},...,\bm{e}_{r_{+}}\}$. Thus, with probability 1, $\forall j\in[K]$, $\forall i:y_{i}=-1$, $\bm{e}_{j}^{\top}x_{i}=0$ and $\forall i:y_{i}=1$, $\bm{e}_{j}^{\top}x_{i}\neq 0$. Then by Equation \eqref{eq::thm2-1}, now we consider the following set of linear equations 
\begin{align*}
&\sum_{i=1}^{n}\ell_{p}'(-y_{i}f(x_{i};\bm{\theta}^{*}))y_{i}\sigma'({\bm{w}^{*}_{1}}^{\top}x_{i})\left(\bm{e}_{1}^{\top}x_{i}\right)=0,...,
\sum_{i=1}^{n}\ell_{p}'(-y_{i}f(x_{i};\bm{\theta}^{*}))y_{i}\sigma'({\bm{w}^{*}_{M_{0}}}^{\top}x_{i}+b_{M_{0}}^{*})\left(\bm{e}_{1}^{\top}x_{i}\right)=0,\\
&...\\
&\sum_{i=1}^{n}\ell_{p}'(-y_{i}f(x_{i};\bm{\theta}^{*}))y_{i}\sigma'({\bm{w}^{*}_{1}}^{\top}x_{i})\left(\bm{e}_{K}^{\top}x_{i}\right)=0,...,
\sum_{i=1}^{n}\ell_{p}'(-y_{i}f(x_{i};\bm{\theta}^{*}))y_{i}\sigma'({\bm{w}^{*}_{M_{0}}}^{\top}x_{i}+b_{M_{0}}^{*})\left(\bm{e}_{K}^{\top}x_{i}\right)=0.
\end{align*}
These equations can be rewritten in a matrix form
\begin{equation*}
\underbrace{
\left(\begin{matrix}
\sigma'({\bm{w}^{*}_{1}}^{\top}x_{1})\left(\bm{e}_{1}^{\top}x_{1}\right)&...&\sigma'({\bm{w}^{*}_{1}}^{\top}x_{n})\left(\bm{e}_{1}^{\top}x_{n}\right)\\
...&...&...\\
\sigma'({\bm{w}^{*}_{M_{0}}}^{\top}x_{1}+b_{M_{0}}^{*})\left(\bm{e}_{1}^{\top}x_{1}\right)&...&\sigma'({\bm{w}^{*}_{M_{0}}}^{\top}x_{n}+b_{M_{0}}^{*})\left(\bm{e}_{1}^{\top}x_{n}\right)\\
...&...&...\\
\sigma'({\bm{w}^{*}_{1}}^{\top}x_{1})\left(\bm{e}_{K}^{\top}x_{1}\right)&...&\sigma'({\bm{w}^{*}_{1}}^{\top}x_{n})\left(\bm{e}_{K}^{\top}x_{n}\right)\\
...&...&...\\
\sigma'({\bm{w}^{*}_{M_{0}}}^{\top}x_{1}+b_{M_{0}}^{*})\left(\bm{e}_{K}^{\top}x_{1}\right)&...&\sigma'({\bm{w}^{*}_{M_{0}}}^{\top}x_{n}+b_{M_{0}}^{*})\left(\bm{e}_{K}^{\top}x_{n}\right)
\end{matrix}\right)_{(KM_{0}\times n)}}_{\bm{P}}
\underbrace{\left(\begin{matrix}
\ell_{p}'(-y_{1}f(x_{1};\bm{\theta}^{*}))y_{1}\\
\ell_{p}'(-y_{2}f(x_{2};\bm{\theta}^{*}))y_{2}\\
...\\
...\\
...\\
...\\
...\\
\ell_{p}'(-y_{n}f(x_{1};\bm{\theta}^{*}))y_{n}\\
\end{matrix}\right)}_{\bm{q}}=\bm{0}_{n}
\end{equation*}
or 
$$\bm{P}\bm{q}=\bm{0}_{n}.$$
Since $M\ge \frac{2n}{\Delta r}=\frac{2n}{K}$, then 
$ M_{0}K\ge MK/2\ge n$. Clearly, if rank$(\bm{P})=n$, we should have $\bm{q}=\bm{0}_{n}$ and this indicates that $\ell'_{p}(-y_{i}f(x_{i};\bm{\theta}^{*}))=0$ for all $i\in[n]$ or $\error=0$. Thus, we only need to consider the case where rank$(\bm{P})<n\le M_{0}K$. This means the raw vectors of the matrix $\bm{P}$ is linearly dependent and thus  there exists coefficients vectors $(\beta_{11},...,\beta_{1K}),...,(\beta_{M_{0}1},...,\beta_{M_{0}K})$, not all zero vectors, such that 
$$\sum_{s=1}^{K}\sum_{j=1}^{M_{0}}\sigma'({\bm{w}_{j}^{*}}^{\top}x_{i})\beta_{js}(\bm{e}_{s}^{\top}x_{i})=0,\quad \forall i\in[n],$$
or 
$$\sum_{j=1}^{M_{0}}a_{j}^{*}\sigma'({\bm{w}_{j}^{*}}^{\top}x_{i})\left(\frac{1}{a_{j}^{*}}\sum_{s=1}^{K}\beta_{js}\bm{e}_{s}\right)^{\top}x_{i}=0,\quad \forall i\in[n],$$
by assumption that $a_{j}^{*}\neq 0$ for all $j=1,...,M_{0}$. 
Define $\bm{u}_{j}=\frac{1}{a_{j}^{*}}\sum_{s=1}^{K}\beta_{js}\bm{e}_{s}$ for $j=1,...,M_{0}$, then we have 
\begin{equation}\label{eq::thm-convex-all-zero}\sum_{j=1}^{M_{0}}a_{j}^{*}\sigma'({\bm{w}_{j}^{*}}^{\top}x_{i})\bm{u}_{j}^{\top}x_{i}=0,\quad \forall i\in[n].\end{equation}
Furthermore, since $\bm{u}_{j}\in \text{Span}(\{\bm{e}_{1},...,\bm{e}_{K}\})$ and with probability 1, $\bm{e}_{j}^{\top}x_{i}=0$, for $\forall i:y_{i}=-1$, $\forall j\in[K]$,  then 
$\forall j\in[M]$, $\forall i:y_{i}=-1$, $\bm{u}_{j}^{\top}x_{i}=0$. Thus, by setting $\bm{u}_{j}=\frac{1}{a_{j}^{*}}\sum_{s=1}^{K}\beta_{js}\bm{e}_{s}$ for $j=1,...,M_{0}$, then we have
\begin{align}
F(\bm{u}_{1},...,\bm{u}_{M_{0}})&=-\frac{1}{n}\sum_{i=1}^{n}\left[\ell_{p}'(-y_{i}f(x_{i};\bm{\theta}^{*}))y_{i}\sum_{j=1}^{M_{0}}\left[a_{j}^{*}\sigma''\left({\bm{w}_{j}^{*}}^{\top}x_{i}\right)\left(\bm{u}_{j}^{\top}x_{i}\right)^{2}\right]\right]\notag\\
&\quad +\frac{1}{n}\sum_{i=1}^{n}\left[\ell_{p}''(-y_{i}f(x_{i};\bm{\theta}^{*}))\left(\sum_{j=1}^{M_{0}}a_{j}^{*}\sigma'\left({\bm{w}_{j}^{*}}^{\top}x_{i}\right)\left(\bm{u}_{j}^{\top}x_{i}\right)\right)^{2}\right]\notag\\
&=-\frac{1}{n}\sum_{i=1}^{n}\left[\ell_{p}'(-y_{i}f(x_{i};\bm{\theta}^{*}))y_{i}\sum_{j=1}^{M_{0}}\left[a_{j}^{*}\sigma''\left({\bm{w}_{j}^{*}}^{\top}x_{i}\right)\left(\bm{u}_{j}^{\top}x_{i}\right)^{2}\right]\right]\notag&&\text{by Eq.~\eqref{eq::thm-convex-all-zero}}\\
&=-\frac{1}{n}\sum_{i:y_{i}=1}\left[\ell_{p}'(-y_{i}f(x_{i};\bm{\theta}^{*}))\sum_{j=1}^{M_{0}}\left[a_{j}^{*}\sigma''\left({\bm{w}_{j}^{*}}^{\top}x_{i}\right)\left(\bm{u}_{j}^{\top}x_{i}\right)^{2}\right]\right]\ge 0\label{eq::thm-convex-F}.
\end{align}

In addition, since $\sigma''(z)>0$ for all $z\in\mathbb{R}$ and $a_{j}^{*}>0$ for all $j\in[M_{0}]$, then we have 
$$\ell_{p}'(-y_{i}f(x_{i};\bm{\theta}^{*}))\sum_{j=1}^{M_{0}}\left[a_{j}^{*}\sigma''\left({\bm{w}_{j}^{*}}^{\top}x_{i}\right)\left(\bm{u}_{j}^{\top}x_{i}\right)^{2}\right]\ge 0,\quad \forall i:y_{i}=1$$
and this leads to 
$$F(\bm{u}_{1},...,\bm{u}_{M_{0}})\le 0.$$
Together with Eq.~\eqref{eq::thm-convex-F}, we have $$F(\bm{u}_{1},...,\bm{u}_{M_{0}})= 0,$$
and thus 
\begin{equation}\label{eq::thm1-cond-i}\ell_{p}'(-y_{i}f(x_{i};\bm{\theta}^{*}))\sum_{j=1}^{M_{0}}\left[a_{j}^{*}\sigma''\left({\bm{w}_{j}^{*}}^{\top}x_{i}\right)\left(\bm{u}_{j}^{\top}x_{i}\right)^{2}\right]= 0,\quad \forall i:y_{i}=1.\end{equation}
Now we split the index $\{1,...,n\}$ set into two disjoint subset $C_{0}, C_{1}$:
$$C_{0}=\{i\in[n]: y_{i}=1,\text{ and }\exists j\in[M_{0}],  \bm{u}_{j}^{\top}x_{i}\neq 0\},\quad C_{1}=\{i\in[n]:y_{i}=1 \text{ and }\forall j\in[M_{0}], \bm{u}_{j}^{\top}x_{i}= 0\}.$$
Clearly, for all $i\in C_{0}$, by the fact that $a^{*}_{j}> 0$ for all $j\in[M_{0}]$ and $\sigma''(z)>0$ for all $z\in\mathbb{R}$, we have 
$$\sum_{j=1}^{M_{0}}\left[a_{j}^{*}\sigma''\left({\bm{w}_{j}^{*}}^{\top}x_{i}\right)\left(\bm{u}_{j}^{\top}x_{i}\right)^{2}\right]>0,$$
and by Equation~\eqref{eq::thm1-cond-i}, we have  
$$\ell_{p}'(-y_{i}f(x_{i};\bm{\theta}^{*}))=0,\quad \forall i\in C_{0}.$$
Now we need to consider the index set $C_{1}$. First, we show that the following inequality holds with probability 1, $$|C_{1}|< r_{+}\le \max\{r_{+},r_{-}\}.$$ Since $\bm{u}_{j}=\frac{1}{a_{j}^{*}}\sum_{i=1}^{K}\beta_{js}\bm{e}_{s}$ for  $j=1,...,M_{0}$ and coefficient vectors $(\beta_{11},...,\beta_{1K}),...,(\beta_{M_{0}1},...,\beta_{M_{0}K})$ are not all zero vectors, then the there exists a $j_{0}\in[K]$ such that the non-zero vector $\bm{u}_{j_{0}}$ satisfy $\bm{u}_{j_{0}}^{\top}x_{i}=0$ for all $i\in C_{1}$ and  $\bm{u}_{j_{0}}\in\text{Span}(\{\bm{e}_{1},...,\bm{e}_{K}\})$. Furthermore, by assumption  $\mathcal{U}_{+}=\{\bm{e}_{1},...,\bm{e}_{r_{+}}\}$, thus we have 
\begin{equation}\label{eq::thm2-eq2}\bm{u}_{j_{0}}^{\top}x_{i}=\sum_{s=1}^{K}(\bm{u}_{j_{0}}^{\top}\bm{e}_{s})(x_{i}^{\top}\bm{e}_{s})=\sum_{s=1}^{r_{+}}(\bm{u}_{j_{0}}^{\top}\bm{e}_{s})(x_{i}^{\top}\bm{e}_{s})=0\end{equation} 
holds for all $i\in C_{1}$. If $|C_{1}|\ge r_{+}$, then without loss of generality, we assume that $\{1,...,r_{+}\}\subseteq C_{1}$. Thus, with probability 1, the matrix $$\left(\begin{matrix}\bm{e}_{1}^{\top}x_{1}&...&\bm{e}_{r_{+}}^{\top}x_{1}\\
...&...&...\\
\bm{e}_{1}^{\top}x_{r_{+}}&...&\bm{e}_{r_{+}}^{\top}x_{r_{+}}\\
\end{matrix}\right)=
\left(\begin{matrix}x_{1}^{\top}\\
...\\
x_{r_{+}}^{\top}\\
\end{matrix}\right)\left(\begin{matrix}\bm{e}_{1}&...&\bm{e}_{r_{+}}
\end{matrix}\right)$$
has a full rank equal to $r_{+}$, by the fact that $\{x_{1},...,x_{r_{+}}\}\subset\text{Span}(\mathcal{U}_{+})$ and $\left(x_{1},...,x_{r_{+}}\right)$ is a full rank matrix with probability 1.
 Thus, by Equation \eqref{eq::thm2-eq2}, we have
$$\left(\begin{matrix}\bm{e}_{1}^{\top}x_{1}&...&\bm{e}_{r_{+}}^{\top}x_{1}\\
...&...&...\\
\bm{e}_{1}^{\top}x_{r_{+}}&...&\bm{e}_{r_{+}}^{\top}x_{r_{+}}\\
\end{matrix}\right)
\left(\begin{matrix}
\bm{u}_{j_{0}}^{\top}\bm{e}_{1}\\
...\\
\bm{u}_{j_{0}}^{\top}\bm{e}_{r_{+}}
\end{matrix}\right)=\bm{0}_{d}
$$
and this leads to $\bm{u}_{j_{0}}^{\top}\bm{e}_{s}=0$ for all $s\in[K]$. This contradicts with the fact that $\bm{u}_{j_{0}}\in\text{Span}(\{\bm{e}_{1},...,\bm{e}_{K}\})$ and $\bm{u}_{j_{0}}$ is not a zero vector. Therefore, $|C_{1}|<r_{+}\le M_{0}$. Furthermore, since $\ell'(z)=0$ if and only if $z\le-z_{0}$ for some positive $z_{0}>0$, then $\ell''(z)=0$ when $z\le -z_{0}$.
Now we consider the function $F$, since $\forall i\in C_{0}:\ell_{p}'(-y_{i}f(x_{i};\bm{\theta}^{*}))=0$ and $\ell_{p}''(-y_{i}f(x_{i};\bm{\theta}^{*}))=0$, then
\begin{align}
F(\bm{u}_{1},...,\bm{u}_{M_{0}})&=-\frac{1}{n}\sum_{i\in C_{1}}\left[\ell_{p}'(-y_{i}f(x_{i};\bm{\theta}^{*}))\sum_{j=1}^{M_{0}}\left[a_{j}^{*}\sigma''\left({\bm{w}_{j}^{*}}^{\top}x_{i}\right)\left(\bm{u}_{j}^{\top}x_{i}\right)^{2}\right]\right]\notag\\
&\quad +\frac{1}{n}\sum_{i\in C_{1}}\left[\ell_{p}''(-y_{i}f(x_{i};\bm{\theta}^{*}))\left(\sum_{j=1}^{M_{0}}a_{j}^{*}\sigma'\left({\bm{w}_{j}^{*}}^{\top}x_{i}\right)\left(\bm{u}_{j}^{\top}x_{i}\right)\right)^{2}\right]\ge 0\notag
\end{align}
holds  for all $\bm{u}_{1},...,\bm{u}_{M_{0}}\in\text{Span}(\{\bm{e}_{1},...,\bm{e}_{K}\})$.
Now we set $\bm{u}_{j}=\alpha_{j} \bm{e}_{1}$, $j=1,...,M_{0}$ for some scalar $\alpha_{j}$. We only need to find $\alpha_{1},...,\alpha_{M_{0}}$ such that
$$\sum_{j=1}^{M_{0}}\alpha_{j}a_{j}^{*}\sigma'\left({\bm{w}_{j}^{*}}^{\top}x_{i}\right)\bm{e}_{1}^{\top}x_{i}=\bm{0},\quad \forall i\in C_{1}.$$
Since $|C_{1}|<r_{+}\le M_{0}$, then there exists $\alpha^{*}_{1},...,\alpha^{*}_{M_{0}}$, not all zeros, such that 
$$\sum_{j=1}^{M_{0}}\alpha^{*}_{j}a_{j}^{*}\sigma'\left({\bm{w}_{j}^{*}}^{\top}x_{i}\right)\bm{e}_{1}^{\top}x_{i}={0},\quad \forall i\in C_{1}.$$
Then by setting $\bm{u}_{j}=\alpha^{*}_{j} \bm{e}_{1}$, we have 
\begin{align*}
F(\bm{u}_{1},...,\bm{u}_{M_{0}})&=-\frac{1}{n}\sum_{i\in C_{1}}\left[\ell_{p}'(-y_{i}f(x_{i};\bm{\theta}^{*}))\sum_{j=1}^{M_{0}}\left[|\alpha_{j}^{*}|^{2}a_{j}^{*}\sigma''\left({\bm{w}_{j}^{*}}^{\top}x_{i}\right)\left(\bm{e}_{1}^{\top}x_{i}\right)^{2}\right]\right]\ge 0.\\
\label{eq::thm-F2}.
\end{align*}
Similarly, since $|\alpha_{1}|,...,|\alpha_{M_{0}}|$ are not all zeros, $a_{j}^{*}>0$ for all $j\in[M_{0}]$, $\sigma''(z)>0$ for all $z\in\mathbb{R}$ and $\bm{e}_{1}^{\top}x_{i}\neq 0$ holds for all $i$ with probability 1, then 
$$\ell_{p}'(-y_{i}f(x_{i};\bm{\theta}^{*}))=0,\quad \forall i\in C_{1}.$$
Therefore, this indicates that 
$$\ell_{p}'(-y_{i}f(x_{i};\bm{\theta}^{*}))=0,\quad  \forall i:y_{i}=1.$$
Furthermore, since $\bm{\theta}^{*}$ is a local minima and thus 
\begin{align*}
0=\frac{d\loss}{da_{0}}&=\frac{1}{n}\sum_{i=1}^{n}\ell_{p}'(-y_{i}f(x_{i};\bm{\theta}^{*}))(-y_{i})=-\frac{1}{n}\sum_{i:y_{i}=1}\ell_{p}'(-y_{i}f(x_{i};\bm{\theta}^{*}))+\frac{1}{n}\sum_{i:y_{i}=-1}\ell_{p}'(-y_{i}f(x_{i};\bm{\theta}^{*}))\\
&=\frac{1}{n}\sum_{i:y_{i}=-1}\ell_{p}'(-y_{i}f(x_{i};\bm{\theta}^{*})).
\end{align*}
This means when $\ell_{p}'(-y_{i}f(x_{i};\bm{\theta}^{*}))=0$ holds for all $i:y_{i}=1$, we have $\ell_{p}'(-y_{i}f(x_{i};\bm{\theta}^{*}))=0$ for all $i:y_{i}=-1$. These two together give us $\error =0$. Similarly, when $\sgn(a_{1})=...=\sgn(a_{M_{0}})=-1$, we have the similar the results. Therefore, $\bm{\theta}^{*}$ is a local minima  with $\error=0$.
\end{proof}

\clearpage

\subsection{Proof of Proposition~\ref{prop::results-quadratic}}\label{appendix::prop-results-quadratic}

\begin{proposition}
Assume that the loss function $\ell_{p}$ satisfies  assumption~\ref{assump::loss}, the distribution $\mathbb{P}_{\bm{X}\times Y}$ satisfies assumption~\ref{assump::full-rank} and \ref{assump::different-subspaces}, the network architecture satisfies assumption~\ref{assump::shortcut-connection} and neurons in the network satisfy assumption~\ref{assump::neurons}.
Assume that  samples in the dataset $\mathcal{D}=\{(x_{i},y_{i})\}_{i=1}^{n}, n\ge 1$ are independently drawn from the distribution $\mathbb{P}_{\bm{X}\times Y}$. Assume that the neuron $\sigma(z)=z^{2}$ and the number of neurons $M>r$. If the real parameters $\bm{\theta}^{*}=(\bm{\theta}_{S}^{*},\bm{\theta}_{D}^{*})$ denote a local minimum of the loss function $\hat{L}_{n}(\bm{\theta}_{S},\bm{\theta}_{D};p)$ and $p\ge 6$, then $\hat{R}_{n}(\bm{\theta}^{*})=\hat{L}_{n}(\bm{\theta}^{*};p)=0$ holds with probability one.
\end{proposition}

% differentiable
\begin{proof}
We first recall some notations defined in the paper. The output of the neural network is 
$$f(x;\bm{\theta})=f_{S}(x;\bm{\theta}_{S})+f_{D}(x;\bm{\theta}_{D}),$$
where $f_{S}(x;\bm{\theta}_{S})$ is the single layer neural network parameterized by $\bm{\theta}_{S}$, i.e., 
$$f_{S}(x;\bm{\theta}_{S})=a_{0}+\sum_{j=1}^{M}a_{j}\sigma\left(\bm{w}_{j}^{\top}x\right),$$
and $f_{D}(x;\bm{\theta}_{D})$ is a deep neural network parameterized by $\bm{\theta}_{D}$. 
The empirical loss function is given by
$$\hat{L}_{n}(\bm{\theta};p)=\hat{L}_{n}(\bm{\theta}_{S},\bm{\theta}_{D};p)=\frac{1}{n}\sum_{i=1}^{n}\ell_{p}(-y_{i}f(x_{i};\bm{\theta})).$$
We first assume that the $\bm{\theta}^{*}=(\bm{\theta}^{*}_{S},\bm{\theta}_{D}^{*})$ is a local minima.
We next prove the following two claims: 

\textbf{Claim 1:} If $\bm{\theta}^{*}=(\bm{\theta}_{S}^{*},\bm{\theta}_{D}^{*})$ is a local minima and there exists $j\in[M]$ such that $a^{*}_{j}=0$, then $\error =0$. 

\textbf{Claim 2:} If $\bm{\theta}^{*}=(\bm{\theta}_{S}^{*},\bm{\theta}_{D}^{*})$ is a local minima and $a^{*}_{j}\neq 0$ for all $j\in [M]$, then $\error =0$.

\textbf{(a) Proof of claim 1.} We prove that  if $\bm{\theta}^{*}=(\bm{\theta}_{S}^{*},\bm{\theta}_{D}^{*})$ is a local minima and there exists $j\in[M]$ such that $a^{*}_{j}=0$, then $\error =0$. Without loss of generality, we assume that $a_{1}^{*}=0$. Since $\bm{\theta}^{*}=(\bm{\theta}_{S}^{*},\bm{\theta}_{D}^{*})$ is a local minima, then there exists $\varepsilon_{0}>0$ such that for any small perturbations $\Delta{a}_{1}$, $\Delta \bm{w}_{1}$ on parameters $a^{*}_{1}$ and $ \bm{w}^{*}_{1}$, i.e., $|\Delta a_{1}|^{2}+\|\Delta\bm{w}_{1}\|_{2}^{2}\le \varepsilon_{0}^{2}$, we have 
$$\hat{L}_{n}(\tilde{\bm{\theta}}_{S},\bm{\theta}^{*}_{D})\ge \tilde{L}_{n}(\bm{\theta}^{*}_{S},\bm{\theta}_{D}^{*}),$$
where  $\tilde{\bm{\theta}}=(\tilde{a}_{0}, \tilde{a}_{1},...,\tilde{a}_{M},\tilde{\bm{w}}_{1},...,\tilde{\bm{w}}_{M})$, $\tilde{a}_{1}=a^{*}_{1}+\Delta a_{1}$, $\tilde{\bm{w}}_{1}=\bm{w}_{1}^{*}+\Delta \bm{w}_{1}$ and $\tilde{a}_{j}=a^{*}_{j}$, $\tilde{\bm{w}}_{j}=\bm{w}^{*}_{j}$ for $j\neq 1$.  Now we consider Taylor expansion of $\tilde{L}_{n}(\tilde{\bm{\theta}}_{S},\bm{\theta}^{*}_{D})$ at $(\bm{\theta}^{*}_{S},\bm{\theta}_{D}^{*})$.
We note here that the Taylor expansion of $\hat{L}(\bm{\theta}_{S},\bm{\theta}_{D}^{*};p)$ on $\bm{\theta}_{S}$ always exists, since the empirical loss function $\hat{L}_{n}$ has continuous derivatives with respect to $f_{S}$ up to the $p$-th order and the output of the neural network $f(x;\bm{\theta}_{S})$ is infinitely differentiable with respect to $\bm{\theta}_{S}$ due to the fact that neuron activation function $\sigma$ is real analytic.

We first calculate the first order derivatives at the point $(\bm{\theta}^{*}_{S},\bm{\theta}_{D}^{*})$
\begin{align*}
\frac{d\hat{L}_{n}(\bm{\theta}^{*})}{da_{1}}&=\frac{1}{n}\sum_{i=1}^{n}\ell_{p}'(-y_{i}f(x_{i};\bm{\theta}^{*}))(-y_{i})\sigma\left({\bm{w}_{1}^{*}}^{\top}x_{i}\right)=0,&& \text{$\bm{\theta}^{*}$ is a critical point,}\\
\nabla_{\bm{w}_{1}}\hat{L}_{n}(\bm{\theta}^{*})&=\frac{a^{*}_{1}}{n}\sum_{i=1}^{n}\ell_{p}'(-y_{i}f(x_{i};\bm{\theta}^{*}))(-y_{i})\sigma'\left({\bm{w}_{1}^{*}}^{\top}x_{i}\right)x_{i}=\bm{0}_{d},&& \text{$\bm{\theta}^{*}$ is a critical point.}
\end{align*}
Next, we calculate the second order derivatives at the point $(\bm{\theta}^{*}_{S},\bm{\theta}_{D}^{*})$,
\begin{align*}
\frac{d^{2}\hat{L}(\bm{\theta}^{*})}{da_{1}^{2}}&=\frac{1}{n}\sum_{i=1}^{N}\ell''_{p}(-y_{i}f(x_{i};\bm{\theta}^{*}))\sigma^{2}\left({\bm{w}_{1}^{*}}^{\top}x_{i}\right)\ge 0,\\
\frac{d}{da_{1}}(\nabla_{\bm{w}_{1}}L(\bm{\theta}^{*}))&=\frac{1}{n}\sum_{i=1}^{n}\ell_{p}'(-y_{i}f(x_{i};\bm{\theta}^{*}))(-y_{i})\sigma'\left({\bm{w}_{1}^{*}}^{\top}x_{i}\right)x_{i}\\
&\quad+\frac{a^{*}_{1}}{n}\sum_{i=1}^{n}\ell''_{p}(-y_{i}f(x_{i};\bm{\theta}^{*}))\sigma\left({\bm{w}_{1}^{*}}^{\top}x_{i}\right)\sigma'\left({\bm{w}_{1}^{*}}^{\top}x_{i}\right)x_{i}\\
&=\bm{0}_{d},
\end{align*}
where the first term equals to the zero vector by  the necessary condition for a local minima presented in Lemma~\ref{lemma::nec-single} and the second term equals to the zero vector by the assumption that $a^{*}_{1}=0$. Furthermore, by the assumption that $a^{*}_{1}=0$, we have 
\begin{equation*}
\nabla^{2}_{\bm{w}_{1}}\hat{L}_{n}(\bm{\theta}^{*};p)=\frac{a_{1}^{*}}{n}\nabla_{w_{1}}\left[\sum_{i=1}^{n}\ell_{p}'(-y_{i}f(x_{i};\bm{\theta}))(-y_{i})\sigma'\left({\bm{w}_{1}^{*}}^{\top}x_{i}\right)x_{i}\right]=\bm{0}_{d\times d}.
\end{equation*}
We further calculate the third order derivatives 
\begin{align*}
\frac{d}{da_{1}}\left[\nabla_{\bm{w}_{1}}^{2}{ \hat{L}_{n}(\bm{\theta}^{*};p)}\right]&=\frac{d}{da_{1}}\left[a_{1}^{*}\nabla_{\bm{w}_{1}}\left[\frac{1}{n}\sum_{i=1}^{n}\ell_{p}'(-y_{i}f(x_{i};\bm{\theta}))(-y_{i})\sigma'\left({\bm{w}_{1}^{*}}^{\top}x_{i}\right)x_{i}\right]\right]\\
&=\nabla_{\bm{w}_{1}}\left[\frac{1}{n}\sum_{i=1}^{n}\ell_{p}'(-y_{i}f(x_{i};\bm{\theta}))(-y_{i})\sigma'\left({\bm{w}_{1}^{*}}^{\top}x_{i}\right)x_{i}\right]+\bm{0}_{d\times d}&& \text{by $a_{1}^{*}=0$}\\
&=\frac{1}{n}\sum_{i=1}^{n}\ell_{p}'(-y_{i}f(x_{i};\bm{\theta}))(-y_{i})\sigma''\left({\bm{w}_{1}^{*}}^{\top}x_{i}\right)x_{i}x_{i}^{\top}\\
&\quad+\frac{a^{*}_{1}}{n}\sum_{i=1}^{n}\ell_{p}''(-y_{i}f(x_{i};\bm{\theta}))\left[\sigma'\left({\bm{w}_{1}^{*}}^{\top}x_{i}\right)\right]^{2}x_{i}x_{i}^{\top}\\
&=\frac{1}{n}\sum_{i=1}^{n}\ell_{p}'(-y_{i}f(x_{i};\bm{\theta}))(-y_{i})\sigma''\left({\bm{w}_{1}^{*}}^{\top}x_{i}\right)x_{i}x_{i}^{\top}&& \text{by $a_{1}^{*}=0$}
\end{align*}
and
$$\nabla^{3}_{\bm{w}_{1}}\hat{L}_{n}(\bm{\theta}^{*};p)=a^{*}_{1}\nabla^{2}_{\bm{w}_{1}}\left[\frac{1}{n}\sum_{i=1}^{n}\ell_{p}'(-y_{i}f(x_{i};\bm{\theta}))(-y_{i})\sigma'\left({\bm{w}_{1}^{*}}^{\top}x_{i}\right)x_{i}\right]=\bm{0}_{d\times d\times d}.$$
In fact, it is easy to show that for any $2\le k\le p$, 
$$\nabla^{k}_{\bm{w}_{1}}\hat{L}_{n}(\bm{\theta}^{*};p)=a_{1}^{*}\nabla^{k-1}_{\bm{w}_{1}}\left[\frac{1}{n}\sum_{i=1}^{n}\ell_{p}'(-y_{i}f(x_{i};\bm{\theta}))(-y_{i})\sigma'\left({\bm{w}_{1}^{*}}^{\top}x_{i}\right)x_{i}\right]=\bm{0}_{\underbrace{d\times d\times  ...\times d}_{ \text{$k$ times}}}.$$
Let $\varepsilon>0$, $\Delta a_{1}=\text{sgn}(a_{1})\varepsilon^{9/4}$ and $\Delta \bm{w}_{1}=\varepsilon \bm{u}_{1}$ for $\bm{u}_{1}:\|\bm{u}_{1}\|_{2}=1$. Clearly, when $\varepsilon\rightarrow 0$, $\Delta a_{1}=o(\|\Delta \bm{w}_{1}\|_{2})$, $\Delta a_{1}=o(1)$ and $\|\Delta \bm{w}_{1}\|=o(1)$. Then we expand $\hat{L}_{n}(\tilde{\bm{\theta}}_{S},\bm{\theta}_{D}^{*})$ at the point $\bm{\theta}^{*}$ up to the sixth order  and thus as $\varepsilon\rightarrow 0$,
\begin{align*}
\hat{L}_{n}(\tilde{\bm{\theta}}_{S}, \bm{\theta}_{D}^{*})&=\hat{L}_{n}({\bm{\theta}}^{*}_{S}, \bm{\theta}_{D}^{*})+\frac{1}{2!n}\frac{d^{2}\hat{L}_{n}(\bm{\theta}^{*})}{d^{2}a_{1}}(\Delta a_{1})^{2}\\
&\quad+\frac{1}{2n}\Delta a_{1}\Delta \bm{w}_{1}^{\top}\frac{d}{da_{1}}\left[\bm{D}_{\bm{w}_{1}}^{2}{ \hat{L}_{n}(\bm{\theta}^{*};p)}\right]\Delta \bm{w}_{1} + o(|a_{1}|^{2})+o(|a_{1}|\|\bm{w}_{1}\|^{2}_{2})+o(\|\Delta \bm{w}_{1}\|_{2}^{5})\\
&=\hat{L}_{n}({\bm{\theta}}^{*}_{S}, \bm{\theta}_{D}^{*})+\frac{1}{2!n}\frac{d^{2}\hat{L}_{n}(\bm{\theta}^{*})}{d^{2}a_{1}}\varepsilon^{9/2}+\frac{1}{2n}\text{sgn}(a_{1}) \varepsilon^{9/4+2}\sum_{i=1}^{n}\ell_{p}'(-y_{i}f(x_{i};\bm{\theta}))\sigma''\left({\bm{w}_{1}^{*}}^{\top}x_{i}\right)(\bm{u}_{1}^{\top}x_{i})^{2}\\
&\quad+o(\varepsilon^{9/2})+o(\varepsilon^{9/4+2})+o(\varepsilon^{5})\\
&=\hat{L}_{n}({\bm{\theta}}^{*}_{S}, \bm{\theta}_{D}^{*})+\frac{1}{2n}\text{sgn}(a_{1})\varepsilon^{17/4}\sum_{i=1}^{n}\ell_{p}'(-y_{i}f(x_{i};\bm{\theta}))(-y_{i})\sigma''\left({\bm{w}_{1}^{*}}^{\top}x_{i}\right)(\bm{u}_{1}^{\top}x_{i})^{2}+o(\varepsilon^{17/4})
\end{align*}
Since $\varepsilon>0$ and $\hat{L}_{n}(\tilde{\bm{\theta}}_{S},\bm{\theta}^{*}_{D};p)\ge \loss$ holds for any $\bm{u}_{1}:\|\bm{u}_{1}\|_{2}=1$ and any $\sgn(a_{1})\in\{-1, 1\}$, then 
\begin{equation}\label{eq::thm3-part1-cond}\sum_{i=1}^{n}\ell_{p}'(-y_{i}f(x_{i};\bm{\theta}))(-y_{i})\sigma''\left({\bm{w}_{1}^{*}}^{\top}x_{i}\right)(\bm{u}^{\top}x_{i})^{2}=0, \quad\text{for any } \bm{u}\in\mathbb{R}^{d}.\end{equation}
Therefore, 
\begin{equation*}
\sum_{i=1}^{n}\ell_{p}'(-y_{i}f(x_{i};\bm{\theta}))(-y_{i})\sigma''\left({\bm{w}_{1}^{*}}^{\top}x_{i}\right)x_{i}x_{i}^{\top}=\bm{0}_{d\times d}.
\end{equation*}
By assumption that there exists a set of orthogonal basis $\mathcal{E}=\{\bm{e}_{1},...,\bm{e}_{d}\}$ in $\mathbb{R}^{d}$ and a subset $\mathcal{U}_{+}\subseteq \mathcal{E}$  such that $\mathbb{P}_{\bm{X}|Y}(\bm{X}\in\text{Span}(\mathcal{U}_{1})|Y=1)=1$ and by assumption that $r=|\mathcal{U}_{+}\cup \mathcal{U}_{-}|>\max\{r_{+},r_{-}\}=\max\{|\mathcal{U}_{+}|,|\mathcal{U}_{-}|\}$, then the set $\mathcal{U}_{+}\backslash \mathcal{U}_{-}$ is not an empty set. It is easy to show that for any vector $\bm{v}\in\mathcal{U}_{+}\backslash\mathcal{U}_{-}$, $\mathbb{P}_{\bm{X}\times Y}(\bm{v}^{\top}\bm{X}=0|Y=1)=0$. Otherwise, if $p=\mathbb{P}_{\bm{X}\times Y}(\bm{v}^{\top}\bm{X}=0|Y=1)>0$, then for random vectors $\bm{X}_{1},...,\bm{X}_{|\mathcal{U}_{+}|}$ independently drawn from the conditional distribution $\mathbb{P}_{\bm{X}| Y=1}$, 
\begin{align*}
\mathbb{P}_{\bm{X} |Y=1}\left(\bigcup_{i=1}^{|\mathcal{U}_{+}|}\left\{\bm{v}^{\top}\bm{X}_{i}=0\right\}\Bigg|Y=1\right)&=\prod_{i=1}^{|\mathcal{U}_{+}|}\mathbb{P}_{\bm{X}|Y=1}\left(\bm{v}^{\top}\bm{X}_{i}=0|Y=1\right)=p^{|\mathcal{U}_{+}|}>0.
\end{align*} 
Furthermore, since $\bm{X}_{1},...,\bm{X}_{|\mathcal{U}_{+}|}\in\text{Span}(\mathcal{U}_{+})$, $\bm{v}^{\top}\bm{X}_{i}=0$, $i=1,...,|\mathcal{U}_{+}|$ and $\bm{v}\in\mathcal{U}_{+}$, then the rank of the matrix $\left(\bm{X}_{1},...,\bm{X}_{|\mathcal{U}_{+}|}\right)$ is at most $|\mathcal{U}_{+}|-1$ and this indicates that the matrix is not a full rank matrix with probability $p^{|\mathcal{U}_{+}|}>0$. This leads to the contradiction with the Assumption~\ref{assump::full-rank}. Thus, with probability 1, $\bm{v}^{\top}x_{i}\neq 0$ for all $i:y_{i}=1$ and $\bm{v}^{\top}x_{i}= 0$ for all $i:y_{i}=-1$.

Therefore, by setting $\bm{u}=\bm{v}$ in Equation~\eqref{eq::thm3-part1-cond}, we have 
\begin{align*}
0=-\sum_{i:y_{i}=1}\ell'_{p}(-y_{i}f(x_{i};\bm{\theta}^{*}))\sigma''({\bm{w}_{1}^{*}}^{\top}x_{i})(\bm{v}^{\top}x_{i})^{2}\le 0,
\end{align*}
where the equality holds if and only if $\forall i: y_{i}=1$, $\ell_{p}'(-y_{i}f(x_{i};\bm{\theta}^{*}))=0$ and this further indicates that $\forall i: y_{i}=1$, $y_{i}f(x_{i};\bm{\theta}^{*})\ge z_{0}>0$. 
Furthermore, since $\bm{\theta}^{*}$ is a critical point and thus 
\begin{align*}
0=\frac{d\loss}{da_{0}}&=\frac{1}{n}\sum_{i=1}^{n}\ell_{p}'(-y_{i}f(x_{i};\bm{\theta}^{*}))(-y_{i})=-\frac{1}{n}\sum_{i:y_{i}=1}\ell_{p}'(-y_{i}f(x_{i};\bm{\theta}^{*}))+\frac{1}{n}\sum_{i:y_{i}=-1}\ell_{p}'(-y_{i}f(x_{i};\bm{\theta}^{*}))\\
&=\frac{1}{n}\sum_{i:y_{i}=-1}\ell_{p}'(-y_{i}f(x_{i};\bm{\theta}^{*})).
\end{align*}
Therefore, $\forall i: y_{i}=-1$, $y_{i}f(x_{i};\bm{\theta}^{*})\ge z_{0}>0$ and this indicates that $\hat{R}_{n}(\bm{\theta}^{*})=0.$

\textbf{(b) Proof of Claim 2:} To prove the claim 2, we first prove that if $M>r$, then there exists coefficients $\alpha_{1},...,\alpha_{M}$, not all zero, such that $$\left(\alpha_{1}\bm{w}_{1}^{*}+...+\alpha_{M}\bm{w}_{M}^{*}\right)^{\top}x_{i}=0,\quad \text{for all }i\in[n].$$
Since we assume that $\mathcal{U}_{+}\subseteq \mathcal{E}$ and $\mathcal{U}_{-}\subseteq \mathcal{E}$  such that $\mathbb{P}_{\bm{X}|Y}(\bm{X}\in\text{Span}(\mathcal{U}_{+})|Y=1)=1$ and $\mathbb{P}_{\bm{X}|Y}(\bm{X}\in\text{Span}(\mathcal{U}_{-})|Y=-1)=1$,
then without loss generality, we assume that $x_{i}$s locate in the linear span of $\{\bm{e}_{1},...,\bm{e}_{r}\}\subseteq\{\bm{e}_{1},...,\bm{e}_{d}\}$ (note that $r=|\mathcal{U}_{+}\cup\mathcal{U}_{-}|$). Clearly, for any $\bm{w}_{1}^{*},...,\bm{w}_{M}^{*}$, if $M> r$, then there exists coefficients $\alpha_{1},...,\alpha_{M}$, not all zero, such that 
\begin{align*}
&\alpha_{1}\bm{w}_{1}^{*}+...+\alpha_{M}\bm{w}_{M}^{*}\in\text{Span}(\{\bm{e}_{r+1},...,\bm{e}_{d}\}), &&\text{if } r<d,\\
&\alpha_{1}\bm{w}_{1}^{*}+...+\alpha_{M}\bm{w}_{M}^{*}=\bm{0}_{d}, &&\text{if } r=d.
\end{align*}
Therefore, if $M>r$, then there exists coefficients $\alpha_{1},...,\alpha_{M}$, not all zero, such that $$(\alpha_{1}\bm{w}_{1}^{*}+...+\alpha_{M}\bm{w}_{M}^{*})^{\top}x_{i}=0,\quad \text{for all }i\in[n].$$
Now we prove the claim 2. First, we consider the Hessian matrix $H(\bm{w}_{1}^{*},...,\bm{w}_{M}^{*})$. Since $\bm{\theta}^{*}$ is a local minima, then 
\begin{equation*}
F(\bm{u}_{1},...,\bm{u}_{M})=\sum_{j=1}^{M}\sum_{k=1}^{M}\bm{u}_{j}^{\top}\nabla^{2}_{\bm{w}_{j},\bm{w}_{k}}\loss \bm{u}_{k}\ge 0
\end{equation*}
holds for any vectors $\bm{u}_{1},...,\bm{u}_{M}\in\mathbb{R}^{d}$. 
Since $\sigma''(z)=2$ and $\sigma'(z)=2z$ for all $z\in\mathbb{R}$, then
\begin{align*}
\nabla_{\bm{w}_{j}}^{2}\loss&=\frac{a_{j}^{*}}{n}\sum_{i=1}^{n}\ell_{p}'(-y_{i}f(x_{i};\bm{\theta}^{*}))(-y_{i})\sigma''\left({\bm{w}_{j}^{*}}^{\top}x_{i}\right)x_{i}x_{i}^{\top}\\
&\quad +\frac{{a_{j}^{*}}^{2}}{n}\sum_{i=1}^{n}\ell_{p}''(-y_{i}f(x_{i};\bm{\theta}^{*}))\left[\sigma'\left({\bm{w}_{j}^{*}}^{\top}x_{i}\right)\right]^{2}x_{i}x_{i}^{\top}\\
&=-\frac{2a_{j}^{*}}{n}\sum_{i=1}^{n}\ell_{p}'(-y_{i}f(x_{i};\bm{\theta}^{*}))y_{i}x_{i}x_{i}^{\top}+\frac{4{a_{j}^{*}}^{2}}{n}\sum_{i=1}^{n}\ell_{p}''(-y_{i}f(x_{i};\bm{\theta}^{*}))\left({\bm{w}_{j}^{*}}^{\top}x_{i}\right)^{2}x_{i}x_{i}^{\top},
\end{align*}
and 
\begin{align*}
\nabla_{\bm{w}_{j},\bm{w}_{k}}^{2}\loss&=\frac{{a_{j}^{*}}a_{k}^{*}}{n}\sum_{i=1}^{n}\ell_{p}''(-y_{i}f(x_{i};\bm{\theta}^{*}))\left[\sigma'\left({\bm{w}_{j}^{*}}^{\top}x_{i}\right)\right]\left[\sigma'\left({\bm{w}_{k}^{*}}^{\top}x_{i}\right)\right]x_{i}x_{i}^{\top}\\
&=\frac{4{a_{j}^{*}}a_{k}^{*}}{n}\sum_{i=1}^{n}\ell_{p}''(-y_{i}f(x_{i};\bm{\theta}^{*}))\left({\bm{w}_{j}^{*}}^{\top}x_{i}\right)\left({\bm{w}_{k}^{*}}^{\top}x_{i}\right)x_{i}x_{i}^{\top}.
\end{align*}
Thus, we have 
\begin{align*}
F(\bm{u}_{1},...,\bm{u}_{M})&=-2\sum_{j=1}^{M}\left[\frac{a_{j}^{*}}{n}\sum_{i=1}^{n}\ell_{p}'(-y_{i}f(x_{i};\bm{\theta}^{*}))y_{i}\left(\bm{u}_{j}^{\top}x_{i}\right)^{2}\right]\\
&\quad +4\sum_{j=1}^{M}\sum_{k=1}^{M}\left[\frac{{a_{j}^{*}}a_{k}^{*}}{n}\sum_{i=1}^{n}\ell_{p}''(-y_{i}f(x_{i};\bm{\theta}^{*}))\left({\bm{w}_{j}^{*}}^{\top}x_{i}\right)\left({\bm{w}_{k}^{*}}^{\top}x_{i}\right)\left(\bm{u}_{j}^{\top}x_{i}\right)\left(\bm{u}_{k}^{\top}x_{i}\right)\right]\\
&=-\frac{2}{n}\sum_{j=1}^{M}\left[a_{j}^{*}\sum_{i=1}^{n}\ell_{p}'(-y_{i}f(x_{i};\bm{\theta}^{*}))y_{i}\left(\bm{u}_{j}^{\top}x_{i}\right)^{2}\right]\\
&\quad +\frac{4}{n}\sum_{i=1}^{n}\left[\ell_{p}''(-y_{i}f(x_{i};\bm{\theta}^{*}))\left(\sum_{j=1}^{M}a_{j}^{*}\left({\bm{w}_{j}^{*}}^{\top}x_{i}\right)\left(\bm{u}_{j}^{\top}x_{i}\right)\right)^{2}\right].
\end{align*}
Since there exists coefficients $\alpha_{1},...,\alpha_{M}$, not all zero, such that $(\alpha_{1}\bm{w}_{1}^{*}+...+\alpha_{M}\bm{w}_{M}^{*})^{\top}x_{i}=0$, for all $i\in[n],$ and $a_{j}^{*}\neq 0$ for all $j\in[M]$ then by setting $\bm{u}_{j}=\alpha_{j}\bm{u}/a_{j}^{*}$ for all $j\in[M]$, we have that the inequality
\begin{align*}
F(\bm{u}_{1},...,\bm{u}_{M})&=-\frac{2}{n}\sum_{j=1}^{M}\left[a_{j}^{*}\sum_{i=1}^{n}\ell_{p}'(-y_{i}f(x_{i};\bm{\theta}^{*}))y_{i}\left(\alpha_{j}/a_{j}^{*}\right)^{2}\left(\bm{u}^{\top}x_{i}\right)^{2}\right]\\
&\quad +\frac{4}{n}\sum_{i=1}^{n}\left[\ell_{p}''(-y_{i}f(x_{i};\bm{\theta}^{*}))\left(\sum_{j=1}^{M}\alpha_{j}\left({\bm{w}_{j}^{*}}^{\top}x_{i}\right)\left(\bm{u}^{\top}x_{i}\right)\right)^{2}\right]\\
&=-\frac{2}{n}\sum_{j=1}^{M}\left[a_{j}^{*}\sum_{i=1}^{n}\ell_{p}'(-y_{i}f(x_{i};\bm{\theta}^{*}))y_{i}\left(\alpha_{j}/a_{j}^{*}\right)^{2}\left(\bm{u}^{\top}x_{i}\right)^{2}\right]\\
&\quad +\frac{4}{n}\sum_{i=1}^{n}\left[\ell_{p}''(-y_{i}f(x_{i};\bm{\theta}^{*}))\left(\left(\sum_{j=1}^{M}\alpha_{j}{\bm{w}_{j}^{*}}\right)^{\top}x_{i}\right)^{2}\left(\bm{u}^{\top}x_{i}\right)^{2}\right]\\
&=-\frac{2}{n}\sum_{j=1}^{M}\left(\alpha_{j}^{2}/a_{j}^{*}\right)\cdot\sum_{i=1}^{n}\ell_{p}'(-y_{i}f(x_{i};\bm{\theta}^{*}))y_{i}\left(\bm{u}^{\top}x_{i}\right)^{2}\ge 0
\end{align*}
holds for any $\bm{u}\in\mathbb{R}^{d}$.

Next we consider the following two cases: (1) $\sum_{j=1}^{M}\left(\alpha_{j}^{2}/a_{j}^{*}\right)\neq 0$; (2) $\sum_{j=1}^{M}\left(\alpha_{j}^{2}/a_{j}^{*}\right)=0$. 

\textbf{Case 1: } If $\sum_{j=1}^{M}\left(\alpha_{j}^{2}/a_{j}^{*}\right)\neq 0$, then without loss of generality, we assume that $\sum_{j=1}^{M}\left(\alpha_{j}^{2}/a_{j}^{*}\right)<0$. This indicates that $$\sum_{i=1}^{n}\ell_{p}'(-y_{i}f(x_{i};\bm{\theta}^{*}))y_{i}\left(\bm{u}^{\top}x_{i}\right)^{2}\ge 0,\quad \text{for all }\bm{u}\in\mathbb{R}^{d}.$$
By the assumption that there exists two vectors $\bm{e}_{r},\bm{e}_{s}$ such that $\forall i:y_{i}=1$, $\bm{e}_{r}^{\top}x_{i}=0$, $\bm{e}^{\top}_{s}x_{i}\neq 0$ hold with probability 1 and $\forall i:y_{i}=-1$, $\bm{e}_{s}^{\top}x_{i}=0$, $\bm{e}^{\top}_{r}x_{i}\neq 0$ hold with probability 1, then by setting $\bm{u}=\bm{e}_{r}$, we have that 
\begin{align*}
0\le \sum_{i=1}^{n}\ell_{p}'(-y_{i}f(x_{i};\bm{\theta}^{*}))y_{i}\left(\bm{e}_{r}^{\top}x_{i}\right)^{2}&=-\sum_{i:y_{i}=-1}\ell_{p}'(-y_{i}f(x_{i};\bm{\theta}^{*}))\left(\bm{e}_{r}^{\top}x_{i}\right)^{2}\le 0,
\end{align*}
where the equality holds if and only if $\ell_{p}'(-y_{i}f(x_{i};\bm{\theta}^{*}))=0$ or $y_{i}f(x_{i};\bm{\theta}^{*})\ge z_{0}>0$ holds for all $i:y_{i}=-1$. Furthermore, since $\bm{\theta}^{*}$ is a local minima and thus 
\begin{align*}
0=\frac{d\loss}{da_{0}}&=\sum_{i=1}^{n}\ell_{p}'(-y_{i}f(x_{i};\bm{\theta}^{*}))(-y_{i})=-\sum_{i:y_{i}=1}\ell_{p}'(-y_{i}f(x_{i};\bm{\theta}^{*}))+\sum_{i:y_{i}=-1}\ell_{p}'(-y_{i}f(x_{i};\bm{\theta}^{*}))\\
&=-\sum_{i:y_{i}=1}\ell_{p}'(-y_{i}f(x_{i};\bm{\theta}^{*})).
\end{align*}
This means when $\ell_{p}'(-y_{i}f(x_{i};\bm{\theta}^{*}))=0$ holds for all $i:y_{i}=-1$, we have $\ell_{p}'(-y_{i}f(x_{i};\bm{\theta}^{*}))=0$ for all $i:y_{i}=1$. These two together give us $\error =0$. When $\sum_{j=1}^{M}\left(\alpha_{j}^{2}/a_{j}^{*}\right)>0$, by setting $\bm{u}=\bm{e}_{s}$ and  following the similar analysis presented above, we can obtain the same result. Therefore, when $\sum_{j=1}^{M}\left(\alpha_{j}^{2}/a_{j}^{*}\right)\neq 0$, we have $\error =0$. 

\textbf{{Case 2}:} If $\sum_{j=1}^{M}\left(\alpha_{j}^{2}/a_{j}^{*}\right)= 0$,  then by setting $\bm{u}_{j}=(\alpha_{j}/a_{j}^{*}+v\sgn(\alpha_{j}))\bm{u}$ for some scalar $v$ and vector $\bm{u}\in\mathbb{R}^{d}$, we have 
\begin{align*}
F(v,\bm{u})&=-\frac{2}{n}\sum_{j=1}^{M}\left[a_{j}^{*}\sum_{i=1}^{n}\ell_{p}'(-y_{i}f(x_{i};\bm{\theta}^{*}))y_{i}\left((\alpha_{j}/a_{j}^{*}+v\sgn(\alpha_{j}))\bm{u}^{\top}x_{i}\right)^{2}\right]\\
&\quad +\frac{4}{n}\sum_{i=1}^{n}\left[\ell_{p}''(-y_{i}f(x_{i};\bm{\theta}^{*}))\left(\sum_{j=1}^{M}a_{j}^{*}\left({\bm{w}_{j}^{*}}^{\top}x_{i}\right)\left((\alpha_{j}/a_{j}^{*}+v\sgn(\alpha_{j}))\bm{u}^{\top}x_{i}\right)\right)^{2}\right]\\
&=-\frac{2}{n}\sum_{j=1}^{M}\left[a_{j}^{*}\sum_{i=1}^{n}\ell_{p}'(-y_{i}f(x_{i};\bm{\theta}^{*}))y_{i}\left((\alpha_{j}/a_{j}^{*}+v\sgn(\alpha_{j}))\bm{u}^{\top}x_{i}\right)^{2}\right]\\
&\quad +\frac{4}{n}\sum_{i=1}^{n}\left[\ell_{p}''(-y_{i}f(x_{i};\bm{\theta}^{*}))\left(\left(\sum_{j=1}^{M}(\alpha_{j}+v\sgn(\alpha_{j})a^{*}_{j})\bm{w}_{j}^{*}\right)^{\top}x_{i}\right)\left(\bm{u}^{\top}x_{i}\right)^{2}\right]\\
&=-\frac{2}{n}\sum_{j=1}^{M}\left[a_{j}^{*}\sum_{i=1}^{n}\ell_{p}'(-y_{i}f(x_{i};\bm{\theta}^{*}))y_{i}\left((\alpha_{j}/a_{j}^{*}+v\sgn(\alpha_{j}))\bm{u}^{\top}x_{i}\right)^{2}\right]\\
&\quad +4v^{2}\sum_{i=1}^{n}\left[\ell_{p}''(-y_{i}f(x_{i};\bm{\theta}^{*}))\left(\left(\sum_{j=1}^{M}\sgn(\alpha_{j})a_{j}^{*}\bm{w}_{j}^{*}\right)^{\top}x_{i}\right)^{2}\left(\bm{u}^{\top}x_{i}\right)^{2}\right]\\
&\triangleq  -\frac{2}{n}\sum_{j=1}^{M}\left[a_{j}^{*}\sum_{i=1}^{n}\ell_{p}'(-y_{i}f(x_{i};\bm{\theta}^{*}))y_{i}\left((\alpha_{j}/a_{j}^{*}+v\sgn(\alpha_{j}))\bm{u}^{\top}x_{i}\right)^{2}\right]+v^{2}R(\bm{u}),
\end{align*}
where we define 
$$R(\bm{u})=\frac{4}{n}\sum_{i=1}^{n}\left[\ell_{p}''(-y_{i}f(x_{i};\bm{\theta}^{*}))\left(\left(\sum_{j=1}^{M}\sgn(\alpha_{j})a_{j}^{*}\bm{w}_{j}^{*}\right)^{\top}x_{i}\right)^{2}\left(\bm{u}^{\top}x_{i}\right)^{2}\right].$$
In addition, we have 
\begin{align*}
\sum_{j=1}^{M}&\left[a_{j}^{*}\sum_{i=1}^{n}\ell_{p}'(-y_{i}f(x_{i};\bm{\theta}^{*}))y_{i}\left((\alpha_{j}/a_{j}^{*}+v\sgn(\alpha_{j}))\bm{u}^{\top}x_{i}\right)^{2}\right]\\
&=\sum_{i=1}^{n}\ell'_{p}(-y_{i}f(x_{i};\bm{\theta}))y_{i}(\bm{u}^{\top}x_{i})^{2}\cdot\left[\sum_{j=1}^{M}(\alpha_{j}^{2}/a_{j}^{*}+2v\sgn(\alpha_{j})\alpha_{j}+v^{2}a_{j}^{*})\right]\\
&=\sum_{i=1}^{n}\ell'_{p}(-y_{i}f(x_{i};\bm{\theta}))y_{i}(\bm{u}^{\top}x_{i})^{2}\cdot\left[\sum_{j=1}^{M}(2v\sgn(\alpha_{j})\alpha_{j}+v^{2}a_{j}^{*})\right]\\
&=2v\left[\sum_{j=1}^{M}|\alpha_{j}|\right]\sum_{i=1}^{n}\ell'_{p}(-y_{i}f(x_{i};\bm{\theta}))y_{i}(\bm{u}^{\top}x_{i})^{2}+v^{2}\left[\sum_{j=1}^{M}a_{j}^{*}\right]\sum_{i=1}^{n}\ell'_{p}(-y_{i}f(x_{i};\bm{\theta}))y_{i}(\bm{u}^{\top}x_{i})^{2}.
\end{align*}
Therefore, we can rewrite $F(v, \bm{u})$ as 
\begin{align*}
F(v,\bm{u})&=-\frac{4v}{n}\sum_{j=1}^{M}|\alpha_{j}|\sum_{i=1}^{n}\ell'_{p}(-y_{i}f(x_{i};\bm{\theta}))y_{i}(\bm{u}^{\top}x_{i})^{2}-\frac{2v^{2}}{n}\sum_{j=1}^{M}a_{j}^{*}\cdot\sum_{i=1}^{n}\ell'_{p}(-y_{i}f(x_{i};\bm{\theta}))y_{i}(\bm{u}^{\top}x_{i})^{2}+v^{2}R(\bm{u})\\
&\triangleq-\frac{4v}{n}\sum_{j=1}^{M}|\alpha_{j}|\sum_{i=1}^{n}\ell'_{p}(-y_{i}f(x_{i};\bm{\theta}))y_{i}(\bm{u}^{\top}x_{i})^{2}+v^{2}\hat{R}(\bm{u})
\end{align*}
Since $F(\bm{v},\bm{u})\ge 0$ holds for any scalar $v$ and vector $\bm{u}\in\mathbb{R}^{d}$, then we should have 
$$\sum_{j=1}^{M}|\alpha_{j}|\sum_{i=1}^{n}\ell'_{p}(-y_{i}f(x_{i};\bm{\theta}))y_{i}(\bm{u}^{\top}x_{i})^{2}=0,\quad \text{ for any }\bm{u}\in\mathbb{R}^{d}. $$
Since the coefficient $\alpha_{1},...,\alpha_{M}$ are not all zero, then for any $\bm{u}\in\mathbb{R}^{d}$, we have 
$$\sum_{i=1}^{n}\ell'_{p}(-y_{i}f(x_{i};\bm{\theta}))y_{i}(\bm{u}^{\top}x_{i})^{2}=0. $$
Since there exists two vectors $\bm{e}_{r},\bm{e}_{s}$: $\forall i:y_{i}=1$, $\bm{e}_{r}^{\top}x_{i}=0$ and $\bm{e}^{\top}_{s}x_{i}\neq 0$ hold with probability 1 and $\forall i:y_{i}=-1$, $\bm{e}_{s}^{\top}x_{i}=0$ and $\bm{e}^{\top}_{r}x_{i}\neq 0$ hold with probability 1, then by setting $\bm{u}=\bm{e}_{r}$, we have 
\begin{align*}
0=\sum_{i=1}^{n}\ell'_{p}(-y_{i}f(x_{i};\bm{\theta}))y_{i}(\bm{e}_{r}^{\top}x_{i})^{2}=-\sum_{i:y_{i}=-1}\ell'_{p}(-y_{i}f(x_{i};\bm{\theta}))(\bm{e}_{r}^{\top}x_{i})^{2}\le 0,
\end{align*}
where the equality holds if and only if $\ell_{p}'(-y_{i}f(x_{i};\bm{\theta}^{*}))=0$ or $y_{i}f(x_{i};\bm{\theta}^{*})\ge z_{0}>0$ holds for all $i:y_{i}=-1$. Similar to the case 1, we have that $\ell_{p}'(-y_{i}f(x_{i};\bm{\theta}^{*}))=0$ holds for all $i$ and this leads to $\error =0$.  
\end{proof}

\clearpage
\subsection{Proof of Theorem~\ref{thm::linear-sep-deep}}\label{appendix::linear-sep}
\begin{theorem}
 Assume that the loss function $\ell_{p}$ satisfies  assumption~\ref{assump::loss} and the network architecture satisfies assumption~\ref{assump::shortcut-connection}.
 Assume that  samples in the dataset $\mathcal{D}=\{(x_{i},y_{i})\}_{i=1}^{n}, n\ge 1$ are independently drawn from a distribution satisfying assumption~\ref{assump::linear-sep}. Assume that the single layer network $f_{S}$ has $M\ge1$ neurons and neurons $\sigma$ in the network $f_{S}$ are twice differentiable and satisfy $\sigma'(z)>0$ for all $z\in\mathbb{R}$. If a set of real parameters $\bm{\theta}^{*}=(\bm{\theta}^{*}_{S},\bm{\theta}^{*}_{D})$ denotes a local minimum of the loss function $\hat{L}_{n}(\bm{\theta}^{}_{S},\bm{\theta}^{}_{D};p)$, $p\ge 3$, then  $\hat{R}_{n}(\bm{\theta}^{*}_{S},\bm{\theta}^{*}_{D})=0$ holds with probability one.
\end{theorem}

\begin{proof}
We first recall some notations defined in the paper. The output of the neural network is 
$$f(x;\bm{\theta})=f_{S}(x;\bm{\theta}_{S})+f_{D}(x;\bm{\theta}_{D}),$$
where $f_{S}(x;\bm{\theta}_{S})$ is the single layer neural network parameterized by $\bm{\theta}_{S}$, i.e., 
$$f_{S}(x;\bm{\theta}_{S})=a_{0}+\sum_{j=1}^{M}a_{j}\sigma\left(\bm{w}_{j}^{\top}x\right),$$
and $f_{D}(x;\bm{\theta}_{D})$ is a deep neural network parameterized by $\bm{\theta}_{D}$. 
The empirical loss function is given by
$$\hat{L}_{n}(\bm{\theta};p)=\hat{L}_{n}(\bm{\theta}_{S},\bm{\theta}_{D};p)=\frac{1}{n}\sum_{i=1}^{n}\ell_{p}(-y_{i}f(x_{i};\bm{\theta})).$$
By the assumption that $\bm{\theta}^{*}=(\bm{\theta}_{S}^{*},\bm{\theta}_{D}^{*})$ is a local minima and by the necessary condition presented in Lemma~\ref{lemma::nec-single}, we have 
\begin{align*}
&\sum_{i=1}^{n}\ell_{p}'(-y_{i}f(x_{i};\bm{\theta}^{*}))y_{i}\sigma'({\bm{w}^{*}_{j}}^{\top}x_{i})x_{i}=\bm{0}_{d}.
\end{align*}
Thus,  for any $\bm{w}\in\mathbb{R}^{d}$ and any $j\in[M]$, we have
\begin{equation*}
\sum_{i=1}^{n}\ell_{p}'(-y_{i}f(x_{i};\bm{\theta}^{*}))\sigma'({\bm{w}^{*}_{j}}^{\top}x_{i})y_{i}(\bm{w}^{\top}x_{i})=0.
\end{equation*}
Furthermore, by assumption $$\ell'_{p}(z)\ge 0$$
and the equality holds if and only if $z\le -z_{0}$.
Thus, by assumption that $\sigma'(z)>0$ for all $z\in\mathbb{R}$ and assumption that there exists a vector $\mathbb{P}_{\bm{X}\times Y}(Y\bm{w}^{\top}X>0)=1$, then there exists and positive constant $c>0$ such that 
$$y_{i}(\bm{w}^{\top}x_{i})>c>0,\quad \forall i\in[n].$$ Thus, we have 
\begin{align*}
0=\sum_{i=1}^{n}\ell_{p}'(-y_{i}f(x_{i};\bm{\theta}^{*}))\sigma'({\bm{w}^{*}_{j}}^{\top}x_{i})y_{i}(\bm{w}^{\top}x_{i})\ge c\sum_{i=1}^{n}\ell_{p}'(-y_{i}f(x_{i};\bm{\theta}^{*}))\sigma'({\bm{w}^{*}_{j}}^{\top}x_{i})\ge 0,
\end{align*}
where the equality holds if and only if $\ell'_{p}(-y_{i}f(x_{i};\bm{\theta}^{*}))=0$ for all $i\in[n]$. Equivalently, if $\bm{\theta}^{*}$ is a local minima, then $y_{i}f(x_{i};\bm{\theta}^{*})\ge z_{0}>0$ for all $i\in[n]$. This indicates that $L_{n}(\bm{\theta}^{*};p)=\error=0$.
\end{proof}

\clearpage

\section{Additional Results in Section~\ref{sec::discussions}}

\subsection{Proof of Proposition~\ref{prop::relus}}\label{appendix::prop-relus}

\begin{proposition}
Assume that assumption~\ref{assump::loss} and~\ref{assump::shortcut-connection} are satisfed.
Assume that neurons in the network $f_{S}$ satisfy that $\sigma(z)=0$ for all $z\le 0$ and $\sigma(z)$ is piece-wise continuous on $\mathbb{R}$. Then there exists a feedforward network $f_{D}$ and a distribution satisfying assumptions in Theorem~\ref{thm::convex-finite-deep} or~\ref{thm::linear-sep-deep} such that with probability one, the empirical loss $\hat{L}_{n}(\bm{\theta};p),p\ge2$ has a local minima $\bm{\theta}^{*}=(\bm{\theta}_{S}^{*},\bm{\theta}^{*}_{D})$ 
satisfying $\hat{R}_{n}(\bm{\theta}^{*})\ge\frac{\min\{n_{+},n_{-}\}}{n}$, where $n_{+}$ and $n_{-}$ are the number of positive and negative samples, respectively.  
\end{proposition}
\begin{proof}
We choose the network architecture $f_{D}(x;\bm{\theta}_{D})\equiv0$ for all $x\in\mathbb{R}^{d}$. Then the output of the network is $$f(x;\bm{\theta})=f_{S}(x;\bm{\theta}_{S})=a_{0}+\sum_{j=1}^{M}a_{j}\sigma\left(\bm{w}^{\top}_{j}x_{i}\right).$$
Now we prove the following claim showing that if the dataset contains both positive and negative samples, then the empirical loss has a local minimum with a non-zero training error. 
\begin{claim}\label{claim::prop-relu}
Under the conditions in proposition~\ref{prop::relus}, if the dataset contains both positive and negative samples and samples in the dataset are drawn in the space $\mathbb{R}^{d-1}\times\{1\}\times\{1,-1\}$, the empirical loss has a local minimum with a non-zero training error. Furthermore, the training error is no smaller than $\frac{\min\{n_{+},n_{-}\}}{n}$.
\end{claim}

\begin{proof}
We construct the local minimum as follows. Now we construct a local minimum $\bm{\theta}^{*}=(\bm{\theta}_{S}^{*})$. The key idea of constructing the local minimum having a training error no smaller than $\frac{\min\{n_{+},n_{-}\}}{n}$ is appropriately choosing $\bm{w}_{j}$ such that all neurons in the last layer keep inactive on all samples in the dataset. This is possible since the number of samples is bounded.

Next,  for any data set $\mathcal{D}=\{(x_{i};y_{i})\}_{i=1}^{n}$, we define $$K=\max_{i\in[n]}\|x_{i}\|_{2}.$$
Since all samples in the dataset $x_{i}\in\mathbb{R}^{d-1}\times \{1\}$, then by choosing $\bm{w}_{j}^{*}=\left({w_{j}^{(1)}}^{*},...,{w_{j}^{(d-1)}}^{*},{w_{j}^{(d)}}^{*}\right)$ such that $$\sum_{k=1}^{d-1}\left({w_{j}^{(1)}}^{*}\right)^{2}=1,$$
and ${w_{j}^{(d)}}^{*}=-K-1$. Since for all samples in the dataset 
$$\bm{w}_{j}^{\top}x_{i}=\sum_{k=1}^{d-1}{w_{j}^{(k)}}^{*}x^{(k)}_{i}+{w_{j}^{(d)}}^{*}\le K-K-1=-1,$$
then $$\sigma(\bm{w}_{j}^{\top}x_{i})=0,\quad \forall i\in[n].$$

 Therefore, the neural network becomes 
 $$f(x_{i};\bm{\theta}^{*})=a^{*}_{0},\quad \forall i\in[n].$$
 
 Finally, we set $a_{0}^{*}$ to the global minimizer of the following convex optimization problem:
$$\min_{a\in\mathbb{R}}\frac{1}{n}\sum_{i=1}^{n}\ell(-y_{i}a).$$
This indicates that for any $a\in\mathbb{R}$, 
$$\frac{1}{n}\sum_{i=1}^{n}\ell(-y_{i}a)\ge \frac{1}{n}\sum_{i=1}^{n}\ell(-y_{i}a_{0}^{*}).$$

Now we show that $\bm{\theta}^{*}$ is local minimum of the empirical loss function. Now we slightly perturb the parameters $a_{0},...,a_{M},\bm{w}_{1},...,\bm{w}_{M}$ by $\Delta a_{0},...,\Delta a_{M},\Delta \bm{w}_{1},...,\Delta \bm{w}_{M}$. Define 
$$\tilde{\bm{\theta}}=(a_{0}^{*}+\Delta a_{0},...,a^{*}_{M}+\Delta a_{M},\bm{w}^{*}_{1}+\Delta \bm{w}_{1},...,\bm{w}^{*}_{M}+\Delta \bm{w}_{M}).$$
Then, if $\|\bm{\theta}-\tilde{\bm{\theta}}\|_{2}\le\varepsilon$ and $\varepsilon$ is positive and sufficiently small, then for $\forall j\in[M]$ and $\forall \in[n]$, we have 
\begin{align*}
\bm{w}^{*}_{j}x_{i}+\Delta \bm{w}_{j}^{\top}x_{i}\le-1+\left\|\Delta \bm{w}_{j}\right\|_{2}\left\|x_{i}\right\|_{2}\le -1+K\varepsilon<0.
\end{align*}
This means that if $\varepsilon$ is positive and sufficiently small, then 
$$f(x_{i};\tilde{\bm{\theta}})=a_{0}^{*}+\Delta a_{0}.$$
In addition, for all $\Delta a_{0}\in\mathbb{R}$,
$$\frac{1}{n}\sum_{i=1}^{n}\ell(-y_{i}a^{*}+\Delta a_{0})\ge \frac{1}{n}\sum_{i=1}^{n}\ell(-y_{i}a_{0}^{*}),$$
therefore for $\tilde{\bm{\theta}}:\|\tilde{\bm{\theta}}-\bm{\theta}^{*}\|_{2}\le\delta(\varepsilon)$ and any $a_{0}\in\mathbb{R}$
\begin{align*}
\hat{L}_{n}(\tilde{\bm{\theta}}_{})&=\frac{1}{n}\sum_{i=1}^{n}\ell(-y_{i}f(x_{i};\tilde{\bm{\theta}}_{}))=\frac{1}{n}\sum_{i=1}^{n}\ell(-y_{i}(a_{0}^{*}+\Delta a_{0}))\\&\ge\frac{1}{n}\sum_{i=1}^{n}\ell(-y_{i}a_{0}^{*})\ge\frac{1}{n}\sum_{i=1}^{n}\ell(-y_{i}f(x_{i};\bm{\theta}^{*}))=\hat{L}_{n}(\bm{\theta}^{*}).
\end{align*}
This means that $\bm{\theta}^{*}$ is a local minimum of the empirical loss and $f(x_{i};\bm{\theta}^{*})=a_{0}^{*}$ for all $i\in[n]$. This further indicates that 
$$\hat{R}_{n}(\bm{\theta}^{*})\ge\frac{\min\{n_{-},n_{+}\}}{n}.$$
\end{proof}

Now we only need to construct the data distribution satisfying assumptions in Theorem~\ref{thm::convex-finite-deep} and Theorem~\ref{thm::linear-sep-deep}, respectively, such that with probability at least $1-e^{-\Omega(n)}$, the dataset drawn from this distribution satisfies the assumption in claim~\ref{claim::prop-relu}.

\textbf{Distribution for Theorem~\ref{thm::convex-finite-deep}}: Now we define a distribution as follows, $\mathbb{P}_{\bm{X}|Y=1}$ is a uniform distribution on the region $[-2,-1]\cup[1,2]\times\{0\}\times \{1\}\times \{0\}^{d-3}$ and $\mathbb{P}_{\bm{X}|Y=-1}$ is a uniform distribution on the region $\{0\}\times[-2,-1]\cup[1,2]\times \{1\}\times \{0\}^{d-3}$. In addition, $\mathbb{P}(Y=1)=\mathbb{P}(Y=-1)=0.5$. It is easy to check that $r=3>\max\{r_{+},r_{-}\}=2$ and for any two samples independently drawn from the distribution $\mathbb{P}_{\bm{X}|Y=1}$ or $\mathbb{P}_{\bm{X}|Y=-1}$, these two samples are linearly independent. This means that this data distribution satisfies the conditions in Theorem~\ref{thm::convex-finite-deep}. In addition, if samples in the dataset are independently drawn from this distribution, then with probability $1-\frac{1}{2^{n-1}}$, the dataset contains both  positive and negative samples.  

\textbf{Distribution for Theorem~\ref{thm::linear-sep-deep}}: Now we define a distribution as follows, $\mathbb{P}_{\bm{X}|Y=1}$ is a uniform distribution on the region $[-2,-1]\times\{0\}\times \{1\}\times \{0\}^{d-3}$ and $\mathbb{P}_{\bm{X}|Y=-1}$ is a uniform distribution on the region $\{0\}\times[-2,-1]\times \{1\}\times \{0\}^{d-3}$.  It is easy to check that 
This means that this distribution satisfies the conditions in Theorem~\ref{thm::linear-sep-deep}. In addition, if samples in the dataset are independently drawn from this distribution, then with probability $1-\frac{1}{2^{n-1}}$, the dataset contains both  positive and negative samples.
\end{proof}

\clearpage
\subsection{Proof of Proposition~\ref{prop::piecewise}}\label{appendix::prop-piecewise}

\begin{proposition}
Assume that assumption~\ref{assump::loss} and~\ref{assump::shortcut-connection} are satisfed.
Assume that neurons in the network $f_{S}$ satisfy that $\sigma(z)=z$ for all $z\ge 0$ and $\sigma(z)$ is piece-wise continuous on $\mathbb{R}$. Then there exists a network architecture $f_{D}$ and a distribution satisfying assumptions in Theorem~\ref{thm::convex-finite-deep} such that, with probability at least $1-e^{-\Omega(n)}$, the empirical loss $\hat{L}_{n}(\bm{\theta};p),p\ge 2$ has  a local minima $\bm{\theta}^{*}=(\bm{\theta}_{S}^{*},\bm{\theta}^{*}_{D})$ with non-zero training error.
\end{proposition}

\begin{proof}
We choose the network architecture $f_{D}(x;\bm{\theta}_{D})\equiv0$ for all $x\in\mathbb{R}^{d}$. Then the output of the network is $$f(x;\bm{\theta})=f_{S}(x;\bm{\theta}_{S})=a_{0}+\sum_{j=1}^{M}a_{j}\sigma\left(\bm{w}^{\top}_{j}x_{i}\right).$$
Now we prove the following claim showing that if the dataset contains both positive and negative samples, then the empirical loss has a local minimum with a non-zero training error. 
\begin{claim}\label{claim::prop-piecewise}
Under the conditions in proposition~\ref{prop::relus}, if the samples in the dataset are not linearly separable and samples $(x_{i},y_{i})$  are drawn in the space $\mathbb{R}^{d-1}\times\{1\}\times\{1,-1\}$, the empirical loss has a local minimum with a non-zero training error. 
\end{claim}

\begin{proof}
We construct the local minimum as follows. Now we construct a local minimum $\bm{\theta}^{*}=(\bm{\theta}_{S}^{*})$. The key idea of constructing the local minimum having a training error no smaller than $\frac{\min\{n_{+},n_{-}\}}{n}$ is appropriately choosing $\bm{w}_{j}$ such that all neurons in the last layer keep inactive on all samples in the dataset. This is possible since the number of samples is bounded.

First, let $\bm{w}^{*}$ be a global minimizer of the following convex optimization problem:
\begin{equation}\label{eq::piecewise}\min_{\bm{w}\in\mathbb{R}^{d}}\sum_{i=1}^{n}\ell_{p}(-y_{i}(\bm{w}^{\top}x_{i})).\end{equation}

Next,  for any data set $\mathcal{D}=\{(x_{i};y_{i})\}_{i=1}^{n}$, we define $$K=\max_{i\in[n]}|{\bm{w}^{*}}^{\top}x_{i}|\quad\text{and}\quad K_{1}=\max_{i\in[n]}\|x_{i}\|_{2}.$$
Since all samples in the dataset $x_{i}\in\mathbb{R}^{d-1}\times \{1\}$, then by choosing $\bm{w}_{j}^{*}=\left({w_{j}^{(1)}}^{*},...,{w_{j}^{(d-1)}}^{*},{w_{j}^{(d)}}^{*}\right)$ such that $${w_{j}^{(1)}}^{*}={w^{(1)}}^{*},...,{w_{j}^{(d-1)}}^{*}={w^{(d-1)}}^{*},{w_{j}^{(d)}}^{*}={w^{(d)}}^{*}+K+1.$$
 Since for all samples in the dataset 
$${\bm{w}_{j}^{*}}^{\top}x_{i}={\bm{w}^{*}}^{\top}x_{i}+K+1\ge -K+K+1=1,$$
then $$\sigma(\bm{w}_{j}^{\top}x_{i})=\bm{w}^{\top}x_{i},\quad \forall i\in[n].$$

In addition, let $a_{j}^{*}=\frac{1}{M}$ and $a_{0}^{*}=0$. Therefore, the neural network becomes 
 $$f(x_{i};\bm{\theta}^{*})=\bm{w}^{\top}x_{i},\quad \forall i\in[n].$$
 
Since $\bm{w}^{*}$ is the global optimizer of the convex optimization problem defined in Equation~\eqref{eq::piecewise}, this indicates that for any $\bm{w}\in\mathbb{R}^{d}$, 
$$\frac{1}{n}\sum_{i=1}^{n}\ell_{p}(-y_{i}(\bm{w}^{\top}x_{i}))\ge \frac{1}{n}\sum_{i=1}^{n}\ell_{p}(-y_{i}({\bm{w}^{*}}^{\top}x_{i})).$$

Now we show that $\bm{\theta}^{*}$ is local minimum of the empirical loss function. Now we slightly perturb the parameters $a_{0},...,a_{M},\bm{w}_{1},...,\bm{w}_{M}$ by $\Delta a_{0},...,\Delta a_{M},\Delta \bm{w}_{1},...,\Delta \bm{w}_{M}$. Define 
$$\tilde{\bm{\theta}}=(a_{0}^{*}+\Delta a_{0},...,a^{*}_{M}+\Delta a_{M},\bm{w}^{*}_{1}+\Delta \bm{w}_{1},...,\bm{w}^{*}_{M}+\Delta \bm{w}_{M}).$$
Then, if $\|\bm{\theta}-\tilde{\bm{\theta}}\|_{2}\le\varepsilon$ and $\varepsilon$ is positive and sufficiently small, then for $\forall j\in[M]$ and $\forall \in[n]$, we have 
\begin{align*}
\bm{w}^{*}_{j}x_{i}+\Delta \bm{w}_{j}^{\top}x_{i}\ge1-\left\|\Delta \bm{w}_{j}\right\|_{2}\left\|x_{i}\right\|_{2}\ge 1-K_{1}\varepsilon>0.
\end{align*}
This means that if $\varepsilon$ is positive and sufficiently small, then 
$$f(x_{i};\tilde{\bm{\theta}})=\Delta a_{0}+\sum_{j=1}^{M}(a_{j}^{*}+\Delta a_{j})\left(\bm{w}^{\top}x_{i}+\Delta \bm{w}_{j}^{\top}x_{i}\right).$$
This means that $f(x;\tilde{\bm{\theta}})$ behave as a linear model on the dataset. Since $\bm{w}^{*}$ corresponds to the optimal linear model minimizing the empirical loss, then 
\begin{align*}
\hat{L}_{n}(\tilde{\bm{\theta}}_{})&=\frac{1}{n}\sum_{i=1}^{n}\ell_{p}(-y_{i}f(x_{i};\tilde{\bm{\theta}}_{}))\\
&\ge\frac{1}{n}\sum_{i=1}^{n}\ell_{p}(-y_{i}(\bm{w}^{\top}x_{i}))\ge\frac{1}{n}\sum_{i=1}^{n}\ell_{p}(-y_{i}f(x_{i};\bm{\theta}^{*}))=\hat{L}_{n}(\bm{\theta}^{*}).
\end{align*}
This means that $\bm{\theta}^{*}$ is a local minimum of the empirical loss and $f(x_{i};\bm{\theta}^{*})=a_{0}^{*}$ for all $i\in[n]$. This further indicates that 
$$\hat{R}_{n}(\bm{\theta}^{*})\ge\frac{\min\{n_{-},n_{+}\}}{n}.$$
\end{proof}

Now we only need to construct the data distribution satisfying assumptions in Theorem~\ref{thm::convex-finite-deep} such that with probability at least $1-e^{-\Omega(n)}$, the dataset drawn from this distribution satisfies the assumption in claim~\ref{claim::prop-piecewise}.

\textbf{Distribution for Theorem~\ref{thm::convex-finite-deep}}: Now we define a distribution as follows, $\mathbb{P}_{\bm{X}|Y=1}$ is a uniform distribution on the region $[-2,-1]\cup[1,2]\times\{0\}\times \{1\}\times \{0\}^{d-3}$ and $\mathbb{P}_{\bm{X}|Y=-1}$ is a uniform distribution on the region $\{0\}\times [-2,-1]\cup[1,2]\times \{1\}\times \{0\}^{d-3}$. In addition, $\mathbb{P}(Y=1)=\mathbb{P}(Y=-1)=0.5$. It is easy to check that $r=3>\max\{r_{+},r_{-}\}=2$ and for any two samples independently drawn from the distribution $\mathbb{P}_{\bm{X}|Y=1}$ or $\mathbb{P}_{\bm{X}|Y=-1}$, these two samples are linearly independent. This means that this data distribution satisfies the conditions in Theorem~\ref{thm::convex-finite-deep}. In addition, if samples in the dataset are independently drawn from this distribution, then with probability $1-e^{-\Omega(n)}$, the dataset contains samples in each of the following four regions: $[-2,-1]\times\{0\}\times \{1\}\times \{0\}^{d-3}$, $[1,2]\times\{0\}\times \{1\}\times \{0\}^{d-3}$, $\{0\}\times[1,2]\times \{1\}\times \{0\}^{d-3}$ and $\{0\}\times[-2,-1]\times \{1\}\times \{0\}^{d-3}$, which makes the samples in the dataset not linearly separable.  

\end{proof}

\clearpage
\subsection{Proof of Proposition~\ref{prop::logistic}}\label{appendix::prop-logistic}

\begin{proposition}\vspace{-0.25cm}
Assume that assumption~\ref{assump::loss} and~\ref{assump::shortcut-connection} are satisfed.
Assume that  there exists a constant $c\in\mathbb{R}$ such that neurons in the network $f_{S}$ satisfy $\sigma(z)+\sigma(-z)\equiv c$ for all $z\in\mathbb{R}$. Assume that the dataset $\mathcal{D}$ has $2n$ samples. Then there exists a network architecture $f_{D}$ and a distribution satisfying assumptions in Theorem~\ref{thm::convex-finite-deep} such that, with probability at least $\Omega(1/n^{2})$, the empirical loss function $\hat{L}_{2n}(\bm{\theta};p)$ has a local minimum $\bm{\theta}^{*}=(\bm{\theta}_{S}^{*},\bm{\theta}^{*}_{D})$ satisfying $\hat{R}_{2n}(\bm{\theta}^{*})\ge\frac{\min\{n_{-},n_{+}\}}{2n}$, where $n_{+}$ and $n_{-}$ denote the number of positive and negative samples in the dataset, respectively.
\end{proposition}\vspace{-0.25cm}

\begin{proof}
We first prove the following claim showing that when the dataset satisfies certain conditions, there exists a local minimum satisfying $\hat{R}_{2n}(\bm{\theta}^{*})\ge\frac{\min\{n_{-},n_{+}\}}{2n}$. Next, we construct a data distribution such that  the dataset drawn from the distribution satisfies these conditions with probability $\Omega(1/n^{2})$.
 
\begin{claim}
Assume that for each sample $(x_{i},y_{i})$ in the dataset $\mathcal{D}=\{(x_{i},y_{i})\}_{i=1}^{2n}$, there exists a sample $(x_{j},y_{j})\in\mathcal{D}$ such that $\left\|x_{i}+x_{j}\right\|_{2}=0$ and $y_{i}=y_{j}$.
 If the function $\sigma(z)+\sigma(-z)\equiv$ constant on $\mathbb{R}$, then the empirical loss function $\hat{L}_{2n}(\bm{\theta})$ has a local minimum $\bm{\theta}^{*}$ satisfying $\hat{R}_{2n}(\bm{\theta}^{*})\ge\frac{\min\{n_{-},n_{+}\}}{2n}$.
\end{claim}
\begin{proof}
Consider a single layer neural network 
$$f(x;\bm{\theta})=a_{0}+\sum_{j=1}^{M}a_{j}\sigma(\bm{w}^{\top}_{j}x).$$
Now we construct a local minimum $\bm{\theta}^{*}$. Let $a_{1}^{*}=...=a_{M}^{*}=-1$, and $\bm{w}^{*}_{1}=...=\bm{w}^{*}_{M}=\bm{0}_{d}$. Thus $f(x;\bm{\theta}^{*})=a_{0}^{*}-M\sigma(0)$. Let $a_{0}^{*}$ be the global optimizer of the following convex optimization problem. 
$$\min_{a}\sum_{i=1}^{2n}\ell_{p}(-y_{i}(a-M\sigma(0))).$$
Thus, we have 
\begin{equation}\label{eq::prop-logistic-1}\sum_{i=1}^{2n}\ell_{p}'(-y_{i}(a_{0}^{*}-M\sigma(0)))(-y_{i})=0,\end{equation}
and this indicates that 
\begin{equation}\label{eq::prop-logistic-2}\sum_{i:y_{i}=1}\ell_{p}'(-(a_{0}^{*}-M\sigma(0)))=\sum_{i:y_{i}=-1}\ell_{p}'(a_{0}^{*}-M\sigma(0))\quad\text{or}\quad{\ell_{p}'(-a_{0}^{*}+M\sigma(0))}{n_{+}}={\ell_{p}'(a_{0}^{*}-M\sigma(0))}{n_{-}}.\end{equation}
In addition, we have, for $\forall j\in[M]$,
\begin{align*}
\frac{\partial \hat{L}_{2n}(\bm{\theta}^{*})}{a_{j}}&=\sum_{i=1}^{2n}\ell_{p}'(-y_{i}(a_{0}^{*}-M\sigma(0)))(-y_{i})\sigma(0)=0,&&\text{by Equation \eqref{eq::prop-logistic-1}}\\
\nabla_{\bm{w}_{j}}\hat{L}_{2n}(\bm{\theta}^{*})&=\sum_{i=1}^{2n}\ell_{p}'(-y_{i}(a_{0}^{*}-M\sigma(0)))(-y_{i})\sigma'(0)x_{i},\\
&=\sigma'(0)\sum_{i=1}^{2n}\ell_{p}'(-y_{i}(a_{0}^{*}-M\sigma(0)))(-y_{i})x_{i}.
\end{align*}
By assumption that for each sample $(x_{i},y_{i})$ in the dataset, there exists a sample $(x_{j},y_{j})$ in the dataset such that $x_{i}+x_{j}=\bm{0}_{d}$ and $y_{i}=y_{j}$, i.e., $y_{i}x_{i}+y_{j}x_{j}=\bm{0}_{d}$, thus we have for any $j\in[M]$,
\begin{equation}\label{eq::prop-logistic-3}
\nabla_{\bm{w}_{j}}\hat{L}_{2n}(\bm{\theta}^{*})=\sigma'(0)\sum_{i=1}^{2n}\ell_{p}'(-y_{i}(a_{0}^{*}-M\sigma(0)))(-y_{i})x_{i}=\bm{0}_{d}.
\end{equation}
Furthermore, we have 
$$\frac{\partial \hat{L}_{2n}(\bm{\theta}^{*})}{a_{0}}=\sum_{i=1}^{2n}\ell_{p}'(-y_{i}(a_{0}^{*}-M\sigma(0)))(-y_{i})=0,$$
then $\bm{\theta}^{*}$ is a critical point. Now we only need to show that it is a local minimum. We prove it by definition. 
Consider any perturbation $\Delta a_{1},...,\Delta a_{M}:|\Delta a_{j}|<\frac{1}{2}$ for all $j\in[M]$, $\Delta \bm{w}_{1},...,\Delta \bm{w}_{M}\in\mathbb{R}^{d}$ and $\Delta a_{0}\in\mathbb{R}$. Define $$\tilde{\bm{\theta}}=(a_{0}^{*}+\Delta a_{0},...,a_{M}^{*}+\Delta a_{M},\bm{w}_{1}^{*}+\Delta \bm{w}_{1},...,\bm{w}_{M}^{*}+\Delta \bm{w}_{M}).$$
Then 
\begin{align*}
\sum_{i=1}^{2n}\ell_{p}(-y_{i}f(x_{i};\tilde{\bm{\theta}}))-\sum_{i=1}^{2n}\ell_{p}(-y_{i}f(x_{i};\bm{\theta}^{*}))&=\sum_{i=1}^{2n}\left[\ell_{p}(-y_{i}f(x_{i};\tilde{\bm{\theta}}))-\ell_{p}(-y_{i}f(x_{i};\bm{\theta}^{*}))\right]\\
&\ge\sum_{i=1}^{2n}\ell_{p}'(-y_{i}f(x_{i};\bm{\theta}^{*}))(-y_{i})[f(x_{i};\tilde{\bm{\theta}})-f(x_{i};{\bm{\theta}}^{*})]\\
&=\sum_{i=1}^{2n}\ell_{p}'(-y_{i}(a_{0}^{*}-M\sigma(0)))(-y_{i})[f(x_{i};\tilde{\bm{\theta}})-a_{0}^{*}+M\sigma(0)]\\
&=\sum_{i=1}^{2n}\ell_{p}'(-y_{i}(a_{0}^{*}-M\sigma(0)))(-y_{i})f(x_{i};\tilde{\bm{\theta}}),
\end{align*}
where the inequality follows from the convexity of  $\ell_{p}$, the second equality follows from the fact that $f(x;\bm{\theta}^{*})\equiv a_{0}^{*}-M\sigma(0)$ and the third equality follows from Equation~\eqref{eq::prop-logistic-1}. In addition, we have 
\begin{align*}
&\sum_{i=1}^{2n}\ell_{p}'(-y_{i}(a_{0}^{*}-M\sigma(0)))(-y_{i})f(x_{i};\tilde{\bm{\theta}})\\
&=\sum_{i=1}^{2n}\ell_{p}'(-y_{i}(a_{0}^{*}-M\sigma(0)))(-y_{i})\left[\sum_{j=1}^{M}(a_{j}^{*}+\Delta a_{j})\sigma\left(\Delta\bm{w}_{j}^{\top}x_{i}\right)+\Delta a_{0}\right]\\
&=\sum_{i=1}^{2n}\ell_{p}'(-y_{i}(a_{0}^{*}-M\sigma(0)))(-y_{i})\left[\sum_{j=1}^{M}(a_{j}^{*}+\Delta a_{j})\sigma\left(\Delta\bm{w}_{j}^{\top}x_{i}\right)\right]&&\text{by Eq.~\eqref{eq::prop-logistic-1}}\\
&=\sum_{j=1}^{M}-(a_{j}^{*}+\Delta a_{j})\left[\sum_{i=1}^{2n}\ell_{p}'(-y_{i}(a_{0}^{*}-M\sigma(0)))y_{i}\sigma\left(\Delta\bm{w}_{j}^{\top}x_{i}\right)\right].
\end{align*}
Now we consider the following term
$$\sum_{i=1}^{2n}\ell_{p}'(-y_{i}(a_{0}^{*}-M\sigma(0)))y_{i}\sigma\left(\Delta\bm{w}_{j}^{\top}x_{i}\right).$$
By assumption that for each sample $(x_{i},y_{i})$ in the dataset, there exists a sample $(x_{k},y_{k})$ in the dataset such that $x_{i}+x_{k}=\bm{0}_{d}$, $y_{i}=y_{k}$ by the assumption that there exists a constant $c_{0}$ such that $\sigma(z)+\sigma(-z)\equiv c_{0}$, thus we have for any $\Delta \bm{w}_{j}\in\mathbb{R}^{d}$, $$y_{i}\sigma\left(\Delta \bm{w}_{j}^{\top}x_{i}\right)+y_{k}\sigma\left(\Delta \bm{w}_{j}^{\top}x_{k}\right)=y_{i}\sigma\left(\Delta \bm{w}_{j}^{\top}x_{i}\right)+y_{i}\sigma\left(-\Delta \bm{w}_{j}^{\top}x_{i}\right)=y_{i}c_{0}=\frac{c_{0}}{2}(y_{i}+y_{k}),$$
where the last equality follows from $y_{i}=y_{k}$. Therefore, we have for all $\Delta\bm{w}_{j}\in\mathbb{R}^{d}$,
\begin{align*}
\sum_{i=1}^{2n}\ell_{p}(-y_{i}(a_{0}^{*}-M\sigma(0)))y_{i}\sigma\left(\Delta\bm{w}_{j}^{\top}x_{i}\right)=\frac{c_{0}}{2}\sum_{i=1}^{2n}\ell_{p}(-y_{i}(a_{0}^{*}-M\sigma(0)))y_{i}=0.
\end{align*}
Thus, we have 
\begin{align*}
\sum_{i=1}^{2n}\ell_{p}'(-y_{i}(a_{0}^{*}-M\sigma(0)))(-y_{i})f(x_{i};\tilde{\bm{\theta}})
&=\sum_{j=1}^{M}-(a_{j}^{*}+\Delta a_{j})\left[\sum_{i=1}^{2n}\ell_{p}(-y_{i}(a_{0}^{*}-M\sigma(0)))y_{i}\sigma\left(\Delta\bm{w}_{j}^{\top}x_{i}\right)\right]=0,
\end{align*}
and this further indicates 
\begin{align*}
\sum_{i=1}^{2n}\ell_{p}(-y_{i}f(x_{i};\tilde{\bm{\theta}}))-\sum_{i=1}^{2n}\ell_{p}(-y_{i}f(x_{i};\bm{\theta}^{*}))
&\ge\sum_{i=1}^{2n}\ell_{p}'(-y_{i}(a_{0}^{*}-M\sigma(0)))(-y_{i})f(x_{i};\tilde{\bm{\theta}})=0.
\end{align*}
Therefore, this means that $\bm{\theta}^{*}$  is a local minimum. Since $f(x;\bm{\theta}^{*})=a_{0}^{*}-M\sigma(0)$, then clearly, 
$$\hat{R}_{2n}(\bm{\theta}^{*})\ge\frac{\min\{n_{+},n_{-}\}}{n}.$$
\end{proof}

Now we construct the data distribution $\mathbb{P}_{\bm{X}\times Y}$ as follows $$\mathbb{P}(\bm{X}=(1,0),Y=1)=\mathbb{P}(\bm{X}=(-1, 0),Y=1)=\mathbb{P}(\bm{X}=(0, 1),Y=-1)=\mathbb{P}(\bm{X}=(0,-1),Y=-1).$$
Assume that samples in the dataset $\mathcal{D}=\{(x_{i},y_{i})\}_{i=1}^{2n}$ are independently draw from the data distribution $\mathbb{P}_{\bm{X}\times Y}$. Let $n_{(1,0)}$ and $n_{(-1,0)}$ denote the number of samples at the point $(1,0)$ and $(-1,0)$, respectively. Let $n_{(0,1)}$ and $n_{(0,-1)}$ denote the number of samples at the point $(0,1)$ and $(0,-1)$, respectively. Then the probability that $n_{(1,0)}=n_{(-1,0)}$ and $n_{(0,1)}=n_{(0,-1)}$ is 
\begin{align*}
&\mathbb{P}_{\bm{X}\times Y}\left[n_{(1,0)}=n_{(-1,0)} \text{ and } n_{(0,1)}=n_{(0,-1)}\right]=\sum_{i=1}^{n}{{2n}\choose{2i}}{{2i}\choose{i}}{{2(n-i)}\choose{n-i}}\left(\frac{1}{4}\right)^{2n}\\
&=\sum_{i=1}^{n}\frac{(2n)!}{(2i)!(2n-2i)!}\frac{(2i)!}{\left[i!\right]^{2}}\frac{(2n-2i)!}{\left[(n-i)!\right]^{2}}\left(\frac{1}{16}\right)^{n}=\sum_{i=1}^{n}\frac{(2n)!}{[i!(n-i)!]^{2}}\frac{1}{16^{n}}\\
&=\frac{(2n)!}{16^{n}(n!)^{2}}\sum_{i=1}^{n}\frac{(n!)^{2}}{\left[i!(n-i)!\right]^{2}}=\frac{(2n)!}{16^{n}(n!)^{2}}\sum_{i=1}^{n}{{n}\choose {i}}^{2}\\
&=\frac{1}{16^{n}}{{2n}\choose{n}}^{2}>\frac{1}{(n+1)^{2}}
\end{align*}
by the equality
$$\sum_{i=1}^{n}{{n}\choose {i}}^{2}={{2n}\choose{n}}$$
and the inequality
 $${{2n}\choose{n}}>\frac{4^{n}}{n+1}.$$
Now we only need to check whether the distribution $\mathbb{P}_{\bm{X}\times Y}$ satisfies the assumptions shown in Theorem~\ref{thm::convex-finite-deep}. Clearly, $r_{+}=r_{-}=1<r=2$ and with probability 1, random vector $X$ drawn from distribution $\mathbb{P}_{\bm{X}|Y=1}$ and  random vector $Z$
drawn from distribution $\mathbb{P}_{\bm{X}|Y=-1}$ has rank one which equals to $r_{+}$ and $r_{-}$. Therefore, the distribution constructed here satisfies the assumptions in Theorem~\ref{thm::convex-finite-deep}.
\end{proof}

\subsection{Proof of Proposition~\ref{prop::quadratic}}\label{appendix::prop-quadratic}

\begin{proposition}
Assume that assumption~\ref{assump::loss} and~\ref{assump::shortcut-connection} are satisfed.
Assume that neurons in $f_{S}$ satisfy that $\sigma$ is strongly convex and twice differentiable on $\mathbb{R}$ and has a global minimum at $z=0$. Then there exists a network architecture $f_{D}$ and a distribution satisfying assumptions in Theorem~\ref{thm::linear-sep-deep} such that with probability one, the empirical loss $\hat{L}_{n}(\bm{\theta};p),p\ge2$ has a local minima $\bm{\theta}^{*}=(\bm{\theta}_{S}^{*},\bm{\theta}^{*}_{D})$ 
satisfying $\hat{R}_{n}(\bm{\theta}^{*})\ge\frac{\min\{n_{+},n_{-}\}}{n}$, where $n_{+}$ and $n_{-}$ denote the number of positive and negative samples in the dataset, respectively.
\end{proposition}
\begin{proof}
We first prove the following claim showing that if the dataset satisfies certain conditions, then the empirical loss has a local minimum satisfying $\hat{R}_{n}(\bm{\theta}^{*})\ge\frac{\min\{n_{-},n_{+}\}}{n}$. Next, we construct a data distribution such that  the dataset drawn from the distribution $\bm{P}_{\bm{X}\times Y}$ satisfies these conditions with probability one. 

\begin{claim} 
If the matrix $\frac{1}{n_{+}}\sum_{i:y_{i}=1}x_{i}x_{i}^{\top}-\frac{1}{n_{-}}\sum_{i:y_{i}=-1}x_{i}x_{i}^{\top}$ is positive or negative definite, then  the empirical loss function $\hat{L}_{n}(\bm{\theta})$ has a local minimum $\bm{\theta}^{*}$ satisfying $\hat{R}_{n}(\bm{\theta}^{*})\ge\frac{\min\{n_{-},n_{+}\}}{n}$.
\end{claim}
\begin{proof}
We prove that if the following matrix 
$$\frac{1}{n_{+}}\sum_{i:y_{i}=1}x_{i}x_{i}^{\top}-\frac{1}{n_{-}}\sum_{i:y_{i}=-1}x_{i}x_{i}^{\top}$$
is either positive definite or negative definite, then there exists a local minima $\bm{\theta}^{*}$ having $f(x;\bm{\theta}^{*})\equiv $ constant and this leads to $\hat{R}_{n}(\bm{\theta}^{*})\ge\frac{\min\{n_{+},n_{-}\}}{n}$.  Without loss of generality, we assume that the matrix is positive definite. 
Consider a single layer neural network 
$$f(x;\bm{\theta})=a_{0}+\sum_{j=1}^{M}a_{j}\sigma\left(\bm{w}^{\top}_{j}x\right).$$
Let $a_{1}^{*}=...=a_{M}^{*}=-1$ and $\bm{w}^{*}_{1}=...=\bm{w}^{*}_{M}=\bm{0}_{d}$. 

Therefore, we have $f(x;\bm{\theta}^{*})=a_{0}^{*}-M\sigma(0)$. Let $a_{0}^{*}$ be the global optimizer of the following convex optimization problem. 
$$\min_{a}\sum_{i=1}^{n}\ell_{p}(-y_{i}(a-M\sigma(0))).$$
Thus, we have 
\begin{equation}\label{eq::prop-single-necc-2}\sum_{i=1}^{n}\ell_{p}'(-y_{i}(a_{0}^{*}-M\sigma(0)))(-y_{i})=0,\end{equation}
and this indicates that 
\begin{equation}\label{eq::thm-convex-counterexample-eq2}\sum_{i:y_{i}=1}\ell_{p}'(-(a_{0}^{*}-M\sigma(0)))=\sum_{i:y_{i}=-1}\ell_{p}'(a_{0}^{*}-M\sigma(0))\quad\text{or}\quad{\ell_{p}'(-a_{0}^{*}+M\sigma(0))}{n_{+}}={\ell_{p}'(a_{0}^{*}-M\sigma(0))}{n_{-}}.\end{equation}
In addition, since for $\forall j\in[M]$,
\begin{align*}
&\frac{\partial \hat{L}_{n}(\bm{\theta}^{*})}{\partial a_{j}}=\sum_{i=1}^{n}\ell_{p}'(-y_{i}(a_{0}^{*}-M\sigma(0)))(-y_{i})\sigma(0)=0,&&\text{by Equation }~\eqref{eq::prop-single-necc-2},\\
& \nabla_{\bm{w}_{j}}\hat{L}_{n}(\bm{\theta}^{*})=\sum_{i=1}^{n}\ell_{p}'(-y_{i}(a_{0}^{*}-M\sigma(0)))(-y_{i})\sigma'(0)x_{i}=\bm{0}_{d},&&\text{by }\sigma'(0)=0,%\\
%&\frac{\partial \hat{L}_{n}(\bm{\theta}^{*})}{\partial b_{j}}=\sum_{i=1}^{n}\ell_{p}'(-y_{i}(a_{0}^{*}-M\sigma(0)))(-y_{i})\sigma'(0)=0,&&\text{by }\sigma'(0)=0,
\end{align*}
and 
$$\frac{\partial \hat{L}_{n}(\bm{\theta}^{*})}{\partial a_{0}}=\sum_{i=1}^{n}\ell_{p}'(-y_{i}(a_{0}^{*}-M\sigma(0)))(-y_{i})=0,$$
then $\bm{\theta}^{*}$ is a critical point. 

Next we show that $\bm{\theta}^{*}=(a_{0}^{*},...,a_{M}^{*},\bm{w}_{1}^{*},...,\bm{w}_{M}^{*})$ is a local minima. Consider any perturbation $\Delta a_{1},...,\Delta a_{M}:|\Delta a_{j}|<\frac{1}{2}$ for all $j\in[M]$, $\Delta \bm{w}_{1},...,\Delta \bm{w}_{M}\in\mathbb{R}^{d}$ and $\Delta a_{0}\in\mathbb{R}$. Define $$\tilde{\bm{\theta}}=(a_{0}^{*}+\Delta a_{0},...,a_{M}^{*}+\Delta a_{M},\bm{w}_{1}^{*}+\Delta \bm{w}_{1},...,\bm{w}_{M}^{*}+\Delta \bm{w}_{M}
).$$

Then 
\begin{align*}
\sum_{i=1}^{n}\ell_{p}(-y_{i}f(x_{i};\tilde{\bm{\theta}}))-\sum_{i=1}^{n}\ell_{p}(-y_{i}f(x_{i};\bm{\theta}^{*}))&=\sum_{i=1}^{n}\left[\ell_{p}(-y_{i}f(x_{i};\tilde{\bm{\theta}}))-\ell_{p}(-y_{i}f(x_{i};\bm{\theta}^{*}))\right]\\
&\ge\sum_{i=1}^{n}\ell_{p}'(-y_{i}f(x_{i};\bm{\theta}^{*}))(-y_{i})[f(x_{i};\tilde{\bm{\theta}})-f(x_{i};{\bm{\theta}}^{*})]\\
&=\sum_{i=1}^{n}\ell_{p}'(-y_{i}(a_{0}^{*}-M\sigma(0)))(-y_{i})[f(x_{i};\tilde{\bm{\theta}})-a_{0}^{*}+M\sigma(0)]\\
&=\sum_{i=1}^{n}\ell_{p}'(-y_{i}(a_{0}^{*}-M\sigma(0)))(-y_{i})f(x_{i};\tilde{\bm{\theta}}),
\end{align*}
where the inequality follows from the convexity of the loss function $\ell_{p}(z)$, the second equality follows from the fact that $f(x;\bm{\theta}^{*})\equiv a_{0}^{*}-M\sigma(0)$ and the third equality follows from Equation~\eqref{eq::thm-convex-counterexample-eq2}. In addition, we have 
\begin{align*}
&\sum_{i=1}^{n}\ell_{p}'(-y_{i}(a_{0}^{*}-M\sigma(0)))(-y_{i})f(x_{i};\tilde{\bm{\theta}})\\
&=\sum_{i=1}^{n}\ell_{p}'(-y_{i}(a_{0}^{*}-M\sigma(0)))(-y_{i})\left[\sum_{j=1}^{M}(a_{j}^{*}+\Delta a_{j})\sigma\left(\Delta\bm{w}_{j}^{\top}x_{i}\right)+\Delta a_{0}\right]\\
&=\sum_{i=1}^{n}\ell_{p}'(-y_{i}(a_{0}^{*}-M\sigma(0)))(-y_{i})\left[\sum_{j=1}^{M}(a_{j}^{*}+\Delta a_{j})\sigma\left(\Delta\bm{w}_{j}^{\top}x_{i}\right)\right]&&\text{by Eq.~\eqref{eq::thm-convex-counterexample-eq2}}\\
&=\sum_{j=1}^{M}-(a_{j}^{*}+\Delta a_{j})\left[\sum_{i=1}^{n}\ell_{p}'(-y_{i}(a_{0}^{*}-M\sigma(0)))y_{i}\sigma\left(\Delta\bm{w}_{j}^{\top}x_{i}\right)\right].
\end{align*}
Now we define the following function $G:\mathbb{R}^{d}\rightarrow \mathbb{R}$, 
\begin{equation*}
G(\bm{u})=\sum_{i=1}^{n}\ell_{p}'(-y_{i}(a_{0}^{*}-M\sigma(0)))y_{i}\sigma\left(\bm{u}^{\top}x_{i}\right).
\end{equation*}
Now we consider the gradient of the function $G$ with respect to the vector $\bm{u}$ at the point $\bm{0}_{d}$,
\begin{align*}
\nabla_{\bm{u}}G(\bm{0}_{d})&=\sum_{i=1}^{n}\ell_{p}'(-y_{i}(a_{0}^{*}-M\sigma(0)))y_{i}\sigma'\left(0\right)x_{i}=\bm{0}_{d}.
\end{align*}
Furthermore,  the Hessian matrix $\nabla_{\bm{u}}^{2}G(\bm{0}_{d})$ satisfies
\begin{align*}
\nabla_{\bm{u}}^{2}G(\bm{0}_{d})&=
\sum_{i=1}^{n}\ell_{p}'(-y_{i}(a_{0}^{*}-M\sigma(0)))y_{i}\sigma''\left(0\right)x_{i}x_{i}^{\top}=\sigma''\left(0\right)\sum_{i=1}^{n}\ell_{p}'(-y_{i}(a_{0}^{*}-M\sigma(0)))y_{i}x_{i}x_{i}^{\top}\\
&=\sigma''(0)\left[\frac{1}{n_{+}}\sum_{i:y_{i}=1}x_{i}x_{i}^{\top}-\frac{1}{n_{-}}\sum_{i:y_{i}=-1}x_{i}x_{i}^{\top}\right]\succ 0,
\end{align*}
then the function $G(\bm{u})=\sum_{i=1}^{n}\ell_{p}(-y_{i}(a_{0}^{*}-M\sigma(0)))y_{i}\sigma\left(\bm{u}^{\top}x_{i}\right)$ has a local minima at $\bm{u}=\bm{0}_{d}$. This indicates that there exists $\varepsilon>0$ such that for all $\Delta\bm{w}:\|\Delta\bm{w}\|_{2}\le\varepsilon$, 
$$\sum_{i=1}^{n}\ell_{p}'(-y_{i}(a_{0}^{*}-M\sigma(0)))y_{i}\sigma\left(\Delta\bm{w}^{\top}x_{i}\right)\ge \sum_{i=1}^{n}\ell_{p}(-y_{i}(a_{0}^{*}-M\sigma(0)))y_{i}\sigma\left(0\right)=0.$$
In addition, since  $a_{j}^{*}=-1$, $|\Delta a_{j}|<\frac{1}{2}$, then for all $\Delta\bm{w}_{j}:\|\Delta\bm{w}_{j}\|_{2}\le\varepsilon$,
\begin{align*}
\sum_{i=1}^{n}\ell_{p}'(-y_{i}(a_{0}^{*}-M\sigma(0)))(-y_{i})f(x_{i};\tilde{\bm{\theta}})=\sum_{j=1}^{M}-(a_{j}^{*}+\Delta a_{j})\left[\sum_{i=1}^{n}\ell_{p}(-y_{i}(a_{0}^{*}-M\sigma(0)))y_{i}\sigma\left(\Delta\bm{w}_{j}^{\top}x_{i}\right)\right]\ge 0.
\end{align*}
Therefore, we have 
$$\sum_{i=1}^{n}\ell_{p}'(-y_{i}(a_{0}^{*}-M\sigma(0)))(-y_{i})f(x_{i};\tilde{\bm{\theta}})\ge 0,$$
and this indicates that 
$$\sum_{i=1}^{n}\ell_{p}(-y_{i}f(x_{i};\tilde{\bm{\theta}}))-\sum_{i=1}^{n}\ell_{p}(-y_{i}f(x_{i};\bm{\theta}^{*}))\ge 0.$$
Thus, $\bm{\theta}^{*}$ is a local minima with $f(x;\bm{\theta}^{*})=a_{0}^{*}-M\sigma(0)=$ constant. Thus, $$\sum_{i=1}^{n}\mathbb{I}\{y_{i}\neq \sgn(f(x_{i};\bm{\theta}^{*}))\}\ge \frac{\min\{n_{-},n_{+}\}}{n}.$$
\end{proof}

Now we define a data distribution  as follows. Let $\mathbb{P}_{Y}(Y=1)=\mathbb{P}(Y=-1)=0.5$. Let $\mathbb{P}_{\bm{X}|Y=1}$ be a continuous distribution (e.g., uniform distribution) defined on the interval $[2, 3]$ and $\mathbb{P}_{\bm{X}|Y=-1}$ be a continuous distribution defined on the interval $[-1, -1/2]$. Then if samples in the dataset $\mathcal{D}$ are drawn independently from the this distribution, the scalar $\frac{1}{n_{+}}\sum_{i:y_{i}=1}x_{i}^{2}-\frac{1}{n_{-}}\sum_{i:y_{i}=-1}x_{i}^{2}>0$ if $n_{+}>0$ and the scalar $\frac{1}{n_{+}}\sum_{i:y_{i}=1}x_{i}^{2}-\frac{1}{n_{-}}\sum_{i:y_{i}=-1}x_{i}^{2}<0$ if $n_{+}=0$. This means that the dataset satisfies the conditions in the claim with probability one. 

\end{proof}

\clearpage

\subsection{Proof of Proposition~\ref{prop::feedforward}}\label{appendix::prop-feedforward}

\begin{proposition}
Assume that assumption \ref{assump::loss} is satisfied.
Assume that the feedforward neural network $f(x;\bm{\theta})$ has at least one hidden layer and has at least one neuron in each hidden layer. If neurons in the network $f$ satisfy that $\sigma(z)=0$ for all $z\le 0$ and $\sigma(z)$ is continuous on $\mathbb{R}$, then  the empirical loss $\hat{L}_{n}(\bm{\theta};p),p\ge 2$ has a local minima $\bm{\theta}^{*}$ satisfying $\hat{R}_{n}(\bm{\theta}^{*})\ge\frac{\min\{n_{+},n_{-}\}}{n}$, where $n_{+}$ and $n_{-}$ denote the number of positive and negative samples in the dataset, respectively. 
\end{proposition}

\begin{proof}
Assume that the multilayer neural network $f(x;\bm{\theta})$ has $L\ge 1$ hidden layers, $M_{l}\ge 1$ neurons in the $l$-th layer. Now we let the vector $\bm{\theta}_{l}$ contain all parameters in the first $l\in[L]$ layers. 
Then the output of the neural network can be rewritten as 
\begin{equation*}
f(x;a_{0},\bm{\theta}_{L})=a_{0}+\sum_{j=1}^{M_{L}}a_{j}\sigma(\bm{w}_{j}^{\top}\bm{\Phi}(x;\bm{\theta}_{L-1})+b_{j}),
\end{equation*}
where $\bm{\Phi}(x;\bm{\theta}_{L-1})=(\Phi_{1}(x;\bm{\theta}_{L-1}),...,\Phi_{M_{L-1}}(x;\bm{\theta}_{L-1}))$ denotes the outputs of the neurons in the layer $L-1$.
Now we construct a local minimum $\bm{\theta}^{*}=(a_{0}^{*},\bm{\theta}_{L}^{*})$. The key idea of constructing the local minimum having a training error no smaller than $\frac{\min\{n_{+},n_{-}\}}{n}$ is appropriately choosing $\bm{w}_{j},b_{j}$ such that all neurons in the last layer keep inactive on all samples in the dataset. This is possible since the outputs of the neurons in the layer $L-1$ are bounded.

We first  set $\bm{\theta}_{L-1}$ to any unit vector $\bm{\theta}^{*}_{L-1}:\|\bm{\theta}^{*}_{L-1}\|_{2}=1$. 
Next,  for any data set $\mathcal{D}=\{(x_{i};y_{i})\}_{i=1}^{n}$, we define $$K=\max_{i\in[n]}\|\bm{\Phi} (x_{i};\bm{\theta}_{L-1}^{*})\|_{2}.$$
In addition, it is easy to show that the function $\varphi_{ij}(\bm{\theta})=\Phi_{j}(x_{i};\bm{\theta})$ is a continuous function. Now we consider the compact set $C_{\delta}=\{\bm{\theta}:\|\bm{\theta}-\bm{\theta}_{L-1}^{*}\|_{2}\le \delta\}$, where $\delta>0$ . Since each function $\varphi_{ij}$ is a continuous function on the compact set $C$,  then by the definition of continuity, 
$$\forall \varepsilon>0, \exists \delta_{ij}(\varepsilon)\in(0,1):|\varphi_{ij}(\bm{\theta})-\varphi_{ij}(\bm{\theta}^{*}_{L-1})|\le \varepsilon\quad\text{for all }\bm{\theta}\in C_{\delta_{ij}}.$$
For a given $\varepsilon>0$, let 
$$\delta(\varepsilon)=\min_{i\in[n],j\in[M_{L-1}]}\delta_{ij}(\varepsilon),$$
then for all $i\in[n], j\in[M_{L-1}]$ and $\forall \bm{\theta}\in C_{\delta}$, 
$$|\varphi_{ij}(\bm{\theta})-\varphi_{ij}(\bm{\theta}_{L-1})|\le \varepsilon.$$ 
Now we set $\bm{w}_{j}$ to some unit vector $\bm{w}_{j}:\|\bm{w}_{j}\|_{2}=1$ for all $j\in[M_{L-1}]$, and we set $b_{j}$ to a scalar $b_{j}^{*}$ satisfying
$${\bm{w}^{*}_{j}}^{\top}\bm{\Phi}(x_{i};\bm{\theta}_{L-1}^{*})+b_{j}^{*}\le -1,\quad \text{for all }i\in[n] \text{ and all }\bm{\theta}\in C.$$ 
 Therefore, the neural network becomes 
 $$f(x_{i};a_{0},\bm{\theta}_{L}^{*})=a_{0},\quad \forall i\in[n].$$
Furthermore, for the $\delta(\varepsilon)$ defined above and for any parameter vector $\tilde{\bm{\theta}}_{L}:\|\tilde{\bm{\theta}}_{L}-\bm{\theta}^{*}_{L}\|_{2}\le\delta(\varepsilon)$, we have for all $j\in[M_{L-1}]$ and all $i\in[n]$,
\begin{align*}
&|\tilde{\bm{w}}_{j}^{\top}\bm{\Phi}(x_{i};\tilde{\bm{\theta}}_{L-1})+\tilde{b}_{j}-{\bm{w}_{j}^{*}}^{\top}\bm{\Phi}(x_{i};\bm{\theta}_{L-1}^{*})-b^{*}_{j}|\\
&\le |\tilde{\bm{w}}_{j}^{\top}\bm{\Phi}(x_{i};\tilde{\bm{\theta}}_{L-1})-\tilde{\bm{w}}_{j}^{\top}\bm{\Phi}(x_{i};{\bm{\theta}}^{*}_{L-1})+\tilde{\bm{w}}_{j}^{\top}\bm{\Phi}(x_{i};{\bm{\theta}}^{*}_{L-1})-{\bm{w}_{j}^{*}}^{\top}\bm{\Phi}(x_{i};\bm{\theta}_{L-1}^{*})|+|\tilde{b}_{j}-b_{j}|\\
&\le  |\tilde{\bm{w}}_{j}^{\top}\bm{\Phi}(x_{i};\tilde{\bm{\theta}}_{L-1})-\tilde{\bm{w}}_{j}^{\top}\bm{\Phi}(x_{i};{\bm{\theta}}^{*}_{L-1})|+|\tilde{\bm{w}}_{j}^{\top}\bm{\Phi}(x_{i};{\bm{\theta}}^{*}_{L-1})-{\bm{w}_{j}^{*}}^{\top}\bm{\Phi}(x_{i};\bm{\theta}_{L-1}^{*})|+|\tilde{b}_{j}-b_{j}|\\
&\le \|\tilde{\bm{w}}_{j}\|_{2}\|\bm{\Phi}(x_{i};\tilde{\bm{\theta}}_{L-1})-\bm{\Phi}(x_{i};{\bm{\theta}}^{*}_{L-1})\|_{2}+\|\tilde{\bm{w}}_{j}-{\bm{w}}^{*}_{j}\|_{2}\|\bm{\Phi}(x_{i};{\bm{\theta}}^{*}_{L-1})\|_{2}+|\tilde{b}_{j}-b_{j}|\\
&\le (1+\delta(\varepsilon))\sqrt{M_{L-1}}\varepsilon+\varepsilon K+\varepsilon\le (2\sqrt{M_{L-1}}+K+1)\varepsilon.
\end{align*}
Thus, if $\varepsilon<\frac{1}{2(2\sqrt{M_{L-1}}+K+1)}$, then for all $\tilde{\bm{\theta}}_{L}:\|\tilde{\bm{\theta}}_{L}-\bm{\theta}^{*}_{L}\|_{2}\le\delta(\varepsilon)$,  $\forall j\in[M]$ and $\forall i\in[n]$
\begin{align}
\tilde{\bm{w}}_{j}^{\top}\bm{\Phi}(x_{i};\tilde{\bm{\theta}}_{L-1})+\tilde{b}_{j}\le {\bm{w}_{j}^{*}}^{\top}\bm{\Phi}(x_{i};\bm{\theta}_{L-1}^{*})+b^{*}_{j}+\frac{1}{2}\le-\frac{1}{2}.
\end{align}
Since $\sigma(z)=0$ for all $z\le 0$, then this indicates that for all $\tilde{\bm{\theta}}_{L}:\|\tilde{\bm{\theta}}_{L}-\bm{\theta}^{*}_{L}\|_{2}\le\delta(\varepsilon)$, $$f(x_{i};a_{0},\tilde{\bm{\theta}}_{L-1})=a_{0},\quad\text{for all }i\in[n].$$

Finally, we set $a_{0}^{*}$ to the global minimizer of the following convex optimization problem:
$$\min_{a\in\mathbb{R}}\frac{1}{n}\sum_{i=1}^{n}\ell(-y_{i}a).$$
This indicates that for any $a\in\mathbb{R}$, 
$$\frac{1}{n}\sum_{i=1}^{n}\ell(-y_{i}a)\ge \frac{1}{n}\sum_{i=1}^{n}\ell(-y_{i}a_{0}^{*}).$$
Therefore, for $\tilde{\bm{\theta}}_{L}:\|\tilde{\bm{\theta}}_{L}-\bm{\theta}^{*}_{L}\|_{2}\le\delta(\varepsilon)$ and any $a_{0}\in\mathbb{R}$
\begin{align*}
\hat{L}_{n}(a_{0},\tilde{\bm{\theta}}_{L})&=\frac{1}{n}\sum_{i=1}^{n}\ell(-y_{i}f(x_{i};\tilde{\bm{\theta}}_{L}))=\frac{1}{n}\sum_{i=1}^{n}\ell(-y_{i}a_{0})\\&\ge\frac{1}{n}\sum_{i=1}^{n}\ell(-y_{i}a_{0}^{*})\ge\frac{1}{n}\sum_{i=1}^{n}\ell(-y_{i}f(x_{i};a_{0}^{*},\bm{\theta}_{L}^{*}))=\hat{L}_{n}(a_{0}^{*},\bm{\theta}_{L}^{*}).
\end{align*}
This means that $(a_{0}^{*}, \bm{\theta}_{L}^{*})$ is a local minima and $f(x_{i};a_{0}^{*},\bm{\theta}^{*}_{L})=a_{0}^{*}$ for all $i\in[n]$. This further indicates that 
$$\hat{R}_{n}(\bm{\theta}^{*})\ge\frac{\min\{n_{-},n_{+}\}}{n}.$$

\end{proof}

\clearpage
\subsection{Proof of Proposition~\ref{prop::short-cut}}\label{appendix::prop-short-cut}

\begin{proposition}
Assume that $H:\mathbb{R}^{d}\rightarrow\mathbb{R}^{d}$ is a feedforward neural network parameterized by $\bm{\theta}$ and all neurons in $H$ are ReLUs. Define a network $f:\mathbb{R}^{d}\rightarrow \mathbb{R}$ with identity shortcut connections as $f(x;\bm{a},\bm{\theta},b)=\bm{a}^{\top}(x+H(x;\bm{\theta}))+b$, $\bm{a}\in\mathbb{R}^{d}, b\in\mathbb{R}$. Then there exists a distribution $\mathbb{P}_{\bm{X}\times Y}$ satisfying the assumptions in Theorem~\ref{thm::convex-finite-deep} such that with probability at least $1-e^{-\Omega(n)}$, the empirical loss $\hat{L}_{n}(\bm{a},\bm{\theta},b;p)=\frac{1}{n}\sum_{i=1}^{n}\ell(-y_{i}f(x_{i};\bm{\theta});p), p\ge 2$ has a local minimum with non-zero training error.  
\end{proposition}
\begin{proof}
We first show that if the samples in the dataset are not linearly separable, then empirical loss has a local minimum with a non-zero training error. Next, we construct a data distribution such that $n$ samples independently drawn from this data distribution are not linearly separable with probability at least $1-\exp(-\Omega(n))$.

\begin{claim}
If the samples in the dataset are not linearly separable, i.e., $\min_{\bm{w}\in\mathbb{R}^{d},b\in\mathbb{R}}\frac{1}{n}\sum_{i=1}^{n}\mathbb{I}\{y_{i}\neq \sgn(\bm{w}^{\top}x_{i}+b)\}>0$, then empirical loss has a local minimum with a non-zero training error. 
\end{claim}
\begin{proof}
The proof follows from the proof of Proposition~\ref{prop::relus} in Appendix~\ref{appendix::prop-relus} where we show that when the dataset has both positive and negative samples and all neurons in the multilayer network are ReLUs, then the empirical loss has a local minimum with a non-zero training error. 

Assume that the multilayer neural network $H(x;\bm{\theta})$ has $L\ge 1$ hidden layers, $M_{l}\ge 1$ neurons in the $l$-th layer in the multilayer neural network $H$. Clearly, $M_{L}=d$. Now we let the vector $\bm{\theta}_{l}$ contain all parameters in the first $l\in[L]$ layers. 
Then the output of the neural network $f(x;\bm{a},\bm{\theta},b)$ can be rewritten as 
\begin{equation*}
f(x;\bm{a},\bm{\theta},b)=b+\sum_{j=1}^{M_{L}}a_{j}\sigma(\bm{w}_{j}^{\top}\bm{\Phi}(x;\bm{\theta}_{L-1})+b_{j})+\bm{a}^{\top}x,
\end{equation*}
where $\bm{\Phi}(x;\bm{\theta}_{L-1})=(\Phi_{1}(x;\bm{\theta}_{L-1}),...,\Phi_{M_{L-1}}(x;\bm{\theta}_{L-1}))$ denotes the outputs of the neurons in the layer $L-1$.
Now we construct a local minimum $(\bm{a}^{*},\bm{\theta}^{*},b^{*})$. The whole idea of constructing the local minimum having a non-zero training error  is as follows. We first appropriately choose $\bm{w}_{j},b_{j}$ such that all neurons in the last layer of the multilayer network $H$ keep inactive on all samples in the dataset. Then the neural network becomes a linear model $$f(x;\bm{a}^{*},\bm{\theta}^{*},b^{*})=b^{*}+{\bm{a}^{*}}^{\top}x.$$
Next we only need to set $\bm{a}^{*},{b}^{*}$ to the global optimizer of the convex optimization problem $$\min_{\bm{a}\in\mathbb{R}^{d},b\in\mathbb{R}}\frac{1}{n}\sum_{i=1}^{n}\ell_{p}\left(-y_{i}(\bm{a}^{\top}x_{i}+b)\right).$$
Therefore, as we have shown in the proof of Propsition~\ref{prop::relus}, if we slightly perturb the parameter $\bm{\theta}^{*}$, the output of the multilayer network $H(x;\tilde{\bm{\theta}})$ on all samples are still zero and this makes $f(x_{i};\bm{a}^{*},\tilde{\bm{\theta}},b^{*})={\bm{a}^{*}}^{\top}x_{i}+b^{*}$. In addition, if we further perturb the vector $\bm{a}^{*}$ and $b^{*}$, the value of the empirical loss will not decrease since $\bm{a}^{*}$ and $b^{*}$ are the global optimizer of the empirical loss function. 

Now we present the proof.
We first  set $\bm{\theta}_{L-1}$ to any unit vector $\bm{\theta}^{*}_{L-1}:\|\bm{\theta}^{*}_{L-1}\|_{2}=1$. 
Next,  for any data set $\mathcal{D}=\{(x_{i};y_{i})\}_{i=1}^{n}$, we define $$K=\max_{i\in[n]}\|\bm{\Phi} (x_{i};\bm{\theta}_{L-1}^{*})\|_{2}.$$
In addition, it is easy to show that the function $\varphi_{ij}(\bm{\theta})=\Phi_{j}(x_{i};\bm{\theta})$ is a continuous function. Now we consider the compact set $C_{\delta}=\{\bm{\theta}:\|\bm{\theta}-\bm{\theta}_{L-1}^{*}\|_{2}\le \delta\}$, where $\delta>0$ . Since each function $\varphi_{ij}$ is a continuous function on the compact set $C$,  then by the definition of continuity, 
$$\forall \varepsilon>0, \exists \delta_{ij}(\varepsilon)\in(0,1):|\varphi_{ij}(\bm{\theta})-\varphi_{ij}(\bm{\theta}^{*}_{L-1})|\le \varepsilon\quad\text{for all }\bm{\theta}\in C_{\delta_{ij}}.$$
For a given $\varepsilon>0$, let 
$$\delta(\varepsilon)=\min_{i\in[n],j\in[M_{L-1}]}\delta_{ij}(\varepsilon),$$
then for all $i\in[n], j\in[M_{L-1}]$ and $\forall \bm{\theta}\in C_{\delta}$, 
$$|\varphi_{ij}(\bm{\theta})-\varphi_{ij}(\bm{\theta}_{L-1})|\le \varepsilon.$$ 
Now we set $\bm{w}_{j}$ to some unit vector $\bm{w}_{j}:\|\bm{w}_{j}\|_{2}=1$ for all $j\in[M_{L-1}]$, and we set $b_{j}$ to a scalar $b_{j}^{*}$ satisfying
$${\bm{w}^{*}_{j}}^{\top}\bm{\Phi}(x_{i};\bm{\theta}_{L-1}^{*})+b_{j}^{*}\le -1,\quad \text{for all }i\in[n] \text{ and all }\bm{\theta}\in C.$$ 
 Therefore, the neural network becomes 
 $$f(x_{i};\bm{a},\tilde{\bm{\theta}},b)=\bm{a}^{\top}x_{i}+b,\quad \forall i\in[n].$$
Furthermore, for the $\delta(\varepsilon)$ defined above and for any parameter vector $\tilde{\bm{\theta}}_{L}:\|\tilde{\bm{\theta}}_{L}-\bm{\theta}^{*}_{L}\|_{2}\le\delta(\varepsilon)$, we have for all $j\in[M_{L-1}]$ and all $i\in[n]$,
\begin{align*}
&|\tilde{\bm{w}}_{j}^{\top}\bm{\Phi}(x_{i};\tilde{\bm{\theta}}_{L-1})+\tilde{b}_{j}-{\bm{w}_{j}^{*}}^{\top}\bm{\Phi}(x_{i};\bm{\theta}_{L-1}^{*})-b^{*}_{j}|\\
&\le |\tilde{\bm{w}}_{j}^{\top}\bm{\Phi}(x_{i};\tilde{\bm{\theta}}_{L-1})-\tilde{\bm{w}}_{j}^{\top}\bm{\Phi}(x_{i};{\bm{\theta}}^{*}_{L-1})+\tilde{\bm{w}}_{j}^{\top}\bm{\Phi}(x_{i};{\bm{\theta}}^{*}_{L-1})-{\bm{w}_{j}^{*}}^{\top}\bm{\Phi}(x_{i};\bm{\theta}_{L-1}^{*})|+|\tilde{b}_{j}-b_{j}|\\
&\le  |\tilde{\bm{w}}_{j}^{\top}\bm{\Phi}(x_{i};\tilde{\bm{\theta}}_{L-1})-\tilde{\bm{w}}_{j}^{\top}\bm{\Phi}(x_{i};{\bm{\theta}}^{*}_{L-1})|+|\tilde{\bm{w}}_{j}^{\top}\bm{\Phi}(x_{i};{\bm{\theta}}^{*}_{L-1})-{\bm{w}_{j}^{*}}^{\top}\bm{\Phi}(x_{i};\bm{\theta}_{L-1}^{*})|+|\tilde{b}_{j}-b_{j}|\\
&\le \|\tilde{\bm{w}}_{j}\|_{2}\|\bm{\Phi}(x_{i};\tilde{\bm{\theta}}_{L-1})-\bm{\Phi}(x_{i};{\bm{\theta}}^{*}_{L-1})\|_{2}+\|\tilde{\bm{w}}_{j}-{\bm{w}}^{*}_{j}\|_{2}\|\bm{\Phi}(x_{i};{\bm{\theta}}^{*}_{L-1})\|_{2}+|\tilde{b}_{j}-b_{j}|\\
&\le (1+\delta(\varepsilon))\sqrt{M_{L-1}}\varepsilon+\varepsilon K+\varepsilon\le (2\sqrt{M_{L-1}}+K+1)\varepsilon.
\end{align*}
Thus, if $\varepsilon<\frac{1}{2(2\sqrt{M_{L-1}}+K+1)}$, then for all $\tilde{\bm{\theta}}_{L}:\|\tilde{\bm{\theta}}_{L}-\bm{\theta}^{*}_{L}\|_{2}\le\delta(\varepsilon)$,  $\forall j\in[M]$ and $\forall i\in[n]$
\begin{align}
\tilde{\bm{w}}_{j}^{\top}\bm{\Phi}(x_{i};\tilde{\bm{\theta}}_{L-1})+\tilde{b}_{j}\le {\bm{w}_{j}^{*}}^{\top}\bm{\Phi}(x_{i};\bm{\theta}_{L-1}^{*})+b^{*}_{j}+\frac{1}{2}\le-\frac{1}{2}.
\end{align}
Since $\sigma(z)=0$ for all $z\le 0$, then this indicates that for all $\tilde{\bm{\theta}}_{L}:\|\tilde{\bm{\theta}}_{L}-\bm{\theta}^{*}_{L}\|_{2}\le\delta(\varepsilon)$, $$f(x_{i};\bm{a},\tilde{\bm{\theta}},b)=\bm{a}^{\top}x_{i}+b,\quad\text{for all }i\in[n].$$

Finally, we set $\bm{a}^{*},b^{*}$ to the global minimizer of the following convex optimization problem:
$$\min_{\bm{a}\in\mathbb{R}^{d},b\in\mathbb{R}}\frac{1}{n}\sum_{i=1}^{n}\ell_{p}\left(-y_{i}(\bm{a}^{\top}x_{i}+b)\right).$$
This indicates that for any $\bm{a}\in\mathbb{R}^{d},b\in\mathbb{R}$, 
$$\frac{1}{n}\sum_{i=1}^{n}\ell_{p}(-y_{i}(\bm{a}^{\top}x_{i}+b))\ge \frac{1}{n}\sum_{i=1}^{n}\ell_{p}(-y_{i}({\bm{a}^{*}}^{\top}x_{i}+b^{*})).$$
Therefore, for $\tilde{\bm{\theta}}_{L}:\|\tilde{\bm{\theta}}_{L}-\bm{\theta}^{*}_{L}\|_{2}\le\delta(\varepsilon)$ and any $a\in\mathbb{R}^{d},b\in\mathbb{R}$
\begin{align*}
\hat{L}_{n}(\bm{a},\tilde{\bm{\theta}}_{L},b;p)&=\frac{1}{n}\sum_{i=1}^{n}\ell_{p}(-y_{i}f(x_{i};\bm{a},\tilde{\bm{\theta}}_{L},b))=\frac{1}{n}\sum_{i=1}^{n}\ell_{p}(-y_{i}(\bm{a}^{\top}x_{i}+b))\\&\ge\frac{1}{n}\sum_{i=1}^{n}\ell_{p}(-y_{i}({\bm{a}^{*}}^{\top}x_{i}+b^{*}))\ge\frac{1}{n}\sum_{i=1}^{n}\ell_{p}(-y_{i}f(x_{i};a_{0}^{*},\bm{\theta}_{L}^{*},b^{*}))=\hat{L}_{n}(\bm{a}^{*},\bm{\theta}_{L}^{*},b^{*};p).
\end{align*}
This means that $(\bm{a}^{*},\bm{\theta}_{L}^{*},b^{*})$ is a local minima and $f(x_{i};\bm{a}^{*},\bm{\theta}_{L}^{*},b^{*})={\bm{a}^{*}}^{\top}x_{i}+b^{*}$ for all $i\in[n]$. This further indicates that 
$$\hat{R}_{n}(\bm{\theta}^{*})\ge\min_{\bm{w}\in\mathbb{R}^{d},b\in\mathbb{R}}\frac{1}{n}\sum_{i=1}^{n}\mathbb{I}\{y_{i}\neq \sgn(\bm{w}^{\top}x_{i}+b)\}>0.$$

\end{proof}

Now we consider the following distribution $\mathbb{P}_{\bm{X}\times Y}$ defined on the $\mathbb{R}^{d}$. Let $\mathbb{P}_{\bm{X}|Y=1}$ is a uniform distribution on the region $[1,2]\cup[-2, -1]\times \{0\}^{d-1}$ and $\mathbb{P}_{\bm{X}|Y=-1}$ is a uniform distribution on the region $\{0\}\times [1,2]\cup[-2,-1]\times\{0\}^{d-2}$. In addition, let $\mathbb{P}_{Y}(Y=1)=\mathbb{P}_{Y}(Y=-1)=0.5$ Clearly, $r_{+}=r_{-}=1<r=2$ and this distribution satisfies the assumptions in Theorem~\ref{thm::convex-finite-deep}. Furthermore, with probability at least $1-\frac{1}{4^{n-1}}$, there exists at least one sample in the following four regions: $[1,2]\times \{0\}^{d-1},[-2, -1]\times \{0\}^{d-1}, \{0\}\times [1,2]\times\{0\}^{d-2}$ and $\{0\}\times [-2,-1]\times\{0\}^{d-2}$ and this makes the samples in the dataset not linearly separable. 

\end{proof}

\clearpage
\subsection{Proof of Example~\ref{prop::quadratic-loss-linsep}}\label{appendix::prop-quadratic-loss-linsep}
\begin{example}
Let the distribution $\mathbb{P}_{\bm{X}\times Y}$ satisfy that $\mathbb{P}(Y=1)=\mathbb{P}(Y=-1)=0.5$, $\mathbb{P}(X=5/4|Y=1)=1$ and $\mathbb{P}(X|Y=-1)$ is a uniform distribution on the interval $[0,1]$.  For a linear model $f(x;a,b)=ax+b,$ $a,b\in\mathbb{R}$, then every global minimum $(a^{*},b^{*})$ of the population loss ${L}(a,b)=\mathbb{E}_{X\times Y}[(1-Yf(X;a,b))^{2}]$ satisfies $\mathbb{P}_{\bm{X}\times Y}[Y\neq \sgn(f(X;a^{*},b^{*}))]\ge 1/16$.
\end{example}
\begin{proof}
The proof is simple. We first consider a simpler form of the problem. Given the distribution $\mathbb{P}_{\bm{X}\times Y}$, the optimal linear estimator $\hat{\mathbb{E}}[Y|X]$ is 
$$\hat{\mathbb{E}}[Y|X]=\mathbb{E}[Y]+Cov(Y,X)Var^{-1}(X)(X-\mathbb{E}[X]).$$
Since $\mathbb{E}[Y]=0$, $Cov(Y, X)=\mathbb{E}[XY]-\mathbb{E}[X]\mathbb{E}[Y]=1$, $Var(X)>0$, $\mathbb{E}[X]=7/8$, the misclassification rate is $1/16$.

\end{proof}

\clearpage

\subsection{Proof of Example~\ref{prop::quadratic-loss-1} and~\ref{prop::quadratic-loss-2}}\label{appendix::prop-quadratic-loss}
In this subsection, we present two counterexamples to show that neither Theorem~\ref{thm::convex-finite-deep} nor~\ref{thm::linear-sep-deep} holds if we replace the loss function with the quadratic loss. 
\begin{example}\label{prop::quadratic-loss-1}
Let the distribution $\mathbb{P}_{\bm{X}\times Y}$ defined on $\mathbb{R}^{2}\times\{-1,1\}$ satisfy that  $\mathbb{P}(Y=1)=\mathbb{P}(Y=-1)=0.5$,  $\mathbb{P}(X=(\alpha, 0)|Y=1)=\mathbb{P}(X=(1, 0)|Y=1)=0.5$ and $\mathbb{P}({X=(0, \alpha)|Y=-1})=\mathbb{P}({X=(0, 1)|Y=-1})=0.5$. Assume that samples in the dataset $\mathcal{D}=\{(x_{i},y_{i})\}_{i=1}^{4n}$ are independently drawn from the distribution $\mathbb{P}_{\bm{X}\times Y}$. Assume that the network $f_{S}$ has $M\ge 2$ neurons and all neurons in the network $f_{S}$ are quadratic neurons, i.e., $\sigma(z)=z^{2}$. 
Then there exists an $\alpha\in[0,1]$ such that every global minimum of the empirical loss function $\hat{L}_{4n}(\bm{\theta})=\frac{1}{4n}\sum_{i=1}^{4n}(1-y_{i}f(x_{i};\bm{\theta}))^{2}$ has a training error greater  than $1/8$  with probability at least $\Omega(1/n^{3})$.
\end{example}
\textbf{Remark:} This is a counterexample for Theorem~\ref{thm::convex-finite-deep}. It is easy to check that the distribution satisfies assumption~\ref{assump::full-rank} and \ref{assump::different-subspaces}, where $r=2>\max\{1,1\}=\max\{r_{+},r_{-}\}$.
\begin{proof}
Let $\bm{X}=(X_{1},X_{2})$. Set the feedforward network $f_{D}$ to a constant. 
Since the positive and negative samples locate on two orthogonal subspaces, then it is easy to check that under this distribution, for any quadratic function of the form $g(X_{1},X_{2})=a_{1}X_{1}^{2}+a_{2}X_{2}^{2}+a_{0}$, there always exists a neural network of the form $f(X_{1},X_{2})=a_{0}+\sum_{j=1}^{M}a_{j}(w_{j1}X_{1}+w_{j2}X_{2})^{2}=a_{0}+\sum_{j=1}^{M}a_{j}(w_{j1}^{2}X_{1}^{2}+w_{j2}^{2}X_{2}^{2})$, $M\ge 2$ satisfying $$\mathbb{P}_{\bm{X}\times Y}(f(\bm{X})=g(\bm{X}))=1.$$
In addition, for any neural network $f(X_{1},X_{2})=a_{0}+\sum_{j=1}^{M}a_{j}(w_{j1}X_{1}+w_{j2}X_{2})^{2}$, there exists a quadratic function of the form  $g(X_{1},X_{2})=a_{1}X_{1}^{2}+a_{2}X_{2}^{2}+a_{0}$ satisfying 
$$\mathbb{P}_{\bm{X}\times Y}(f(\bm{X})=g(\bm{X}))=1.$$
 This indicates that the optimal neural network $f(x;\bm{\theta}^{*})$ should be the solution of 
$$\min_{a_{0}\in\mathbb{R},\bm{a}\in\mathbb{R}^{2}}\frac{1}{4n}\sum_{i=1}^{4n}\left(1-y_{i}\left(a_{0}+a_{1}(x_{i}^{(1)})^{2}+a_{2}(x_{i}^{(2)})^{2}\right)\right).$$

Let $n_{1}, n_{2},n_{3}$ and $n_{4}$ denote the number of samples at the point $(\alpha, 0), (1, 0), (0,\alpha)$ and $(0,1)$, respectively.
We only need to focus the case where $n_{1}=n_{2}=n_{3}=n_{4}=n$. In this case, the optimal linear estimator should be of the form
$$g(X^{2}_{1},X^{2}_{2};a_{0}^{*},a_{1}^{*},a_{2}^{*})=a_{1}^{*}(X_{1}^{2}-\hat{\mathbb{E}}X_{1}^{2})+a_{2}^{*}(X_{2}^{2}-\hat{\mathbb{E}}X_{2}^{2})=a_{1}^{*}\left(X_{1}^{2}-\frac{1+\alpha^{2}}{4}\right)+a_{2}^{*}\left(X_{2}^{2}-\frac{1+\alpha^{2}}{4}\right).$$
When $\alpha=1/2$, then $\frac{1+1/4}{4}=5/16>1/4=\alpha^{2}$ and $\frac{1+1/4}{4}=5/16<1$. Therefore, $(1+\alpha^{2})/4\in(\alpha^{2},1)$. In this case, for any $a_{1}^{*},a_{2}^{*}$, the training error cannot be smaller than $1/4$. This can be easily seen by investigating positive and negative samples separately. For positive samples at $(1,0)$, the output of the network is $g(1,0;a_{0}^{*},a_{1}^{*},a_{2}^{*})=a_{1}^{*}(1-(1+\alpha^{2})/4)$. For positive samples at $(\alpha, 0)$, the output of the network is $g(\alpha,0;a_{0}^{*},a_{1}^{*},a_{2}^{*})=a_{1}^{*}(\alpha^{2}-(1+\alpha^{2})/4)$. Since $\alpha^{2}<\frac{1+\alpha^{2}}{4}<1$, then if $a_{1}^{*}\neq 0$, then the network will misclassify all samples at $(\alpha, 0)$ or $(1, 0)$. This indicates that $a_{1}^{*}=0$ or training error is no smaller than $1/4$. Using the same analysis on the negative samples, we will have $a_{2}^{*}=0$ or training error is no smaller than $1/4$. This indicates that the output of the network is a constant equal to zero, which has a training error $1/2$. In all, the training error is no smaller than $1/4$. The probability of the case where $n_{1}=n_{2}=n_{3}=n_{4}$ is 
\begin{align*}
{{4n}\choose{2n}}{{2n}\choose{n}}^{2}\frac{1}{4^{4n}}>\frac{4^{2n}}{2n+1}\left(\frac{4^{n}}{n+1}\right)^{2}\frac{1}{4^{4n}}=\frac{1}{(2n+1)(n+1)^{2}}
\end{align*}
\end{proof}

\begin{example}\label{prop::quadratic-loss-2}
Let the distribution $\mathbb{P}_{\bm{X}\times Y}$ satisfy that  $\mathbb{P}(Y=1)=\mathbb{P}(Y=-1)=0.5$,  $\mathbb{P}(X=1+\alpha|Y=1)=\mathbb{P}(X=1+2\alpha|Y=1)=0.5$ and $\mathbb{P}({X=0|Y=-1})=\mathbb{P}({X=1|Y=-1})=0.5$. Assume that samples in the dataset $\mathcal{D}=\{(x_{i},y_{i})\}_{i=1}^{4n}$ are independently drawn from the distribution $\mathbb{P}_{\bm{X}\times Y}$. Assume that the network $f_{S}$ has $M\ge 1$ neurons and each neuron is a linear neuron $\sigma(z)=z$. 
If  $\alpha\in[0,1/6]$, then every global minimum of the empirical loss function $\hat{L}_{4n}(\bm{\theta})=\frac{1}{4n}\sum_{i=1}^{4n}(1-y_{i}f(x_{i};\theta))^{2}$ has a training error greater  than $1/8$  with probability at least $\Omega(1/n^{3})$.
\end{example}
\textbf{Remark:} This is counterexample for Theorem~\ref{prop::quadratic-loss-2}. It is easy to check that distribution is linearly separable. 
\begin{proof}
Let $n_{-1}, n_{1},n_{1+\alpha}$ denote the number of samples at the point $-1, 1$ and $1+\alpha$.
We only need to focus the case where $n_{-1}=n$, $n_{1}=n$ and $n_{1+\alpha}=2n$. Since the network is a linear network, then under this distribution, the optimal linear estimator should be of the form
$$f(x;\bm{\theta})=a^{*}\left(x-\frac{3+3\alpha}{4}\right).$$
If $a^{*}=0$, then the training error is $1/2$. If $a^{*}>0$, then the training error is $1/4$, due to the misclassification of all points at $x=1$. If $a^{*}<0$, then the training error is $3/4$, due to the misclassification of all points at $x=1+\alpha$ and $x=-1$. This means that the training error in this case should be greater or equal to $1/4$. The probability of this case is 
\begin{align*}
{{4n}\choose{2n}}{{2n}\choose{n}}^{2}\frac{1}{4^{4n}}>\frac{4^{2n}}{2n+1}\left(\frac{4^{n}}{n+1}\right)^{2}\frac{1}{4^{4n}}=\frac{1}{(2n+1)(n+1)^{2}}
\end{align*}
\end{proof}

\clearpage
\subsection{Proof of Proposition~\ref{prop::loss}}\label{appendix::prop-loss}

\begin{proposition}
Let $f:\mathbb{R}^{d}\rightarrow\mathbb{R}$ denote a feedforward network parameterized by $\bm{\theta}$ and let the dataset have $n$ samples. When the loss function $\ell_{p}$ satisfies assumption~\ref{assump::loss} and $p\ge1$, we have $\min_{\bm{\theta}}\hat{L}_{n}(\bm{\theta};p)=0$ if and only if $\min_{\bm{\theta}}\hat{R}_{n}(\bm{\theta})=0$. Furthermore, if $\min_{\bm{\theta}}\hat{R}_{n}(\bm{\theta})=0$, every global minimum $\bm{\theta}^{*}$ of the empirical loss $\hat{L}_{n}(\bm{\theta};p)$ has zero training error, i.e., $\hat{R}_{n}(\bm{\theta}^{*})=0$. 
\end{proposition}
\textbf{Remark:} Using the same proof shown as follows, we can show that Proposition~\ref{prop::loss} holds for any multilayer network architectures satisfying that for any set of parameters $\bm{\theta}$ and any real numbers $a, b\in\mathbb{R}$, there always exists a set of parameters $\tilde{\bm{\theta}}$ such that $f(x;\tilde{\bm{\theta}})=a(f(x;\bm{\theta})-b)$ holds for all $x$. It is easy to check that both network architectures in Fig.~\ref{fig::network} satisfy this condition.  

\begin{proof}
We first prove the ``only if'' part. The proof is trivial since, by definition $\ell_{p}(z)\ge \mathbb{I}\{z\ge0\}$, then 
\begin{align*}
\hat{R}_{n}(\bm{\theta})=\frac{1}{n}\sum_{i=1}^{n}\mathbb{I}\{y_{i}\neq \sgn(f(x_{i};\bm{\theta}))\}\le \frac{1}{n}\sum_{i=1}^{n}\mathbb{I}\{y_{i}f(x_{i};\bm{\theta})\le 0\}\le \frac{1}{n}\sum_{i=1}^{n}\ell_{p}(-y_{i}f(x_{i};\bm{\theta}))=\hat{L}_{n}(\bm{\theta};p).
\end{align*}
Therefore, if $\min_{\bm{\theta}}\hat{L}_{n}(\bm{\theta};p)=0$ then $\min_{\bm{\theta}}\hat{R}_{n}(\bm{\theta})=0$.

Next, we prove the ``if'' part. If $\min_{\bm{\theta}}\hat{R}_{n}(\bm{\theta})=0$, then there exists a set of parameter $\bm{\theta}^{*}$ such that $\mathbb{I}\{y_{i}\neq \sgn(f(x;\bm{\theta}^{*}))\}=0$ holds for all $i\in[n]$. This indicates that $f(x_{i};\bm{\theta}^{*})\ge 0$ for all $i:y_{i}=1$ and $f(x_{i};\bm{\theta}^{*})<0$ for all $i:y_{i}=-1$. This means that there exists two real numbers $c_{1}<c_{2}$ such that $f(x_{i};\bm{\theta}^{*})>c_{2}$ holds for all $i:y_{i}=1$ and $f(x_{i};\bm{\theta}^{*})<c_{1}$ holds for all $i:y_{i}=-1$. Now, we define a new network  ${f}(x;\tilde{\bm{\theta}})=\alpha(f(x;\bm{\theta}^{*})-\frac{c_{1}+c_{2}}{2})$. Therefore, for this network ${f}(x;\tilde{\bm{\theta}})$, we have ${f}(x_{i};\tilde{\bm{\theta}})>\alpha(c_{2}-c_{1})/2$ holds for all $i:y_{i}=1$ and ${f}(x_{i};\tilde{\bm{\theta}})<-\alpha(c_{2}-c_{1})/2$ holds for all $i:y_{i}=-1$. Since $\ell_{p}(z)=0$ iff $z\le -z_{0}$, then by choosing $\alpha>\frac{2z_{0}}{c_{2}-c_{1}},$ we have 
$$y_{i}{f}(x_{i};\tilde{\bm{\theta}})>z_{0}\quad\text{holds for }\forall i\in[n].$$
This means that $\hat{L}_{n}(\tilde{\bm{\theta}};p)=0.$
Now we need to show that there exits a set of parameter $\tilde{\bm{\theta}}$ such that $$f(x;\tilde{\bm{\theta}})=\alpha\left(f(x;\bm{\theta}^{*})-\frac{c_{1}+c_{2}}{2}\right).$$
Since the output of the neural network can be written as 
$$f(x;\bm{\theta})=a_{0}+\sum_{j=1}^{M_{L}}a_{j}\sigma(\bm{w}_{j}^{\top}\bm{\Phi}(x;\bm{\theta})+b_{j}),$$
where $M_{L}$ denotes the number of neurons in the last layer and $\bm{\Phi}(x_{i};\bm{\theta})$ denotes the outputs from the previous layers. Then by shifting $a_{0}$ and scaling $a_{j}$, we have 
\begin{align*}
f(x;\tilde{\bm{\theta}})&=\alpha\left(f(x;\bm{\theta}^{*})-\frac{c_{1}+c_{2}}{2}\right)\\
&=a_{0}^{*}-\frac{\alpha(c_{1}+c_{2})}{2}+\sum_{j=1}^{M_{L}}\alpha a^{*}_{j}\sigma({\bm{w}^{*}}^{\top}\bm{\Phi}(x;\bm{\theta}^{*})+b_{j}^{*})\\
&=\tilde{a}_{0}+\sum_{j=1}^{M_{L}}\tilde{a}_{j}\sigma({\bm{w}^{*}}^{\top}\bm{\Phi}(x;\bm{\theta}^{*})+b_{j}^{*}).
\end{align*}
Therefore, this means that there exists a set of parameters $\tilde{\bm{\theta}}$ such that $\hat{L}_{n}(\tilde{\bm{\theta}};p)=0$, i.e., $\min_{\bm{\theta}}\hat{L}_{n}(\bm{\theta};p)=0.$ This means, the global minimum of the empirical loss $\hat{L}_{n}(\bm{\theta};p)$ is zero. Furthermore, since $\hat{R}_{n}(\bm{\theta})\le\hat{L}_{n}(\bm{\theta};p)$ holds for all $\bm{\theta}$, then every global minimum of the empirical loss has zero training error.

\end{proof}

\subsection{Proof of Proposition~\ref{prop::logit-convex-finite}}\label{appendix::prop-logit}
\begin{proposition}
Assume that the loss function is the logistic loss, i.e., $\ell(z)=\log_{2}(1+e^{z})$.
Assume that assumptions~\ref{assump::full-rank}-\ref{assump::neurons} are satisfied. 
Assume that  samples in the dataset $\mathcal{D}=\{(x_{i},y_{i})\}_{i=1}^{n}, n\ge 1$ are independently drawn from the distribution $\mathbb{P}_{\bm{X}\times Y}$.  Assume that the number of neurons $M$ in the network $f_{S}$ satisfies $M\ge 2\max\{\frac{n}{\Delta r},r_{+},r_{-}\}$, where $\Delta r=r-\max\{r_{+},r_{-}\}$. If a set of real parameters $\bm{\theta}^{*}$ denotes a critical point of the empirical loss $\hat{L}_{n}(\bm{\theta})$, then $\bm{\theta}^*$ is a saddle point.
\end{proposition}

\begin{proof}
We first recall some notations defined in the paper. The output of the neural network is 
$$f(x;\bm{\theta})=f_{S}(x;\bm{\theta}_{S})+f_{D}(x;\bm{\theta}_{D}),$$
where $f_{S}(x;\bm{\theta}_{S})$ is the single layer neural network parameterized by $\bm{\theta}_{S}$, i.e., 
$$f_{S}(x;\bm{\theta}_{S})=a_{0}+\sum_{j=1}^{M}a_{j}\sigma\left(\bm{w}_{j}^{\top}x\right),$$
and $f_{D}(x;\bm{\theta}_{D})$ is a deep neural network parameterized by $\bm{\theta}_{D}$. 
The empirical loss function is given by
$$\hat{L}_{n}(\bm{\theta})=\hat{L}_{n}(\bm{\theta}_{S},\bm{\theta}_{D})=\frac{1}{n}\sum_{i=1}^{n}\ell(-y_{i}f(x_{i};\bm{\theta})).$$
We assume that there exists a local minimum $\bm{\theta}^{*}=(\bm{\theta}^{*}_{S},\bm{\theta}_{D}^{*})$.
We next complete the proof by proving the following two claims: 

\begin{claim}\label{claim::logit-1} If there exists $j\in[M]$ such that $a^{*}_{j}=0$, then $\bm{\theta}^{*}$ is not a local minimum. \end{claim}

\begin{claim}\label{claim::logit-2} If $a^{*}_{j}\neq 0$ for all $j\in [M]$, then $\bm{\theta}^{*}$ is not a local minimum.\end{claim}

Therefore, these two claims contradict with the assumption that $\bm{\theta}^{*}=(\bm{\theta}^{*}_{S},\bm{\theta}_{D}^{*})$ is a local minimum. Therefore, every critical point is not a local minimum. In addition, it is very easy to show that every critical point is not a local maximum, since the loss function is strictly convex with respect to $a_{0}$. Therefore, every critical point is a saddle point. 

\textbf{(a) Proof of Claim \ref{claim::logit-1}.} In this part, we prove that if there exists $j\in [M]$ such that $a^{*}_{j}=0$, then $\bm{\theta}^{*}$ is not a local minima. Without loss of generality, we assume that $a_{1}^{*}=0$. Using the same analysis presented in the proof of Theorem~\ref{thm::convex-finite-deep}, we have
\begin{equation*}
\sum_{i=1}^{n}\ell'(-y_{i}f(x_{i};\bm{\theta}))(-y_{i})\sigma''\left({\bm{w}_{1}^{*}}^{\top}x_{i}\right)x_{i}x_{i}^{\top}=\bm{0}_{d\times d}.
\end{equation*}
By assumption that there exists a set of orthogonal basis $\mathcal{E}=\{\bm{e}_{1},...,\bm{e}_{d}\}$ in $\mathbb{R}^{d}$ and a subset $\mathcal{U}_{+}\subseteq \mathcal{E}$  such that $\mathbb{P}_{\bm{X}|Y}(\bm{X}\in\text{Span}(\mathcal{U}_{1})|Y=1)=1$ and by assumption that $r=|\mathcal{U}_{+}\cup \mathcal{U}_{-}|>\max\{r_{+},r_{-}\}=\max\{|\mathcal{U}_{+}|,|\mathcal{U}_{-}|\}$, then the set $\mathcal{U}_{+}\backslash \mathcal{U}_{-}$ is not an empty set. It is easy to show that for any vector $\bm{v}\in\mathcal{U}_{+}\backslash\mathcal{U}_{-}$, $\mathbb{P}_{\bm{X}\times Y}(\bm{v}^{\top}\bm{X}=0|Y=1)=0$. We prove it by contradiction. If we assume $p=\mathbb{P}_{\bm{X}\times Y}(\bm{v}^{\top}\bm{X}=0|Y=1)>0$, then for random vectors $\bm{X}_{1},...,\bm{X}_{|\mathcal{U}_{+}|}$ independently drawn from the conditional distribution $\mathbb{P}_{\bm{X}| Y=1}$, 
\begin{align*}
\mathbb{P}_{\bm{X} |Y=1}\left(\bigcup_{i=1}^{|\mathcal{U}_{+}|}\left\{\bm{v}^{\top}\bm{X}_{i}=0\right\}\Bigg|Y=1\right)&=\prod_{i=1}^{|\mathcal{U}_{+}|}\mathbb{P}_{\bm{X}|Y=1}\left(\bm{v}^{\top}\bm{X}_{i}=0|Y=1\right)=p^{|\mathcal{U}_{+}|}>0.
\end{align*} 
Furthermore, since $\bm{X}_{1},...,\bm{X}_{|\mathcal{U}_{+}|}\in\text{Span}(\mathcal{U}_{+})$, $\bm{v}^{\top}\bm{X}_{i}=0$, $i=1,...,|\mathcal{U}_{+}|$ and $\bm{v}\in\mathcal{U}_{+}$, then the rank of the matrix $\left(\bm{X}_{1},...,\bm{X}_{|\mathcal{U}_{+}|}\right)$ is at most $|\mathcal{U}_{+}|-1$ and this indicates that the matrix is not a full rank matrix with probability $p^{|\mathcal{U}_{+}|}>0$. This leads to the contradiction with the Assumption~\ref{assump::full-rank}. Thus, with probability 1, $\bm{v}^{\top}x_{i}\neq 0$ for all $i:y_{i}=1$ and $\bm{v}^{\top}x_{i}= 0$ for all $i:y_{i}=-1$.

\textbf{Proof of Claim \ref{claim::logit-2}:}  Now we have proved that $a_{j}^{*}\neq 0$ for all $j\in[M]$. Here, we define $M_{0}=\lceil M/2\rceil$.
Since $$M_{0}\ge\max\{r_{+},r_{-}\},$$
and $\max\{r_{+},r_{-}\}+\min\{r_{+},r_{-}\}\ge r$, then 
$$2M_{0}\ge2\max\{r_{+},r_{-}\}>2r-r_{+}-r_{-}\ge 2\min\{r-r_{+},r-r_{-}\}\triangleq 2K.$$
Thus, there exists $a_{i_{1}},..., a_{i_{M_{0}}}$, $i_{1}<i_{2}<...<i_{M_{0}}$ such that $$\sgn(a_{i_{1}})=...=\sgn(a_{i_{M_{0}}}).$$
Without loss of generality, we assume that $\sgn(a_{1})=...=\sgn(a_{M_{0}})=+1$.

Now we prove the claim \ref{claim::logit-2}. 
First, we consider the Hessian matrix $H(\bm{w}_{1}^{*},...,\bm{w}_{M_{0}}^{*})$. Since $\bm{\theta}^{*}$ is a local minima with $\error>0$, then 
\begin{equation*}
F(\bm{u}_{1},...,\bm{u}_{M_{0}})=\sum_{j=1}^{M_{0}}\sum_{k=1}^{M_{0}}\bm{u}_{j}^{\top}\nabla^{2}_{\bm{w}_{j},\bm{w}_{k}}\lossc \bm{u}_{k}\ge 0
\end{equation*}
holds for any vectors $\bm{u}_{1},...,\bm{u}_{M_{0}}\in\mathbb{R}^{d}$. 
Since 
\begin{align*}
\nabla_{\bm{w}_{j}}^{2}\lossc&=a_{j}^{*}\sum_{i=1}^{n}\ell'(-y_{i}f(x_{i};\bm{\theta}^{*}))(-y_{i})\sigma''\left({\bm{w}_{j}^{*}}^{\top}x_{i}\right)x_{i}x_{i}^{\top}\\
&\quad+{a_{j}^{*}}^{2}\sum_{i=1}^{n}\ell''(-y_{i}f(x_{i};\bm{\theta}^{*}))\left[\sigma'\left({\bm{w}_{j}^{*}}^{\top}x_{i}\right)\right]^{2}x_{i}x_{i}^{\top},
\end{align*}
and 
\begin{align*}
\nabla_{\bm{w}_{j},\bm{w}_{k}}^{2}\loss&={a_{j}^{*}}a_{k}^{*}\sum_{i=1}^{n}\ell''(-y_{i}f(x_{i};\bm{\theta}^{*}))\left[\sigma'\left({\bm{w}_{j}^{*}}^{\top}x_{i}\right)\right]\left[\sigma'\left({\bm{w}_{k}^{*}}^{\top}x_{i}\right)\right]x_{i}x_{i}^{\top}.
\end{align*}
Thus, we have for any $\bm{u}_{1},...,\bm{u}_{M_{0}}\in\mathbb{R}^{d}$,
\begin{align*}
F(\bm{u}_{1},...,\bm{u}_{M_{0}})
&=-2\sum_{i=1}^{n}\left[\ell'(-y_{i}f(x_{i};\bm{\theta}^{*}))y_{i}\sum_{j=1}^{M_{0}}\left[a_{j}^{*}\sigma''\left(\bm{w}_{j}^{*}x_{i}\right)\left(\bm{u}_{j}^{\top}x_{i}\right)^{2}\right]\right]\\
&\quad +4\sum_{i=1}^{n}\left[\ell''(-y_{i}f(x_{i};\bm{\theta}^{*}))\left(\sum_{j=1}^{M_{0}}a_{j}^{*}\sigma'\left({\bm{w}_{j}^{*}}^{\top}x_{i}\right)\left(\bm{u}_{j}^{\top}x_{i}\right)\right)^{2}\right].
\end{align*}
Now we  find  some coefficients $\alpha_{1},...,\alpha_{M_{0}}$, not all zero and vectors $\bm{u}_{1},...,\bm{u}_{M_{0}}$ satisfying
$$\sum_{j=1}^{M_{0}}\alpha_{j}\sigma'\left({\bm{w}_{j}^{*}}^{\top}x_{i}\right)\bm{u}_{j}^{\top}x_{i}=0,\quad \forall i\in[n],$$
and 
$$\forall i:y_{i}=-1 \text{ and }\forall j\in[M_{0}],\quad \bm{u}_{j}^{\top}x_{i}=0.$$
Since $\bm{\bm{\theta}}^{*}$ is a local minima, then by Lemma~\ref{lemma::nec-single}, we have 
\begin{equation*}
\sum_{i=1}^{n}\ell'(-y_{i}f(x_{i};\bm{\theta}^{*}))y_{i}\sigma'({\bm{w}^{*}_{j}}^{\top}x_{i})x_{i}=\bm{0}_{d}.
\end{equation*}
Consider the orthogonal vectors $\bm{e}_{1},...,\bm{e}_{K}$ from the set of orthogonal basis $\bm{e}_{1},...,\bm{e}_{d}$ satisfying that, with probability 1, $\forall j\in[K]$, $\forall i:y_{i}=-1$, $\bm{e}_{j}^{\top}x_{i}=0$ and $\forall i:y_{i}=1$, $\bm{e}_{j}^{\top}x_{i}\neq 0$. Then, considering the following set of linear equations 
\begin{align*}
&\sum_{i=1}^{n}\ell'(-y_{i}f(x_{i};\bm{\theta}^{*}))y_{i}\sigma'({\bm{w}^{*}_{1}}^{\top}x_{i})\left(\bm{e}_{1}^{\top}x_{i}\right)=0,...,
\sum_{i=1}^{n}\ell'(-y_{i}f(x_{i};\bm{\theta}^{*}))y_{i}\sigma'({\bm{w}^{*}_{M_{0}}}^{\top}x_{i})\left(\bm{e}_{1}^{\top}x_{i}\right)=0,\\
&...\\
&\sum_{i=1}^{n}\ell'(-y_{i}f(x_{i};\bm{\theta}^{*}))y_{i}\sigma'({\bm{w}^{*}_{1}}^{\top}x_{i})\left(\bm{e}_{K}^{\top}x_{i}\right)=0,...,
\sum_{i=1}^{n}\ell'(-y_{i}f(x_{i};\bm{\theta}^{*}))y_{i}\sigma'({\bm{w}^{*}_{M_{0}}}^{\top}x_{i})\left(\bm{e}_{K}^{\top}x_{i}\right)=0.
\end{align*}
These equations can be rewritten in a matrix form
\begin{equation*}
\underbrace{
\left(\begin{matrix}
\sigma'({\bm{w}^{*}_{1}}^{\top}x_{1})\left(\bm{e}_{1}^{\top}x_{1}\right)&...&\sigma'({\bm{w}^{*}_{1}}^{\top}x_{n})\left(\bm{e}_{1}^{\top}x_{n}\right)\\
...&...&...\\
\sigma'({\bm{w}^{*}_{M_{0}}}^{\top}x_{1})\left(\bm{e}_{1}^{\top}x_{1}\right)&...&\sigma'({\bm{w}^{*}_{M_{0}}}^{\top}x_{n})\left(\bm{e}_{1}^{\top}x_{n}\right)\\
...&...&...\\
\sigma'({\bm{w}^{*}_{1}}^{\top}x_{1})\left(\bm{e}_{K}^{\top}x_{1}\right)&...&\sigma'({\bm{w}^{*}_{1}}^{\top}x_{n})\left(\bm{e}_{K}^{\top}x_{n}\right)\\
...&...&...\\
\sigma'({\bm{w}^{*}_{M_{0}}}^{\top}x_{1})\left(\bm{e}_{K}^{\top}x_{1}\right)&...&\sigma'({\bm{w}^{*}_{M_{0}}}^{\top}x_{n})\left(\bm{e}_{K}^{\top}x_{n}\right)
\end{matrix}\right)_{(KM_{0}\times n)}}_{\bm{P}}
\underbrace{\left(\begin{matrix}
\ell'(-y_{1}f(x_{1};\bm{\theta}^{*}))y_{1}\\
\ell'(-y_{2}f(x_{2};\bm{\theta}^{*}))y_{2}\\
...\\
...\\
...\\
...\\
...\\
\ell'(-y_{n}f(x_{1};\bm{\theta}^{*}))y_{n}\\
\end{matrix}\right)}_{\bm{q}}=\bm{0}_{n}
\end{equation*}
or 
$$\bm{P}\bm{q}=\bm{0}_{n}.$$
Since $ M_{0}K\ge MK/2\ge n$, then if rank$(\bm{P})=n$, we should have $\bm{q}=\bm{0}_{n}$ and this indicates that $\ell'(-y_{i}f(x_{i};\bm{\theta}^{*}))=0$ for all $i\in[n]$ and this contradicts with the fact that $\ell'(z)=\frac{1}{1+e^{-z}}>0$ for all $z\in\mathbb{R}$. Therefore, rank$(\bm{P})<n\le M_{0}K$. This means the raw vectors of the matrix $\bm{P}$ is linearly dependent and thus we have that there exists coefficients vectors $(\beta_{11},...,\beta_{1K}),...,(\beta_{M_{0}1},...,\beta_{M_{0}K})$, not all zero vectors, such that 
$$\sum_{s=1}^{K}\sum_{j=1}^{M_{0}}\sigma'({\bm{w}_{j}^{*}}^{\top}x_{i})\beta_{js}(\bm{e}_{s}^{\top}x_{i})=0,\quad \forall i\in[n],$$
or 
$$\sum_{j=1}^{M_{0}}a_{j}^{*}\sigma'({\bm{w}_{j}^{*}}^{\top}x_{i})\left(\frac{1}{a_{j}^{*}}\sum_{s=1}^{K}\beta_{js}\bm{e}_{s}\right)^{\top}x_{i}=0,\quad \forall i\in[n],$$
Define $\bm{u}_{j}=\frac{1}{a_{j}^{*}}\sum_{s=1}^{K}\beta_{js}\bm{e}_{s}$ for $j=1,...,M_{0}$, then we have 
\begin{equation}\label{eq::thm4-all-zero}\sum_{j=1}^{M_{0}}a_{j}^{*}\sigma'({\bm{w}_{j}^{*}}^{\top}x_{i})\bm{u}_{j}^{\top}x_{i}=0,\quad \forall i\in[n].\end{equation}
Furthermore, since $\bm{u}_{j}\in \text{ Span}(\{\bm{e}_{1},...,e_{K}\})$, and with probability 1, $\forall i:y_{i}=-1$ and $\forall j\in[K]$, $\bm{e}_{j}^{\top}x_{i}=0$, then we have that
$\forall j\in[M]$ and $\forall i:y_{i}=-1$: $\bm{u}_{j}^{\top}x_{i}=0$. Thus,
\begin{align}
F(\bm{u}_{1},...,\bm{u}_{M_{0}})
&=-2\sum_{i=1}^{n}\left[\ell'(-y_{i}f(x_{i};\bm{\theta}^{*}))y_{i}\sum_{j=1}^{M_{0}}\left[a_{j}^{*}\sigma''\left(\bm{w}_{j}^{*}x_{i}\right)\left(\bm{u}_{j}^{\top}x_{i}\right)^{2}\right]\right]\notag&&\text{by Eq.~\eqref{eq::thm4-all-zero}}\\
&=-2\sum_{i:y_{i}=1}\left[\ell'(-y_{i}f(x_{i};\bm{\theta}^{*}))\sum_{j=1}^{M_{0}}\left[a_{j}^{*}\sigma''\left(\bm{w}_{j}^{*}x_{i}\right)\left(\bm{u}_{j}^{\top}x_{i}\right)^{2}\right]\right]\ge 0\label{eq::thm-F}.
\end{align}

Since $\sigma''(z)>0$ for all $z\in\mathbb{R}$ and $a_{j}^{*}>0$ for all $j\in[M_{0}]$, then we have 
$$\ell'(-y_{i}f(x_{i};\bm{\theta}^{*}))\sum_{j=1}^{M_{0}}\left[a_{j}^{*}\sigma''\left(\bm{w}_{j}^{*}x_{i}\right)\left(\bm{u}_{j}^{\top}x_{i}\right)^{2}\right]\ge 0,\quad \forall i:y_{i}=1$$
and this leads to 
$$F(\bm{u}_{1},...,\bm{u}_{M_{0}})\le 0.$$
Together with Eq.~\eqref{eq::thm-F}, we have $$F(\bm{u}_{1},...,\bm{u}_{M_{0}})= 0$$
and thus 
\begin{equation}\label{eq::thm3-cond-i}\ell'(-y_{i}f(x_{i};\bm{\theta}^{*}))\sum_{j=1}^{M_{0}}\left[a_{j}^{*}\sigma''\left(\bm{w}_{j}^{*}x_{i}\right)\left(\bm{u}_{j}^{\top}x_{i}\right)^{2}\right]= 0,\quad \forall i:y_{i}=1.\end{equation}
Now we split the index $\{i\in[n]:y_{i}=1\}$ set into two disjoint subset $C_{0}, C_{1}$:
$$C_{0}=\{i\in[n]: y_{i}=1,\text{ and }\exists j\in[M_{0}],  \bm{u}_{j}^{\top}x_{i}\neq 0\},\quad C_{1}=\{i\in[n]:y_{i}=1 \text{ and }\forall j\in[M_{0}], \bm{u}_{j}^{\top}x_{i}= 0\}.$$
Clearly, for all $i\in C_{0}$, by the fact that $a_{j}> 0$ for all $j\in[M_{0}]$ and $\sigma''(z)>0$ for all $z\in\mathbb{R}$, we have 
$$\sum_{j=1}^{M_{0}}\left[a_{j}^{*}\sigma''\left(\bm{w}_{j}^{*}x_{i}\right)\left(\bm{u}_{j}^{\top}x_{i}\right)^{2}\right]>0,$$
and this leads to 
$$\ell'(-y_{i}f(x_{i};\bm{\theta}^{*}))=0,\quad \forall i\in C_{0},$$
which contradict with the fact that $\ell'(z)>0$ for all $z\in\mathbb{R}$. Therefore, $C_{0}=\emptyset$.
Now we need to consider the index set $C_{1}$. First, it is easy to show that with probability 1, $|C_{1}|< r_{+}\le M_{0}$. This is due to the fact that there exists a non-zero vector $\bm{u}_{j}$, such that $\bm{u}_{j}^{\top}x_{i}=0$ for all $i\in C_{1}$ and that $\bm{u}_{j}\in\text{Span}(\{\bm{e}_{1},...,\bm{e}_{K}\})$. Therefore, $\bm{u}_{j}^{\top}x_{i}=\sum_{s=1}^{K}(\bm{u}_{j}^{\top}\bm{e}_{s})(x_{i}^{\top}\bm{e}_{s})=\sum_{s=1}^{r_{+}}(\bm{u}_{j}^{\top}\bm{e}_{s})(x_{i}^{\top}\bm{e}_{s})=0$ holds for all $i\in C_{1}$. If $|C_{1}|\ge r_{+}$, then with probability 1, the matrix $$\left(\begin{matrix}\bm{e}_{1}^{\top}x_{1}&...&\bm{e}_{r_{+}}^{\top}x_{1}\\
...&...&...\\
\bm{e}_{1}^{\top}x_{r_{+}}&...&\bm{e}_{r_{+}}^{\top}x_{r_{+}}\\
\end{matrix}\right)$$
has the full rank equal to $r_{+}$ and this makes $\bm{u}_{j}^{\top}\bm{e}_{s}=0$ for all $s\in[k]$. This contradicts with the fact that $\bm{u}_{j}\in\text{Span}(\{\bm{e}_{1},...,\bm{e}_{K}\})$ and $\bm{u}_{j}$ is not a zero vector. Thus, $|C_{1}|<r_{+}\le M_{0}$. Now we consider the function $F$, since $\forall i\in C_{0}:\ell'(-y_{i}f(x_{i};\bm{\theta}^{*}))=0$, then for all $\bm{u}_{1},...,\bm{u}_{M_{0}}$,
\begin{align}
F(\bm{u}_{1},...,\bm{u}_{M_{0}})&=-2\sum_{i\in C_{1}}\left[\ell'(-y_{i}f(x_{i};\bm{\theta}^{*}))\sum_{j=1}^{M_{0}}\left[a_{j}^{*}\sigma''\left(\bm{w}_{j}^{*}x_{i}\right)\left(\bm{u}_{j}^{\top}x_{i}\right)^{2}\right]\right]\notag\\
&\quad +4\sum_{i\in C_{1}}\left[\ell''(-y_{i}f(x_{i};\bm{\theta}^{*}))\left(\sum_{j=1}^{M_{0}}a_{j}^{*}\sigma'\left({\bm{w}_{j}^{*}}^{\top}x_{i}\right)\left(\bm{u}_{j}^{\top}x_{i}\right)\right)^{2}\right]\notag
\end{align}
Now we set $\bm{u}_{j}=\alpha_{j} \bm{e}_{1}$, $j=1,...,M_{0}$ for some scalar $\alpha_{j}$. Now we only need find $\alpha_{1},...,\alpha_{M_{0}}$ such that
$$\sum_{j=1}^{M_{0}}\alpha_{j}a_{j}^{*}\sigma'\left({\bm{w}_{j}^{*}}^{\top}x_{i}\right)\bm{e}_{1}^{\top}x_{i}=\bm{0},\quad \forall i\in C_{1}.$$
Since $|C_{1}|\le M_{0}-1<M_{0}$, then there exists $\alpha^{*}_{1},...,\alpha^{*}_{M_{0}}$, not all zeros, such that 
$$\sum_{j=1}^{M_{0}}\alpha^{*}_{j}a_{j}^{*}\sigma'\left({\bm{w}_{j}^{*}}^{\top}x_{i}\right)\bm{e}_{1}^{\top}x_{i}={0},\quad \forall i\in C_{1}.$$
Then by setting $\bm{u}_{j}=\alpha^{*}_{j} \bm{e}_{1}$, we have 
\begin{align*}
F(\bm{u}_{1},...,\bm{u}_{M_{0}})&=-2\sum_{i\in C_{1}}\left[\ell'(-y_{i}f(x_{i};\bm{\theta}^{*}))\sum_{j=1}^{M_{0}}\left[|\alpha_{j}^{*}|^{2}a_{j}^{*}\sigma''\left(\bm{w}_{j}^{*}x_{i}\right)\left(\bm{e}_{1}^{\top}x_{i}\right)^{2}\right]\right]\ge 0.\\
\label{eq::thm-F2}.
\end{align*}
Similarly, since $|\alpha_{1}|,...,|\alpha_{M_{0}}|$ are not all zeros, $a_{j}^{*}>0$ for all $j\in[M_{0}]$, $\sigma''(z)>0$ for all $z\in\mathbb{R}$ and $\bm{e}_{1}^{\top}x_{i}\neq 0$ holds for all $i$ with probability 1, then 
$$\ell'(-y_{i}f(x_{i};\bm{\theta}^{*}))=0,\quad \forall i\in C_{1}.$$
Therefore, this indicates that 
$$\ell'(-y_{i}f(x_{i};\bm{\theta}^{*}))=0,\quad  \forall i:y_{i}=1.$$
Since $\ell'(z)>0$ holds for all $z\in\mathbb{R}$, then this leads to the contradiction. Therefore, $\bm{\theta}^{*}$ is not a local minima.

\end{proof}

\clearpage

\subsection{Proof of Proposition~\ref{prop-logit-linear-sep}}\label{appendix::logit-linear-sep}
\setcounter{proposition}{12}
\begin{proposition}\label{prop-logit-linear-sep}
 Assume that the loss function $\ell$ is the logistic loss, i.e., $\ell(z)=\log_{2}(1+e^{z})$. 
 Assume that the network architecture satisfies assumption~\ref{assump::shortcut-connection}.
 Assume that  samples in the dataset $\mathcal{D}=\{(x_{i},y_{i})\}_{i=1}^{n}, n\ge 1$ are independently drawn from a distribution satisfying assumption~\ref{assump::linear-sep}. Assume that the single layer network $f_{S}$ has $M\ge1$ neurons and neurons $\sigma$ in the network $f_{S}$ are twice differentiable and satisfy $\sigma'(z)>0$ for all $z\in\mathbb{R}$. If a set of real parameters $\bm{\theta}^{*}=(\bm{\theta}^{*}_{S},\bm{\theta}^{*}_{D})$ denotes a local minimum of the loss function $\hat{L}_{n}(\bm{\theta}^{}_{S},\bm{\theta}^{}_{D};p)$, $p\ge 3$, then  $\hat{R}_{n}(\bm{\theta}^{*}_{S},\bm{\theta}^{*}_{D})=0$ holds with probability one.
\end{proposition}

\begin{proof} 
We first prove that, if a set of real parameters $\bm{\theta}^{*}$ denotes a critical point, then $\bm{\theta}^{*}$ is a saddle point. We prove it by contradiction. We assume that $\bm{\theta}^{*}$ denotes a local minima. By assumption that $\bm{\theta}^{*}=(\bm{\theta}_{1}^{*},\bm{\theta}_{2}^{*})$ is a local minima and by the necessary condition presented in Lemma~\ref{lemma::nec-single}, we have 
\begin{align*}
&\sum_{i=1}^{n}\ell'(-y_{i}f(x_{i};\bm{\theta}^{*}))y_{i}\sigma'({\bm{w}^{*}_{j}}^{\top}x_{i})x_{i}=\bm{0}_{d}.
\end{align*}
Thus,  for any $\bm{w}\in\mathbb{R}^{d}$, we have
\begin{equation*}
\sum_{i=1}^{n}\ell'(-y_{i}f(x_{i};\bm{\theta}^{*}))\sigma'({\bm{w}^{*}_{j}}^{\top}x_{i})y_{i}(\bm{w}^{\top}x_{i})=0.
\end{equation*}
Furthermore, for the cross entropy loss function, we have  $$\ell'(z)=\frac{1}{1+\exp(-z)}> 0,\quad \forall z\in\mathbb{R}.$$
Thus, by assumption that $\sigma'(z)>0$ for all $z\in\mathbb{R}$ and assumption that there exists a vector $\bm{w}\in\mathbb{R}^{d}$ such that $\mathbb{P}_{\bm{X}\times Y}(Y(\bm{w}^{\top}X)>0)=1$, then there exists a constant $c$ such that for all samples in the dataset $i\in[n]$,
$$y_{i}\bm{w}^{\top}x_{i}>c>0.$$ Thus, we have 
\begin{align*}
0=\sum_{i=1}^{n}\ell'(-y_{i}f(x_{i};\bm{\theta}^{*}))\sigma'({\bm{w}^{*}_{j}}^{\top}x_{i})y_{i}(\bm{w}^{\top}x_{i})\ge c\sum_{i=1}^{n}\ell'(-y_{i}f(x_{i};\bm{\theta}^{*}))\sigma'({\bm{w}^{*}_{j}}^{\top}x_{i})> 0,
\end{align*}
and this leads to the contradiction.

\end{proof}

\clearpage
\setcounter{proposition}{9}
\subsection{Proof of Proposition~\ref{prop::logit-general}}\label{appendix::prop-logit-general}
\begin{proposition}
Assume the dataset ${\mathcal{D}=\{(x_{i},y_{i})\}_{i=1}^{n}}$ is consisted of both positive and negative samples. Assume that $f(x;\bm{\theta})$ is a feedforward network parameterized by $\bm{\theta}$. Assume that the loss function is logistic, i.e., $\ell(z)=\log_{2}\left(1+e^{z}\right)$. If the real parameters $\bm{\theta}^{*}$ denote a critical point of the empirical loss $\hat{L}_{n}(\bm{\theta}^{*})$, then $\hat{R}_{n}(\bm{\theta}^{*})>0$.
\end{proposition}
\begin{proof}
We prove a general statement claiming that the proposition~\ref{prop::logit-general} holds for all differentiable loss functions satisfying $\ell'(z)>0$ for all $z\in\mathbb{R}$. We note that the following claim holds under the assumptions in Proposition~\ref{prop::logit-general}.
\begin{claim}
If the loss function is differentiable and satisfies $\ell'(z)>0$ for all $z\in\mathbb{R}$, then $\hat{R}_{n}(\bm{\theta}^{*})>0$.
\end{claim}
Assume that the multilayer neural network $f(x;\bm{\theta})$ has $L\ge 1$ hidden layers, $M_{l}\ge 1$ neurons in the $l$-th layer. Now we let the vector $\bm{\theta}_{l}$ contain all parameters in the first $l\in[L]$ layers. 
Then the output of the neural network can be rewritten as 
\begin{equation*}
f(x;a_{0},\bm{\theta}_{L})=a_{0}+\sum_{j=1}^{M_{L}}a_{j}\sigma(\bm{w}_{j}^{\top}\bm{\Phi}(x;\bm{\theta}_{L-1})+b_{j}),
\end{equation*}
where $\bm{\Phi}(x;\bm{\theta}_{L-1})=(\Phi_{1}(x;\bm{\theta}_{L-1}),...,\Phi_{M_{L-1}}(x;\bm{\theta}_{L-1}))$ denotes the outputs of the neurons in the layer $L-1$.
Then the empirical loss is defined as 
$$\hat{L}_{n}(\bm{\theta})=\frac{1}{n}\sum_{i=1}^{n}\ell(-y_{i}f(x_{i};\bm{\theta}))$$
If the point $\bm{\theta}^{*}=(a_{0}^{*},\bm{\theta}_{L}^{*})$ denotes a critical point of the empirical loss function, then we should have, for $\forall j\in[M_{L}]$,
\begin{align}
\frac{\partial \hat{L}_{n}(\bm{\theta}^{*})}{\partial a_{0}}&=\frac{1}{n}\sum_{i=1}^{n}\ell'(-y_{i}f(x_{i};\bm{\theta}^{*}))(-y_{i})=0,\label{eq::logit-general-1}\\
\frac{\partial \hat{L}_{n}(\bm{\theta}^{*})}{\partial a_{j}}&=\frac{1}{n}\sum_{i=1}^{n}\ell'(-y_{i}f(x_{i};\bm{\theta}^{*}))(-y_{i})\sigma\left({\bm{w}_{j}^{*}}^{\top}\bm{\Phi}(x_{i};\bm{\theta}^{*}_{L-1})+b_{j}\right)=0.\label{eq::logit-general-2}
\end{align}
In addition, by adding Equations~\eqref{eq::logit-general-1} and \eqref{eq::logit-general-2}, we have 
\begin{align}
0=a^{*}_{0}\frac{\partial \hat{L}_{n}(\bm{\theta}^{*})}{\partial a_{0}}+\sum_{j=1}^{M_{L}}a_{j}^{*}\frac{\partial \hat{L}_{n}(\bm{\theta}^{*})}{\partial a_{j}}&=\frac{1}{n}\sum_{i=1}^{n}\ell'(-y_{i}f(x_{i};\bm{\theta}^{*}))(-y_{i})\left[a_{0}^{*}+\sum_{j=1}^{M_{L}}a_{j}^{*}\sigma\left({\bm{w}_{j}^{*}}^{\top}\bm{\Phi}(x_{i};\bm{\theta}^{*}_{L-1})+b_{j}\right)\right]\notag\\
&=\frac{1}{n}\sum_{i=1}^{n}\ell'(-y_{i}f(x_{i};\bm{\theta}^{*}))(-y_{i})f(x_{i};\bm{\theta}^{*}).\label{eq::logit-general-3}
\end{align}
This indicates that if $\bm{\theta}^{*}$ is a critical point of the empirical loss, then the following equation should hold,
\begin{equation}\frac{1}{n}\sum_{i=1}^{n}\ell'(-y_{i}f(x_{i};\bm{\theta}^{*}))y_{i}f(x_{i};\bm{\theta}^{*})=0.\end{equation}
However, if the dataset contains both positive and the negative samples, $\ell'(z)>0$ for all $z\in\mathbb{R}$, then this means that if $\hat{R}_{n}(\bm{\theta}^{*})=0$, then
\begin{equation}
\frac{1}{n}\sum_{i=1}^{n}\ell'(-y_{i}f(x_{i};\bm{\theta}^{*}))y_{i}f(x_{i};\bm{\theta}^{*})>0.
\end{equation}
We note here that the assumption that the dataset contains both positive and the negative samples is to ensure that when $\hat{R}_{n}(\bm{\theta}^{*})=0$, there is at least one sample in the dataset satisfying
$$y_{i}f(x_{i};\bm{\theta}^{*})>0.$$
Therefore, we have the contradiction. This indicates that $\hat{R}_{n}(\bm{\theta}^{*})>0$.

\end{proof}

\clearpage
\subsection{Proof of Proposition~\ref{prop::convex-necc}}\label{appendix::convex-necc}
\begin{proposition}
Assume that assumptions~\ref{assump::loss},~\ref{assump::shortcut-connection} and~\ref{assump::neurons} are satisfied. For any feedforward architecture $f_{D}(x;\bm{\theta}_{D})$, every local minimum $\bm{\theta}^{*}=(\bm{\theta}_{S}^{*},\bm{\theta}_{D}^{*})$ of the empirical loss function $\hat{L}_{n}(\bm{\theta}_{S},\bm{\theta}_{D};p)$, $p\ge 6$ satisfies $\hat{R}_{n}(\bm{\theta}^{*})=0$ \textbf{only if}  the matrix $\sum_{i=1}^{n}\lambda_{i}y_{i}x_{i}x_{i}^{\top}$ is neither positive nor negative definite for all sequences $\{\lambda_{i}\ge0\}_{i=1}^{n}$ satisfying $\sum_{i:y_{i}=1}\lambda_{i}=\sum_{i:y_{i}=-1}\lambda_{i}>0$ and $\|\sum_{i=1}^{n}\lambda_{i}y_{i}x_{i}\|_{2}=0$.
\end{proposition}
\begin{proof}
We prove Proposition~\ref{prop::convex-necc} by proving the following claim.
\begin{claim}
If  there exists a sequence $\{\lambda_{i}\ge 0\}_{i=1}^{n}$ satisfying $\sum_{i:y_{i}=1}\lambda_{i}=\sum_{i:y_{i}=-1}\lambda_{i}>0$ and $\|\sum_{i=1}^{n}\lambda_{i}y_{i}x_{i}\|_{2}=0$ such that the matrix 
$\sum_{i=1}^{n}\lambda_{i}y_{i}x_{i}x_{i}^{\top}$ is positive or negative positive definite, 
then there exists a feedforward neural architecture $f_{D}$ such that the empirical loss function $\hat{L}_{n}(\bm{\theta}_{S},\bm{\theta}_{D};p),p\ge 6$ has a local minimum with a non-zero training error.  
\end{claim}
\begin{proof}
Let $\mathcal{D}=\{(x_{i},y_{i})\}_{i=1}^{n}$ denote a dataset consisting of $n$ samples. We rewrite the sample $x$ as  $x=\left(x^{(1)},...,x^{(d)}\right)$.
Consider the following network,
$$f(x;\bm{\theta})=f_{S}(x;\bm{\theta}_{S})+f_{D}(x;\bm{\theta}_{D}),$$
where $$f_{S}(x;\bm{\theta}_{S})=a_{0}+\sum_{j=1}^{M}a_{j}\sigma(\bm{w}_{j}^{\top}x_{i}+b_{j}),$$
and the multilayer network is defined as follows,
\begin{equation}
f_{D}(x;\bm{\theta}_{D})=f_{D}(x;\theta_{1},...,\theta_{d})=\sum_{i=1}^{n}\mu_{i}\prod_{k=1}^{d}\bm{1}\left\{x^{(k)}\in\left[x_{i}^{(k)}-\theta_{k},x_{i}^{(k)}+\theta_{k}\right]\right\}.
\end{equation}
We note here that $\mu_{1},...,\mu_{n}$ are not parameters and later we will show that this function can be implemented by a multilayer network consisted of threshold units. A useful property of the function $f_{D}(x;\bm{\theta}_{D})$ is that if all parameters $\theta_{i}$s are positive and sufficiently smalls, then for each sample $(x_{i},y_{i})$ in the dataset, 
$$f_{D}(x_{i};\bm{\theta}_{D})=\mu_{i}.$$
Furthermore, if we slightly perturb all parameters, the output of the function $f_{D}$ on all samples remain the same. In the proof, we use these two properties to construct the local minimum with a non-zero training error. 

By assumption, there exists a sequence $\{\lambda_{i}\ge 0\}_{i=1}^{n}$ satisfying $\sum_{i:y_{i}=1}\lambda_{i}=\sum_{i:y_{i}=-1}\lambda_{i}>0$ and $\|\sum_{i=1}^{n}\lambda_{i}y_{i}x_{i}\|_{2}=0$ such that the matrix 
$\sum_{i=1}^{n}\lambda_{i}y_{i}x_{i}x_{i}^{\top}$ is positive or negative positive definite.
Without loss of generality, we assume that the matrix is positive definite.
Now we construct a local minimum $\bm{\theta}^{*}$. Let $a_{0}^{*}=a_{1}^{*}=...=a_{M}^{*}=-1$, $\bm{w}^{*}_{1}=...=\bm{w}^{*}_{M}=\bm{0}_{d}$ and $b_{1}^{*}=...=b_{M}^{*}=0$. Now we set $\theta^{*}_{1},...,\theta^{*}_{d}$ to be positive and sufficiently small such that for two different samples in the dataset, e.g., $x_{i}\neq x_{j}$, the following equations holds,
$$\prod_{k=1}^{d}\bm{1}\left\{x_{j}^{(k)}\in\left[x_{i}^{(k)}-2\theta^{*}_{k},x_{i}^{(k)}+2\theta^{*}_{k}\right]\right\}=0,\quad\prod_{k=1}^{d}\bm{1}\left\{x_{i}^{(k)}\in\left[x_{j}^{(k)}-2\theta^{*}_{k},x_{j}^{(k)}+2\theta^{*}_{k}\right]\right\}=0.$$
%Furthermore, since the loss function satisfies $\ell'_{p}(z)\ge 0$ and $\ell'_{p}(z)=0$ iff $z\le -z_{0}$, then for  $\forall z>-z_{0}$, $\ell_{p}'(z)>0$. This indicates that the loss function $\ell$ is monotonically increasing on the interval $(-z_{0},\infty)$. Furthermore, since $\ell_{p}(z)\ge \mathbb{I}\{z> 0\}$, then $\ell_{p}(z)>1$ for all $z>0$.
Now we choose $\mu_{1},...,\mu_{n}$ as follows. The output of the neural network on sample $x_{i}$ in the dataset is $f(x_{i};\bm{\theta}^{*})=\mu_{i}-M\sigma(0)$.

We need to choose $\mu_{1},...,\mu_{n}$ to satisfy all conditions shown as follows:
\begin{itemize}
\item[(1)] There exists $i\in[n]$ such that $y_{i}(\mu_{i}-M\sigma(0))<0$.
\item[(2)] For all $i:y_{i}=1$ and all $k:y_{k}=-1$, $$\frac{\ell'(-y_{i}(\mu_{i}-M\sigma(0)))}{\sum_{j:j=1}\ell'(-y_{i}(\mu_{i}-M\sigma(0)))}=\frac{\lambda_{i}}{\sum_{j:j=1}\lambda_{j}},\quad \frac{\ell'(-y_{k}(\mu_{k}-M\sigma(0)))}{\sum_{j:j=-1}\ell'(-y_{i}(\mu_{i}-M\sigma(0)))}=\frac{\lambda_{k}}{\sum_{j:j=-1}\lambda_{j}},$$
and 
$$\sum_{j:j=1}\ell'(-y_{i}(\mu_{i}-M\sigma(0)))=\sum_{j:j=-1}\ell'(-y_{i}(\mu_{i}-M\sigma(0))).$$
\end{itemize}  
Now we start from the largest element in the sequence $\{\lambda_{i}\}_{i=1}^{n}$. Since $\sum_{i=1}^{n}\lambda_{i}>0$, the define the index $i_{\max}$ as the index of the largest element, i.e., $$i_{\max}=\arg\max_{i}\lambda_{i}.$$
Let $\lambda_{\max}=\lambda_{i_{\max}}$.
Now we choose $\mu_{i_{\max}}$ such that
$$y_{i_{\max}}(\mu_{i_{\max}}-M\sigma(0))=-1.$$
Thus, the index $i_{\max}$ satisfy the first condition. Then for $i\neq i_{\max}$, we choose $\mu_{i}$ such that 
\begin{equation}\label{eq::prop-necc-1}\ell'(-y_{i}(\mu_{i}-M\sigma(0)))=\frac{\lambda_{i}}{\lambda_{\max}}\ell(-y_{i_{\max}}(\mu_{i\max}-M\sigma(0)))=\frac{\lambda_{i}}{\lambda_{\max}}\ell'(1)\le\ell'(1).\end{equation}
We note here that for each $i\in[n]$, there always exists a $\mu_{i}$ solving the above equation. This can be seen by the fact that $\ell'$ is continuous, $\ell'_{p}(z)\ge 0$ and $\ell'_{p}(z)=0$ iff $z\le -z_{0}$. This indicates that for  $\forall z>-z_{0}$, $\ell_{p}'(z)>0$, i.e., $\ell'(1)>0$ and that $\ell'(-z_{0})=0$. Since $\ell'(z)$ is continuous, then for $\forall r\in[0,\ell'(1)]$, there always exists $z\in\mathbb{R}$ such that $\ell'(z)=r$, which further indicates that for $\forall i\in[n]$, there always exists $\mu_{i}\in\mathbb{R}$ solving the Equation~\eqref{eq::prop-necc-1}. Under this construction, it is easy to show that the second condition is satisfied as well. 

Now we only need to show that $\bm{\theta}^{*}$ is local minimum. We first show that $\bm{\theta}^{*}$ is a critical point of the empirical loss function. 
Since for $\forall j\in[M]$,
\begin{align*}
\frac{\partial \hat{L}_{n}(\bm{\theta}^{*})}{\partial a_{j}}&=\sum_{i=1}^{n}\ell'(-y_{i}(\mu_{i}-M\sigma(0)))(-y_{i})\sigma(0)\\
&=\sigma(0)\sum_{i=1}^{n}\frac{\lambda_{i}}{\lambda_{\max}}\ell'(1)(-y_{i})=-\frac{\sigma(0)\ell'(1)}{\lambda_{\max}}\sum_{i=1}^{n}y_{i}\lambda_{i}\\
&=0&&\text{by }\sum_{i:y_{i}=1}\lambda_{i}=\sum_{i:y_{i}=-1}\lambda_{i}\\
\nabla_{\bm{w}_{j}}\hat{L}_{n}(\bm{\theta}^{*})&=\sum_{i=1}^{n}\ell'(-y_{i}(\mu_{i}-M\sigma(0)))(-y_{i})\sigma'(0)x_{i}\\
&=-\sigma'(0)\sum_{i=1}^{n}\frac{\lambda_{i}}{\lambda_{\max}}\ell'(1)y_{i}x_{i}=-\frac{\sigma'(0)\ell'(1)}{\lambda_{\max}}\sum_{i=1}^{n}\lambda_{i}y_{i}x_{i}\\
&=\bm{0}_{d}&&\text{by }\left\|\sum_{i=1}^{n}\lambda_{i}y_{i}x_{i}\right \|_{2}=0
\end{align*}
and 
$$\frac{\partial \hat{L}_{n}(\bm{\theta}^{*})}{\partial a_{0}}=\sum_{i=1}^{n}\ell'(-y_{i}(\mu_{i}-M\sigma(0)))(-y_{i})=-\frac{\ell'(1)}{\lambda_{\max}}\sum_{i=1}^{n}y_{i}\lambda_{i}=0.$$
In addition, we have stated earlier, if we slightly perturb the parameter $\theta_{k}^{*}$ in the interval $[\theta_{k}^{*}/2,3\theta_{k}^{*}/2]$, the output of the function $f_{D}(x_{i};\bm{\theta}_{D})$ does not change for all $i\in[n]$, then $\bm{\theta}^{*}$ is a critical point. 

Now we show that $\bm{\theta}^{*}$ is local minimum. Consider any perturbation $\Delta a_{1},...,\Delta a_{M}:|\Delta a_{j}|<\frac{1}{2}$ for all $j\in[M]$, $\Delta \bm{w}_{1},...,\Delta \bm{w}_{M}\in\mathbb{R}^{d}$, $\Delta a_{0}\in\mathbb{R}$, $\Delta\theta_{k}:|\Delta \theta_{k}|\le\theta_{k}/2$ for all $k\in[n]$. Define $$\tilde{\bm{\theta}}=(a_{0}^{*}+\Delta a_{0},...,a_{M}^{*}+\Delta a_{M},\bm{w}_{1}^{*}+\Delta \bm{w}_{1},...,\bm{w}_{M}^{*}+\Delta \bm{w}_{M},\theta_{1}^{*}+\Delta\theta_{1}^{*},..., \theta_{d}^{*}+\Delta\theta_{d}^{*}).$$

Then 
\begin{align*}
\sum_{i=1}^{n}\ell(-y_{i}f(x_{i};\tilde{\bm{\theta}}))-\sum_{i=1}^{n}\ell(-y_{i}f(x_{i};\bm{\theta}^{*}))&=\sum_{i=1}^{n}\left[\ell(-y_{i}f(x_{i};\tilde{\bm{\theta}}))-\ell(-y_{i}f(x_{i};\bm{\theta}^{*}))\right]\\
&\ge\sum_{i=1}^{n}\ell'(-y_{i}f(x_{i};\bm{\theta}^{*}))(-y_{i})[f(x_{i};\tilde{\bm{\theta}})-f(x_{i};{\bm{\theta}}^{*})].
\end{align*}
Since for each sample $x_{i}$ in the dataset, 
\begin{align*}
f(x_{i};\tilde{\bm{\theta}})-f(x_{i};{\bm{\theta}}^{*})&=\Delta a_{0}+\sum_{j=1}^{M}(a^{*}_{j}+\Delta a_{j})\sigma(\Delta \bm{w}_{j}^{\top}x_{i})+\mu_{i}-\mu_{i}\\
&=\Delta a_{0}+\sum_{j=1}^{M}(a^{*}_{j}+\Delta a_{j})\sigma(\Delta \bm{w}_{j}^{\top}x_{i}),
\end{align*}
then
 
\begin{align*}
\sum_{i=1}^{n}\ell(-y_{i}f(x_{i};\tilde{\bm{\theta}}))&-\sum_{i=1}^{n}\ell(-y_{i}f(x_{i};\bm{\theta}^{*}))\\
&\ge\sum_{i=1}^{n}\ell'(-y_{i}f(x_{i};\bm{\theta}^{*}))(-y_{i})[f(x_{i};\tilde{\bm{\theta}})-f(x_{i};{\bm{\theta}}^{*})]\\
&=\sum_{i=1}^{n}\ell'(-y_{i}(\mu_{i}-M\sigma(0)))(-y_{i})\left[\sum_{j=1}^{M}(a_{j}^{*}+\Delta a_{j})\sigma\left(\Delta\bm{w}_{j}^{\top}x_{i}\right)+\Delta a_{0}\right]\\
&=\sum_{i=1}^{n}\frac{\lambda_{i}\ell'(1)}{\lambda_{\max}}(-y_{i})\left[\sum_{j=1}^{M}(a_{j}^{*}+\Delta a_{j})\sigma\left(\Delta\bm{w}_{j}^{\top}x_{i}\right)\right]\\
&=\frac{\ell'(1)}{\lambda_{\max}}\sum_{j=1}^{M}-(a_{j}^{*}+\Delta a_{j})\left[\sum_{i=1}^{n}\lambda_{i}y_{i}\sigma\left(\Delta\bm{w}_{j}^{\top}x_{i}\right)\right].
\end{align*}

Now we define the following function $G:\mathbb{R}^{d}\rightarrow \mathbb{R}$, 
\begin{equation*}
G(\bm{u})=\sum_{i=1}^{n}\lambda_{i}y_{i}\sigma\left(\bm{u}^{\top}x_{i}\right).
\end{equation*}
Now we consider the gradient of the function $G$ with respect to the vector $\bm{u}$ at the point $\bm{0}_{d}$,
\begin{align*}
\nabla_{\bm{u}}G(\bm{0}_{d})&=\sum_{i=1}^{n}\lambda_{i}y_{i}\sigma'\left(0\right)x_{i}=\bm{0}_{d}&&\text{by }\left\|\sum_{i=1}^{n}\lambda_{i}y_{i}x_{i}\right\|_{2}=0.
\end{align*}
Furthermore,  the Hessian matrix $\nabla_{\bm{u}}^{2}G(\bm{0}_{d})$ satisfies
\begin{align*}
\nabla_{\bm{u}}^{2}G(\bm{0}_{d})&=
\sum_{i=1}^{n}\lambda_{i}y_{i}\sigma''\left(0\right)x_{i}x_{i}^{\top}=\sigma''\left(0\right)\sum_{i=1}^{n}\lambda_{i}y_{i}x_{i}x_{i}^{\top}\succ 0,
\end{align*}
then the function $G(\bm{u})=\sum_{i=1}^{n}\lambda_{i}y_{i}\sigma\left(\bm{u}^{\top}x_{i}\right)$ has a local minima at $\bm{u}=\bm{0}_{d}$. This indicates that there exists $\varepsilon>0$ such that for all $(\Delta \bm{w}_{1},...,\Delta \bm{w}_{M}):\sqrt{\sum_{j=1}^{M}\|\Delta \bm{w}_{j}\|^{2}_{2}}\le\varepsilon$, 
$$\sum_{i=1}^{n}\lambda_{i}y_{i}\sigma\left(\Delta\bm{ w}_{j}^{\top}x_{i}\right)\ge \sum_{i=1}^{n}\lambda_{i}y_{i}\sigma\left(0\right)=0,$$
where the equality holds by the fact that $\sum_{i=1}^{n}y_{i}\lambda_{i}=1$.
In addition, since  $a_{j}^{*}=-1$, $|\Delta a_{j}|<\frac{1}{2}$, then for all $\Delta\bm{w}_{j}:\|\Delta\bm{w}_{j}\|_{2}\le\varepsilon$ and $\Delta b_{j}\in\mathbb{R}$,
\begin{align*}
\sum_{i=1}^{n}\ell(-y_{i}f(x_{i};\tilde{\bm{\theta}}))&-\sum_{i=1}^{n}\ell(-y_{i}f(x_{i};\bm{\theta}^{*}))\ge 0.
\end{align*}
Thus, $\bm{\theta}^{*}$ is a local minima of the empirical loss function with $f(x_{i};\bm{\theta}^{*})=\mu_{i}-M\sigma(0)$. Since there exists a $\mu_{i_{\max}}$ such that $y_{i_{\max}}(\mu_{i_{\max}}-M\sigma(0))=1$, then this means that the neural network makes an incorrect prediction on the sample $x_{i_{\max}}$. This indicates that this local minimum has a non-zero training error.

Finally, we present the way we construct the neural network $f_{D}$. Since
\begin{equation*}
f_{D}(x;\bm{\theta}_{D})=f_{D}(x;\theta_{1},...,\theta_{d})=\sum_{i=1}^{n}\mu_{i}\prod_{k=1}^{d}\bm{1}\left\{x^{(k)}\in\left[x_{i}^{(k)}-\theta_{k},x_{i}^{(k)}+\theta_{k}\right]\right\}.
\end{equation*}
Let $\thres$ denote the threshold unit, where $\thres(z)=1$ if $z\ge 0$ and $\thres(z)=0$, otherwise. Therefore, the indicator function can be represented as follows:
$$\bm{1}\left\{x^{(k)}\in\left[x_{i}^{(k)}-\theta_{k},x_{i}^{(k)}+\theta_{k}\right]\right\}=\thres\left(x^{(k)}-x_{i}^{(k)}+\theta_{k}\right)-\thres\left(x^{(k)}-x_{i}^{(k)}-\theta_{k}\right)$$
Therefore, 
\begin{align*}
\prod_{k=1}^{d}&\bm{1}\left\{x^{(k)}\in\left[x_{i}^{(k)}-\theta_{k},x_{i}^{(k)}+\theta_{k}\right]\right\}\\
&=\thres\left(\sum_{k=1}^{d}\left[\thres\left(x^{(k)}-x_{i}^{(k)}+\theta_{k}\right)-\thres\left(x^{(k)}-x_{i}^{(k)}-\theta_{k}\right)\right]-d+\frac{1}{2}\right)\\
\end{align*}
Therefore, we have 
$$f_{D}(x;\bm{\theta}_{D})=\sum_{i=1}^{n}\mu_{i}\thres\left(\sum_{k=1}^{d}\left[\thres\left(x^{(k)}-x_{i}^{(k)}+\theta_{k}\right)-\thres\left(x^{(k)}-x_{i}^{(k)}-\theta_{k}\right)\right]-d+\frac{1}{2}\right).$$
It is very easy to see that this is a two layer network consisted of threshold units. 

Furthermore, we note here that, in the proof shown above, we assume the only parameters in the network $f_{D}$ are $\bm{\theta}_{1},...,\bm{\theta}_{d}$. In fact, we can prove a more general statement where the $f_{D}$ is of the form
$$f_{D}(x;\bm{\theta}_{D})=\sum_{i=1}^{n}\mu_{i}\thres\left(\sum_{k=1}^{d}\left[a_{ik}\thres\left(x^{(k)}+u_{ik}\right)+b_{ik}\thres\left(x^{(k)}+v_{ik}\right)\right]+c_{i}\right),$$
where $a_{ik},b_{ik},u_{ik},v_{ik}, c_{i}$, $i\in[n], k\in[d]$ are all parameters. We can show that the neural network 
$$f_{D}(x;\bm{\theta}_{D})=\sum_{i=1}^{n}\mu_{i}\thres\left(\sum_{k=1}^{d}\left[\thres\left(x^{(k)}-x_{i}^{(k)}+\theta_{k}\right)-\thres\left(x^{(k)}-x_{i}^{(k)}-\theta_{k}\right)\right]-d+\frac{1}{2}\right),$$
denotes a local minimum, since any slight perturbations on parameters $a_{ik},b_{ik},u_{ik},v_{ik}, c_{i}$, $i\in[n], k\in[d]$ do not change the output of the neural network on the samples in the dataset $\mathcal{D}$.

\end{proof}

\end{proof}

\clearpage

\subsection{Proof of Example~\ref{exam::assump-full-rank}}\label{appendix::exam-dataset}

In this subsection, we present two examples to show that if either assumption~\ref{assump::full-rank} or \ref{assump::different-subspaces} is not satisfied, even if the other conditions in Theorem~\ref{thm::convex-finite-deep} are satisfied, Theorem~\ref{thm::convex-finite-deep} does not hold. 

\begin{example}\label{exam::assump-full-rank}
Assume that the distribution $\mathbb{P}_{X\times Y}$ satisfies that $\mathbb{P}_{Y}(Y=1)=\mathbb{P}_{Y}(Y=-1)$, $\mathbb{P}_{X|Y}(X=(1,0)|Y=1)=\mathbb{P}_{X|Y}(X=(-1, 0)|Y=1)=0.5$ and $\mathbb{P}_{X|Y}(X=(0, 0)|Y=-1)$. 
Assume that samples in the dataset $\mathcal{D}=\{(x_{i},y_{i})\}_{i=1}^{2n}$ are independently drawn from the distribution $\mathbb{P}_{\bm{X}\times Y}$.
Assume that the network $f_{S}$ has $M\ge 1$ neurons and neurons in $f_{S}$ satisfy the condition that $\sigma$ is analytic and has a positive second order  derivative on $\mathbb{R}$. There exists a feedforward network $f_{D}$ such that the empirical loss $\hat{L}_{n}(\bm{\theta}_{S},\bm{\theta}_{D})$ has a local minimum with non-zero training error with a probability at least $\Omega(1/n^{2})$. 
\end{example}
\textbf{Remark:} This is a counterexample where Theorem~\ref{thm::convex-finite-deep} does not hold, when Assumption~\ref{assump::different-subspaces} is satisfied and Assumption~\ref{assump::full-rank} is not satisfied. 
 This distribution can be viewed in the following way. The positive data samples are located on the linear span of the set $\{(1,0)\}$, the negative data samples locate on the linear span of the set $\{(0,1)\}$ and all samples are located on the linear span of the set $\{(1,0), (0,1)\}$. Therefore, $r=2>\max\{r_{+},r_{-}\}=1$. This means that Assumption~\ref{assump::different-subspaces} is satisfied. In addition, it is easy to check that Assumption~\ref{assump::full-rank} is not satisfied, since the matrix $\left(0, 0\right)$ has rank zero and thus does not have a full rank. This means that our main results may not hold when the assumption~\ref{assump::full-rank} is not satisfied.

\begin{proof}
Let $n_{1},n_{0},n_{-1}$ denote the number of samples at the point $(1,0),(0,0), (-1,0)$, respectively. It is easy to see that the event that $n_{1}=n_{-1}>0$ and $n_{0}>0$ happens with probability at least $\Omega(1/n^{2})$. We note that this is not a tight bounded, however, we just need to show that this happens with a positive probability. Now we consider the optimization problem under the dataset where $n_{1}=n_{-1}>0$ and $n_{0}>0$.

We first set the feedforward network $f_{D}(x;\bm{\theta}_{D})$ to constant, i.e., $f_{D}(x;\bm{\theta}_{D})\equiv 0$ for $x\in\mathbb{R}^{2}$.
Now the whole network becomes a single layer network,
$$f(x;\bm{\theta})=a_{0}+\sum_{j=1}^{M}a_{j}\sigma\left(\bm{w}^{\top}_{j}x\right).$$
Let $a_{1}^{*}=...=a_{M}^{*}=-1$ and $\bm{w}^{*}_{1}=...=\bm{w}^{*}_{M}=\bm{0}_{2}$. 

Therefore, we have $f(x;\bm{\theta}^{*})=a_{0}^{*}-M\sigma(0)$. Let $a_{0}^{*}$ be the global optimizer of the following convex optimization problem. 
$$\min_{a}\sum_{i=1}^{2n}\ell_{p}(-y_{i}(a-M\sigma(0))).$$
Thus, we have 
\begin{equation}\label{eq::prop-counter-2}\sum_{i=1}^{n}\ell_{p}'(-y_{i}(a_{0}^{*}-M\sigma(0)))(-y_{i})=0,\end{equation}
and this indicates that 
\begin{equation}\label{eq::prop-counterexample-eq2}\sum_{i:y_{i}=1}\ell_{p}'(-(a_{0}^{*}-M\sigma(0)))=\sum_{i:y_{i}=-1}\ell_{p}'(a_{0}^{*}-M\sigma(0))\quad\text{or}\quad{\ell_{p}'(-a_{0}^{*}+M\sigma(0))}{n_{+}}={\ell_{p}'(a_{0}^{*}-M\sigma(0))}{n_{-}}.\end{equation}
In addition, since for $\forall j\in[M]$,
\begin{align*}
&\frac{\partial \hat{L}_{n}(\bm{\theta}^{*})}{\partial a_{j}}=\sum_{i=1}^{2n}\ell_{p}'(-y_{i}(a_{0}^{*}-M\sigma(0)))(-y_{i})\sigma(0)=0,&&\text{by Equation }~\eqref{eq::prop-counter-2},\\
& \nabla_{\bm{w}_{j}}\hat{L}_{n}(\bm{\theta}^{*})=\sum_{i=1}^{2n}\ell_{p}'(-y_{i}(a_{0}^{*}-M\sigma(0)))(-y_{i})\sigma'(0)x_{i}=\bm{0}_{2},&&\text{by }\sum_{i:y_{i}=1}x_{i}=\sum_{i:y_{i}=-1}x_{i}=\bm{0}_{2},%\\
%&\frac{\partial \hat{L}_{n}(\bm{\theta}^{*})}{\partial b_{j}}=\sum_{i=1}^{n}\ell_{p}'(-y_{i}(a_{0}^{*}-M\sigma(0)))(-y_{i})\sigma'(0)=0,&&\text{by }\sigma'(0)=0,
\end{align*}
and 
$$\frac{\partial \hat{L}_{n}(\bm{\theta}^{*})}{\partial a_{0}}=\sum_{i=1}^{n}\ell_{p}'(-y_{i}(a_{0}^{*}-M\sigma(0)))(-y_{i})=0,$$
then $\bm{\theta}^{*}$ is a critical point. 

Next we show that $\bm{\theta}^{*}=(a_{0}^{*},...,a_{M}^{*},\bm{w}_{1}^{*},...,\bm{w}_{M}^{*})$ is a local minima. Consider any perturbation $\Delta a_{1},...,\Delta a_{M}:|\Delta a_{j}|<\frac{1}{2}$ for all $j\in[M]$, $\Delta \bm{w}_{1},...,\Delta \bm{w}_{M}\in\mathbb{R}^{2}$ and $\Delta a_{0}\in\mathbb{R}$. Define $$\tilde{\bm{\theta}}=(a_{0}^{*}+\Delta a_{0},...,a_{M}^{*}+\Delta a_{M},\bm{w}_{1}^{*}+\Delta \bm{w}_{1},...,\bm{w}_{M}^{*}+\Delta \bm{w}_{M}
).$$

Then 
\begin{align*}
\sum_{i=1}^{n}\ell_{p}(-y_{i}f(x_{i};\tilde{\bm{\theta}}))-\sum_{i=1}^{n}\ell_{p}(-y_{i}f(x_{i};\bm{\theta}^{*}))&=\sum_{i=1}^{n}\left[\ell_{p}(-y_{i}f(x_{i};\tilde{\bm{\theta}}))-\ell_{p}(-y_{i}f(x_{i};\bm{\theta}^{*}))\right]\\
&\ge\sum_{i=1}^{n}\ell_{p}'(-y_{i}f(x_{i};\bm{\theta}^{*}))(-y_{i})[f(x_{i};\tilde{\bm{\theta}})-f(x_{i};{\bm{\theta}}^{*})]\\
&=\sum_{i=1}^{n}\ell_{p}'(-y_{i}(a_{0}^{*}-M\sigma(0)))(-y_{i})[f(x_{i};\tilde{\bm{\theta}})-a_{0}^{*}+M\sigma(0)]\\
&=\sum_{i=1}^{n}\ell_{p}'(-y_{i}(a_{0}^{*}-M\sigma(0)))(-y_{i})f(x_{i};\tilde{\bm{\theta}}),
\end{align*}
where the inequality follows from the convexity of the loss function $\ell_{p}(z)$, the second equality follows from the fact that $f(x;\bm{\theta}^{*})\equiv a_{0}^{*}-M\sigma(0)$ and the third equality follows from Equation~\eqref{eq::prop-counterexample-eq2}. In addition, we have 
\begin{align*}
&\sum_{i=1}^{n}\ell_{p}'(-y_{i}(a_{0}^{*}-M\sigma(0)))(-y_{i})f(x_{i};\tilde{\bm{\theta}})\\
&=\sum_{i=1}^{n}\ell_{p}'(-y_{i}(a_{0}^{*}-M\sigma(0)))(-y_{i})\left[\sum_{j=1}^{M}(a_{j}^{*}+\Delta a_{j})\sigma\left(\Delta\bm{w}_{j}^{\top}x_{i}\right)+\Delta a_{0}\right]\\
&=\sum_{i=1}^{n}\ell_{p}'(-y_{i}(a_{0}^{*}-M\sigma(0)))(-y_{i})\left[\sum_{j=1}^{M}(a_{j}^{*}+\Delta a_{j})\sigma\left(\Delta\bm{w}_{j}^{\top}x_{i}\right)\right]&&\text{by Eq.~\eqref{eq::prop-counterexample-eq2}}\\
&=\sum_{j=1}^{M}-(a_{j}^{*}+\Delta a_{j})\left[\sum_{i=1}^{n}\ell_{p}'(-y_{i}(a_{0}^{*}-M\sigma(0)))y_{i}\sigma\left(\Delta\bm{w}_{j}^{\top}x_{i}\right)\right]\\
&=\sum_{j=1}^{M}-(a_{j}^{*}+\Delta a_{j})\left[\sum_{i=1}^{n}\ell_{p}'(-y_{i}(a_{0}^{*}-M\sigma(0)))y_{i}\sigma\left(\Delta\bm{w}^{(1)}_{j}x_{i}^{(1)}\right)\right]&&\text{by } x_{i}^{(2)}=0,\forall i\in[n].
\end{align*}
Now we define the following function $G:\mathbb{R}\rightarrow \mathbb{R}$, 
\begin{equation*}
G(u)=\sum_{i=1}^{n}\ell_{p}'(-y_{i}(a_{0}^{*}-M\sigma(0)))y_{i}\sigma\left(ux_{i}^{(1)}\right).
\end{equation*}
Now we consider the gradient of the function $G$ with respect to the variable $u$ at the point $u={0}$,
\begin{align*}
\nabla_{{u}}G({0})&=\sum_{i=1}^{n}\ell_{p}'(-y_{i}(a_{0}^{*}-M\sigma(0)))y_{i}\sigma'\left(0\right)x_{i}^{(1)}=0.
\end{align*}
Furthermore,  the second order derivative $\nabla_{u}^{2}G({0})$ satisfies
\begin{align*}
\nabla_{u}^{2}G({0})&=
\sum_{i=1}^{n}\ell_{p}'(-y_{i}(a_{0}^{*}-M\sigma(0)))y_{i}\sigma''\left(0\right)\left(x_{i}^{(1)}\right)^{2}=\sigma''\left(0\right)\sum_{i=1}^{n}\ell_{p}'(-y_{i}(a_{0}^{*}-M\sigma(0)))y_{i}\left(x_{i}^{(1)}\right)^{2}\\
&=\sigma''(0)\left[\frac{1}{n_{+}}\sum_{i:y_{i}=1}\left(x_{i}^{(1)}\right)^{2}-\frac{1}{n_{-}}\sum_{i:y_{i}=-1}\left(x_{i}^{(1)}\right)^{2}\right]> 0,
\end{align*}
then the function $G({u})=\sum_{i=1}^{n}\ell_{p}(-y_{i}(a_{0}^{*}-M\sigma(0)))y_{i}\sigma\left(ux_{i}^{(1)}\right)$ has a local minima at ${u}=0$. This indicates that there exists $\varepsilon>0$ such that for all $\Delta\bm{w}:\|\Delta\bm{w}\|_{2}\le\varepsilon$, 
$$\sum_{i=1}^{n}\ell_{p}'(-y_{i}(a_{0}^{*}-M\sigma(0)))y_{i}\sigma\left(\Delta\bm{w}^{\top}x_{i}\right)\ge \sum_{i=1}^{n}\ell_{p}(-y_{i}(a_{0}^{*}-M\sigma(0)))y_{i}\sigma\left(0\right)=0.$$
In addition, since  $a_{j}^{*}=-1$, $|\Delta a_{j}|<\frac{1}{2}$, then for all $\Delta\bm{w}_{j}:\|\Delta\bm{w}_{j}\|_{2}\le\varepsilon$,
\begin{align*}
\sum_{i=1}^{n}\ell_{p}'(-y_{i}(a_{0}^{*}-M\sigma(0)))(-y_{i})f(x_{i};\tilde{\bm{\theta}})=\sum_{j=1}^{M}-(a_{j}^{*}+\Delta a_{j})\left[\sum_{i=1}^{n}\ell_{p}(-y_{i}(a_{0}^{*}-M\sigma(0)))y_{i}\sigma\left(\Delta\bm{w}_{j}^{\top}x_{i}\right)\right]\ge 0.
\end{align*}
Therefore, we have 
$$\sum_{i=1}^{n}\ell_{p}'(-y_{i}(a_{0}^{*}-M\sigma(0)))(-y_{i})f(x_{i};\tilde{\bm{\theta}})\ge 0,$$
and this indicates that 
$$\sum_{i=1}^{n}\ell_{p}(-y_{i}f(x_{i};\tilde{\bm{\theta}}))-\sum_{i=1}^{n}\ell_{p}(-y_{i}f(x_{i};\bm{\theta}^{*}))\ge 0.$$
Thus, $\bm{\theta}^{*}$ is a local minima with $f(x;\bm{\theta}^{*})=a_{0}^{*}-M\sigma(0)=$ constant. Thus, 
$$\frac{1}{n}\sum_{i=1}^{n}\mathbb{I}\{y_{i}\neq \sgn(f(x_{i};\bm{\theta}^{*}))\}\ge \frac{\min\{n_{-},n_{+}\}}{n}.$$
Since the dataset is consisted of both positive and negative examples, then the training error is non-zero. 

\end{proof}
\clearpage
\begin{example}\label{exam::assump-diff-subspaces}
Assume that the distribution $\mathbb{P}_{X\times Y}$ satisfies that $\mathbb{P}_{Y}(Y=1)=\mathbb{P}_{Y}(Y=-1)$ and $\mathbb{P}_{X|Y}(X=2|Y=1)=\mathbb{P}_{X|Y}(X=-1|Y=1)=0.5$ and $\mathbb{P}_{X|Y}(X=0.5|Y=-1)=1$. 
Assume that samples in the dataset $\mathcal{D}=\{(x_{i},y_{i})\}_{i=1}^{2n}$ are independently drawn from the distribution $\mathbb{P}_{\bm{X}\times Y}$.
Assume that the network $f_{S}$ has $M\ge 1$ neurons and neurons in $f_{S}$ satisfy the condition that $\sigma$ is analytic and has a positive second order  derivative on $\mathbb{R}$. There exists a feedforward network $f_{D}$ such that the empirical loss $\hat{L}_{n}(\bm{\theta}_{S},\bm{\theta}_{D})$ has a local minimum with non-zero training error with probability at least $\Omega(1/n^{2})$. 
\end{example}
\textbf{Remark:} 
This is a counterexample where Theorem~\ref{thm::convex-finite-deep} does not hold, when Assumption~\ref{assump::full-rank} is satisfied and Assumption~\ref{assump::different-subspaces}   is not satisfied. 
This distribution can be viewed in the following way. The positive data samples locate on the linear span of the set $\{(1)\}$, the negative data samples locate on the linear span of the set $\{(1)\}$ and all samples locate on the linear span of the set $\{(1)\}$. It is easy to check that assumption~\ref{assump::full-rank} is satisfied. However, $r=1=\max\{r_{+},r_{-}\}=1$. This means the assumption~\ref{assump::different-subspaces} is not satisfied.

\begin{proof}
Let $n_{2},n_{-1},n_{0.5}$ denote the number of samples at the point $(2),(-1), (0.5)$, respectively. It is easy to see that the event that $n_{2}=n_{-1}>0$ and $n_{0.5}>0$ happens with probability at least $\Omega(1/n^{2})$. We note that this is not a tight bounded, however, we just need to show that this happens with a positive probability. Now we consider the optimization problem under the dataset where $n_{2}=n_{-1}>0$ and $n_{0.5}>0$.

We first set the feedforward network $f_{D}(x;\bm{\theta}_{D})$ to constant, i.e., $f_{D}(x;\bm{\theta}_{D})\equiv 0$ for $x\in\mathbb{R}$.
Now the whole network becomes a single layer network,
$$f(x;\bm{\theta})=a_{0}+\sum_{j=1}^{M}a_{j}\sigma\left({w}_{j}x\right).$$
Let $a_{1}^{*}=...=a_{M}^{*}=-1$ and ${w}^{*}_{1}=...={w}^{*}_{M}=0$. 

Therefore, we have $f(x;\bm{\theta}^{*})=a_{0}^{*}-M\sigma(0)$. Let $a_{0}^{*}$ be the global optimizer of the following convex optimization problem. 
$$\min_{a}\sum_{i=1}^{2n}\ell_{p}(-y_{i}(a-M\sigma(0))).$$
Thus, we have 
\begin{equation}\label{eq::exam-2}\sum_{i=1}^{n}\ell_{p}'(-y_{i}(a_{0}^{*}-M\sigma(0)))(-y_{i})=0,\end{equation}
and this indicates that 
\begin{equation}\label{eq::counterexample-eq2}\sum_{i:y_{i}=1}\ell_{p}'(-(a_{0}^{*}-M\sigma(0)))=\sum_{i:y_{i}=-1}\ell_{p}'(a_{0}^{*}-M\sigma(0))\quad\text{or}\quad{\ell_{p}'(-a_{0}^{*}+M\sigma(0))}{n_{+}}={\ell_{p}'(a_{0}^{*}-M\sigma(0))}{n_{-}}.\end{equation}
In addition, since for $\forall j\in[M]$,
\begin{align*}
&\frac{\partial \hat{L}_{n}(\bm{\theta}^{*})}{\partial a_{j}}=\sum_{i=1}^{2n}\ell_{p}'(-y_{i}(a_{0}^{*}-M\sigma(0)))(-y_{i})\sigma(0)=0,&&\text{by Equation }~\eqref{eq::exam-2},\\
& \nabla_{{w}_{j}}\hat{L}_{n}(\bm{\theta}^{*})=\sum_{i=1}^{2n}\ell_{p}'(-y_{i}(a_{0}^{*}-M\sigma(0)))(-y_{i})\sigma'(0)x_{i}=0,&&\text{by }\sum_{i:y_{i}=1}x_{i}=\sum_{i:y_{i}=-1}x_{i}=0,%\\
%&\frac{\partial \hat{L}_{n}(\bm{\theta}^{*})}{\partial b_{j}}=\sum_{i=1}^{n}\ell_{p}'(-y_{i}(a_{0}^{*}-M\sigma(0)))(-y_{i})\sigma'(0)=0,&&\text{by }\sigma'(0)=0,
\end{align*}
and 
$$\frac{\partial \hat{L}_{n}(\bm{\theta}^{*})}{\partial a_{0}}=\sum_{i=1}^{n}\ell_{p}'(-y_{i}(a_{0}^{*}-M\sigma(0)))(-y_{i})=0,$$
then $\bm{\theta}^{*}$ is a critical point. 

Next we show that $\bm{\theta}^{*}=(a_{0}^{*},...,a_{M}^{*},{w}_{1}^{*},...,{w}_{M}^{*})$ is a local minima. Consider any perturbation $\Delta a_{1},...,\Delta a_{M}:|\Delta a_{j}|<\frac{1}{2}$ for all $j\in[M]$, $\Delta {w}_{1},...,\Delta {w}_{M}\in\mathbb{R}^{}$ and $\Delta a_{0}\in\mathbb{R}$. Define $$\tilde{\bm{\theta}}=(a_{0}^{*}+\Delta a_{0},...,a_{M}^{*}+\Delta a_{M},{w}_{1}^{*}+\Delta {w}_{1},...,{w}_{M}^{*}+\Delta {w}_{M}
).$$

Then 
\begin{align*}
\sum_{i=1}^{n}\ell_{p}(-y_{i}f(x_{i};\tilde{\bm{\theta}}))-\sum_{i=1}^{n}\ell_{p}(-y_{i}f(x_{i};\bm{\theta}^{*}))&=\sum_{i=1}^{n}\left[\ell_{p}(-y_{i}f(x_{i};\tilde{\bm{\theta}}))-\ell_{p}(-y_{i}f(x_{i};\bm{\theta}^{*}))\right]\\
&\ge\sum_{i=1}^{n}\ell_{p}'(-y_{i}f(x_{i};\bm{\theta}^{*}))(-y_{i})[f(x_{i};\tilde{\bm{\theta}})-f(x_{i};{\bm{\theta}}^{*})]\\
&=\sum_{i=1}^{n}\ell_{p}'(-y_{i}(a_{0}^{*}-M\sigma(0)))(-y_{i})[f(x_{i};\tilde{\bm{\theta}})-a_{0}^{*}+M\sigma(0)]\\
&=\sum_{i=1}^{n}\ell_{p}'(-y_{i}(a_{0}^{*}-M\sigma(0)))(-y_{i})f(x_{i};\tilde{\bm{\theta}}),
\end{align*}
where the inequality follows from the convexity of the loss function $\ell_{p}(z)$, the second equality follows from the fact that $f(x;\bm{\theta}^{*})\equiv a_{0}^{*}-M\sigma(0)$ and the third equality follows from Equation~\eqref{eq::counterexample-eq2}. In addition, we have 
\begin{align*}
&\sum_{i=1}^{n}\ell_{p}'(-y_{i}(a_{0}^{*}-M\sigma(0)))(-y_{i})f(x_{i};\tilde{\bm{\theta}})\\
&=\sum_{i=1}^{n}\ell_{p}'(-y_{i}(a_{0}^{*}-M\sigma(0)))(-y_{i})\left[\sum_{j=1}^{M}(a_{j}^{*}+\Delta a_{j})\sigma\left(\Delta w_{j}x_{i}\right)+\Delta a_{0}\right]\\
&=\sum_{i=1}^{n}\ell_{p}'(-y_{i}(a_{0}^{*}-M\sigma(0)))(-y_{i})\left[\sum_{j=1}^{M}(a_{j}^{*}+\Delta a_{j})\sigma\left(\Delta w_{j}x_{i}\right)\right]&&\text{by Eq.~\eqref{eq::counterexample-eq2}}\\
&=\sum_{j=1}^{M}-(a_{j}^{*}+\Delta a_{j})\left[\sum_{i=1}^{n}\ell_{p}'(-y_{i}(a_{0}^{*}-M\sigma(0)))y_{i}\sigma\left(\Delta w_{j}x_{i}\right)\right]\\
&=\sum_{j=1}^{M}-(a_{j}^{*}+\Delta a_{j})\left[\sum_{i=1}^{n}\ell_{p}'(-y_{i}(a_{0}^{*}-M\sigma(0)))y_{i}\sigma\left(\Delta w_{j}x_{i}\right)\right].
\end{align*}
Now we define the following function $G:\mathbb{R}\rightarrow \mathbb{R}$, 
\begin{equation*}
G(u)=\sum_{i=1}^{n}\ell_{p}'(-y_{i}(a_{0}^{*}-M\sigma(0)))y_{i}\sigma\left(ux_{i}\right).
\end{equation*}
Now we consider the gradient of the function $G$ with respect to the variable $u$ at the point $u={0}$,
\begin{align*}
\nabla_{{u}}G({0})&=\sum_{i=1}^{n}\ell_{p}'(-y_{i}(a_{0}^{*}-M\sigma(0)))y_{i}\sigma'\left(0\right)x_{i}\\
&=\sigma'(0)\left(\frac{1}{2}\ell_{p}'(-a_{0}^{*}+M\sigma(0))n_{+}-\frac{1}{2}\ell_{p}'(a_{0}^{*}-M\sigma(0))n_{-}\right)=0,
\end{align*}
by Equation~\eqref{eq::counterexample-eq2}.
Furthermore,  the second order derivative $\nabla_{u}^{2}G({0})$ satisfies
\begin{align*}
\nabla_{u}^{2}G({0})&=
\sum_{i=1}^{n}\ell_{p}'(-y_{i}(a_{0}^{*}-M\sigma(0)))y_{i}\sigma''\left(0\right)\left(x_{i}\right)^{2}=\sigma''\left(0\right)\sum_{i=1}^{n}\ell_{p}'(-y_{i}(a_{0}^{*}-M\sigma(0)))y_{i}\left(x_{i}\right)^{2}\\
&=\sigma''(0)\left[\frac{1}{n_{+}}\sum_{i:y_{i}=1}\left(x_{i}\right)^{2}-\frac{1}{n_{-}}\sum_{i:y_{i}=-1}\left(x_{i}\right)^{2}\right]> 0,
\end{align*}
then the function $G({u})=\sum_{i=1}^{n}\ell_{p}(-y_{i}(a_{0}^{*}-M\sigma(0)))y_{i}\sigma\left(ux_{i}\right)$ has a local minima at ${u}=0$. This indicates that there exists $\varepsilon>0$ such that for all $\Delta\bm{w}:\|\Delta\bm{w}\|_{2}\le\varepsilon$, 
$$\sum_{i=1}^{n}\ell_{p}'(-y_{i}(a_{0}^{*}-M\sigma(0)))y_{i}\sigma\left(\Delta\bm{w}^{\top}x_{i}\right)\ge \sum_{i=1}^{n}\ell_{p}(-y_{i}(a_{0}^{*}-M\sigma(0)))y_{i}\sigma\left(0\right)=0.$$
In addition, since  $a_{j}^{*}=-1$, $|\Delta a_{j}|<\frac{1}{2}$, then for all $\Delta\bm{w}_{j}:\|\Delta\bm{w}_{j}\|_{2}\le\varepsilon$,
\begin{align*}
\sum_{i=1}^{n}\ell_{p}'(-y_{i}(a_{0}^{*}-M\sigma(0)))(-y_{i})f(x_{i};\tilde{\bm{\theta}})=\sum_{j=1}^{M}-(a_{j}^{*}+\Delta a_{j})\left[\sum_{i=1}^{n}\ell_{p}(-y_{i}(a_{0}^{*}-M\sigma(0)))y_{i}\sigma\left(\Delta\bm{w}_{j}^{\top}x_{i}\right)\right]\ge 0.
\end{align*}
Therefore, we have 
$$\sum_{i=1}^{n}\ell_{p}'(-y_{i}(a_{0}^{*}-M\sigma(0)))(-y_{i})f(x_{i};\tilde{\bm{\theta}})\ge 0,$$
and this indicates that 
$$\sum_{i=1}^{n}\ell_{p}(-y_{i}f(x_{i};\tilde{\bm{\theta}}))-\sum_{i=1}^{n}\ell_{p}(-y_{i}f(x_{i};\bm{\theta}^{*}))\ge 0.$$
Thus, $\bm{\theta}^{*}$ is a local minima with $f(x;\bm{\theta}^{*})=a_{0}^{*}-M\sigma(0)=$ constant. Thus, 
$$\frac{1}{n}\sum_{i=1}^{n}\mathbb{I}\{y_{i}\neq \sgn(f(x_{i};\bm{\theta}^{*}))\}\ge \frac{\min\{n_{-},n_{+}\}}{n}.$$
Since the dataset is consisted of both positive and negative examples, then the training error is non-zero. 

\end{proof}

\clearpage

\subsection{Proof of Lemma~\ref{lemma::suff-necc}}\label{appendix::lemma-suff-necc}
\begin{lemma}\label{lemma::suff-necc}
If samples in the dataset $\mathcal{D}=\{(x_{i},y_{i})\}_{i=1}^{n}$ satisfies that the matrix $\sum_{i=1}^{n}\lambda_{i}y_{i}x_{i}x_{i}^{\top}$ is indefinite for all sequences $\{\lambda_{i}\ge0\}_{i=1}^{n}$ satisfying $\sum_{i:y_{i}=1}\lambda_{i}=\sum_{i:y_{i}=-1}\lambda_{i}>0$, then there exists a matrix $A\in\mathbb{R}^{d\times d}$ and two real numbers $c_{1}>0$ and $c_{2}\in\mathbb{R}$ such that 
$y_{i}(x_{i}^{\top}Ax_{i}-c_{2})>c_{1}$ holds for all $i\in[n]$.
\end{lemma}
\begin{proof}
For each sample $x_{i}$ in the dataset, let vec$(x_{i}x_{i}^{\top})$ denote the vectorization of the matrix $x_{i}x_{i}^{\top}$. Since we assume that for any sequence $\{\lambda_{i}\ge 0\}_{i=1}^{n}$ satisfying $\sum_{i:y_{i}=1}\lambda_{i}=\sum_{i:y_{i}=-1}\lambda_{i}=1$, the vector $\sum_{i=1}^{n}y_{i}\lambda_{i}\text{vec}(x_{i}x_{i}^{\top})$ does not equal to the zero vector $\bm{0}_{d^{2}}$, then we have that the convex hull of two vector sets $\mathcal{C}_{+}=\{\text{vec}(x_{i}x_{i}^{\top})\}_{i:y_{i}=1}$ and $\mathcal{C}_{-}=\{\text{vec}(x_{i}x_{i}^{\top})\}_{i:y_{i}=-1}$ are two disjoint closed compact sets. By the hyperplane separation theorem, this indicates that there exists a vector $\bm{w}\in\mathbb{R}^{d^{2}}$ and two real numbers $\tilde{c}_{1}<\tilde{c}_{2}$ such that $\bm{w}^{\top}\bm{u}>\tilde{c}_{2}$ and $\bm{w}^{\top}\bm{v}<\tilde{c}_{1}$ for all $\bm{u}\in \mathcal{C}_{+}$ and $\bm{v}\in\mathcal{C}_{-}$. This further indicates that there exists two real numbers $c_{1}>0$ and $c_{2}\in\mathbb{R}$ such that $y_{i}(x_{i}^{\top}Ax_{i}-c_{2})>c_{1}$ holds for all $i\in\mathbb{R}$.
\end{proof}

\clearpage

\subsection{Proof of Proposition~\ref{prop::quadratic-sn}}\label{appendix::quadratic-sn}

\begin{proposition}
Assume that the single layer neural network $f_{S}(x;\bm{\theta}_{S})$ has $M> d$ neurons and
assume that the neuron $\sigma$ is quadratic, i.e., $\sigma(z)=z^{2}$. Assume that the dataset $\mathcal{D}=\{(x_{i},y_{i})\}_{i=1}^{n}$ is consisted of both positive and negative samples. For all multilayer neural network $f_{D}$ parameterized by $\bm{\theta}_{D}$, every local minimum $\bm{\theta}^{*}=(\bm{\theta}_{S}^{*},\bm{\theta}_{D}^{*})$ of the empirical loss function $\hat{L}_{n}(\bm{\theta}_{S},\bm{\theta}_{D};p)$, $p\ge 6$ satisfies $\hat{R}_{n}(\bm{\theta}^{*})=0$ \textbf{ if and only if} the matrix $\sum_{i=1}^{n}\lambda_{i}y_{i}x_{i}x_{i}^{\top}$ is indefinite for all sequences $\{\lambda_{i}\ge0\}_{i=1}^{n}$ satisfying $\sum_{i:y_{i}=1}\lambda_{i}=\sum_{i:y_{i}=-1}\lambda_{i}>0$.
\vspace{-0.25cm}
\end{proposition}
\begin{proof}

\textbf{(1) Proof of ``if'': }
It follows from Lemma~\ref{lemma::suff-necc} that if the assumptions on the dataset are satisfied, 
 there exists a set of parameter $\bm{\theta}_{S}$ such that $f_{S}(x;\bm{\theta}_{S})$ achieves zero training error and this further indicates that  for any neural architecture $f_{D}$, there exists a set of parameter $\bm{\theta}^{*}=(\bm{\theta}^{*}_{S},\bm{\theta}^{*}_{D})$ such that $L_{n}(\bm{\theta}^{*};p)=0$ for all $p\ge1$. This means that the empirical loss function has a global minimum with a  value equal to zero. 

We first assume that the $\bm{\theta}^{*}=(\bm{\theta}^{*}_{1},\bm{\theta}_{2}^{*})$ is a local minimum.
We next prove the following two claims: 

\textbf{Claim 1:} If $\bm{\theta}^{*}=(\bm{\theta}_{S}^{*},\bm{\theta}_{D}^{*})$ is a local minimum and there exists $j\in[M]$ such that $a^{*}_{j}=0$, then $\error =0$. 

\textbf{Claim 2:} If $\bm{\theta}^{*}=(\bm{\theta}_{S}^{*},\bm{\theta}_{D}^{*})$ is a local minimum and $a^{*}_{j}\neq 0$ for all $j\in [M]$, then $\error =0$.

\textbf{(a) Proof of claim 1.} We prove that  if $\bm{\theta}^{*}=(\bm{\theta}_{S}^{*},\bm{\theta}_{D}^{*})$ is a local minima and there exists $j\in[M]$ such that $a^{*}_{j}=0$, then $\error =0$. Without loss of generality, we assume that $a_{1}^{*}=0$. Since $\bm{\theta}^{*}=(\bm{\theta}_{S}^{*},\bm{\theta}_{D}^{*})$ is a local minima, then there exists $\varepsilon_{0}>0$ such that for any small perturbations $\Delta{a}_{1}$, $\Delta \bm{w}_{1}$ on parameters $a^{*}_{1}$ and $ \bm{w}^{*}_{1}$, i.e., $|\Delta a_{1}|^{2}+\|\Delta\bm{w}_{1}\|_{2}^{2}\le \varepsilon_{0}^{2}$, we have 
$$\hat{L}_{n}(\tilde{\bm{\theta}}_{S},\bm{\theta}^{*}_{D})\ge \tilde{L}_{n}(\bm{\theta}^{*}_{S},\bm{\theta}_{D}^{*}),$$
where  $\tilde{\bm{\theta}}=(\tilde{a}_{0}, \tilde{a}_{1},...,\tilde{a}_{M},\tilde{\bm{w}}_{1},...,\tilde{\bm{w}}_{M})$, $\tilde{a}_{1}=a^{*}_{1}+\Delta a_{1}$, $\tilde{\bm{w}}_{1}=\bm{w}_{1}^{*}+\Delta \bm{w}_{1}$ and $\tilde{a}_{j}=a^{*}_{j}$, $\tilde{\bm{w}}_{j}=\bm{w}^{*}_{j}$ for $j\neq 1$.  Now we consider Taylor expansion of $\tilde{L}_{n}(\tilde{\bm{\theta}}_{S},\bm{\theta}^{*}_{D})$ at $(\bm{\theta}^{*}_{S},\bm{\theta}_{D}^{*})$. We note here that the Taylor expansion of $\hat{L}(\bm{\theta}_{S},\bm{\theta}_{D}^{*};p)$ on $\bm{\theta}_{S}$ always exists, since the empirical loss function $\hat{L}_{n}$ has continuous derivatives with respect to $f_{S}$ up to the $p$-th order and the output of the neural network $f(x;\bm{\theta}_{S})$ is infinitely differentiable with respect to $\bm{\theta}_{S}$ due to the fact that neuron activation function $\sigma$ is real analytic.

We first calculate the first order derivatives at the point $(\bm{\theta}^{*}_{1},\bm{\theta}_{2}^{*})$
\begin{align*}
\frac{d\hat{L}_{n}(\bm{\theta}^{*})}{da_{1}}&=\sum_{i=1}^{n}\ell_{p}'(-y_{i}f(x_{i};\bm{\theta}^{*}))(-y_{i})\sigma\left({\bm{w}_{1}^{*}}^{\top}x_{i}\right)=0,&& \text{$\bm{\theta}^{*}$ is a critical point,}\\
\nabla_{\bm{w}_{1}}\hat{L}_{n}(\bm{\theta}^{*})&=a^{*}_{1}\sum_{i=1}^{n}\ell_{p}'(-y_{i}f(x_{i};\bm{\theta}^{*}))(-y_{i})\sigma'\left({\bm{w}_{1}^{*}}^{\top}x_{i}\right)x_{i}=\bm{0}_{d},&& \text{$\bm{\theta}^{*}$ is a critical point.}
\end{align*}
Next, we calculate the second order derivatives at the point $(\bm{\theta}^{*}_{1},\bm{\theta}_{2}^{*})$,
\begin{align*}
\frac{d^{2}\hat{L}(\bm{\theta}^{*})}{da_{1}^{2}}&=\sum_{i=1}^{N}\ell''_{p}(-y_{i}f(x_{i};\bm{\theta}^{*}))\sigma^{2}\left({\bm{w}_{1}^{*}}^{\top}x_{i}\right)\ge 0,\\
\frac{d}{da_{1}}(\nabla_{\bm{w}_{1}}L(\bm{\theta}^{*}))&=\sum_{i=1}^{n}\ell_{p}'(-y_{i}f(x_{i};\bm{\theta}^{*}))(-y_{i})\sigma'\left({\bm{w}_{1}^{*}}^{\top}x_{i}\right)x_{i}\\
&\quad+a^{*}_{1}\sum_{i=1}^{n}\ell''_{p}(-y_{i}f(x_{i};\bm{\theta}^{*}))\sigma\left({\bm{w}_{1}^{*}}^{\top}x_{i}\right)\sigma'\left({\bm{w}_{1}^{*}}^{\top}x_{i}\right)x_{i}\\
&=\bm{0}_{d},
\end{align*}
where the first term equals to the zero vector by  the necessary condition for a local minima presented in Lemma~\ref{lemma::nec-single} and the second term equals to the zero vector by the assumption that $a^{*}_{1}=0$. Furthermore, by the assumption that $a^{*}_{1}=0$, we have 
\begin{equation*}
\nabla^{2}_{\bm{w}_{1}}\hat{L}_{n}(\bm{\theta}^{*};p)=a_{1}^{*}\nabla_{w_{1}}\left[\sum_{i=1}^{n}\ell_{p}'(-y_{i}f(x_{i};\bm{\theta}))(-y_{i})\sigma'\left({\bm{w}_{1}^{*}}^{\top}x_{i}\right)x_{i}\right]=\bm{0}_{d\times d}.
\end{equation*}
We further calculate the third order derivatives 
\begin{align*}
\frac{d}{da_{1}}\left[\nabla_{\bm{w}_{1}}^{2}{ \hat{L}_{n}(\bm{\theta}^{*};p)}\right]&=\frac{d}{da_{1}}\left[a_{1}^{*}\nabla_{\bm{w}_{1}}\left[\sum_{i=1}^{n}\ell_{p}'(-y_{i}f(x_{i};\bm{\theta}))(-y_{i})\sigma'\left({\bm{w}_{1}^{*}}^{\top}x_{i}\right)x_{i}\right]\right]\\
&=\nabla_{\bm{w}_{1}}\left[\sum_{i=1}^{n}\ell_{p}'(-y_{i}f(x_{i};\bm{\theta}))(-y_{i})\sigma'\left({\bm{w}_{1}^{*}}^{\top}x_{i}\right)x_{i}\right]+\bm{0}_{d\times d}&& \text{by $a_{1}^{*}=0$}\\
&=\sum_{i=1}^{n}\ell_{p}'(-y_{i}f(x_{i};\bm{\theta}))(-y_{i})\sigma''\left({\bm{w}_{1}^{*}}^{\top}x_{i}\right)x_{i}x_{i}^{\top}\\
&\quad+a^{*}_{1}\sum_{i=1}^{n}\ell_{p}''(-y_{i}f(x_{i};\bm{\theta}))\left[\sigma'\left({\bm{w}_{1}^{*}}^{\top}x_{i}\right)\right]^{2}x_{i}x_{i}^{\top}\\
&=\sum_{i=1}^{n}\ell_{p}'(-y_{i}f(x_{i};\bm{\theta}))(-y_{i})\sigma''\left({\bm{w}_{1}^{*}}^{\top}x_{i}\right)x_{i}x_{i}^{\top}&& \text{by $a_{1}^{*}=0$}
\end{align*}
and
$$\nabla^{3}_{\bm{w}_{1}}\hat{L}_{n}(\bm{\theta}^{*};p)=a^{*}_{1}\nabla^{2}_{\bm{w}_{1}}\left[\sum_{i=1}^{n}\ell_{p}'(-y_{i}f(x_{i};\bm{\theta}))(-y_{i})\sigma'\left({\bm{w}_{1}^{*}}^{\top}x_{i}\right)x_{i}\right]=\bm{0}_{d\times d\times d}.$$
In fact, it is easy to show that for any $2\le k\le p$, 
$$\nabla^{k}_{\bm{w}_{1}}\hat{L}_{n}(\bm{\theta}^{*};p)=a_{1}^{*}\nabla^{k-1}_{\bm{w}_{1}}\left[\sum_{i=1}^{n}\ell_{p}'(-y_{i}f(x_{i};\bm{\theta}))(-y_{i})\sigma'\left({\bm{w}_{1}^{*}}^{\top}x_{i}\right)x_{i}\right]=\bm{0}_{\underbrace{d\times d\times  ...\times d}_{ \text{$k$ times}}}.$$
Let $\varepsilon>0$, $\Delta a_{1}=\text{sgn}(a_{1})\varepsilon^{9/4}$ and $\Delta \bm{w}_{1}=\varepsilon \bm{u}_{1}$ for $\bm{u}_{1}:\|\bm{u}_{1}\|_{2}=1$. Clearly, when $\varepsilon\rightarrow 0$, $\Delta a_{1}=o(\|\Delta \bm{w}_{1}\|_{2})$, $\Delta a_{1}=o(1)$ and $\|\Delta \bm{w}_{1}\|=o(1)$. Then we expand $\hat{L}_{n}(\tilde{\bm{\theta}}_{1},\bm{\theta}_{2}^{*})$ at the point $\bm{\theta}^{*}$ up to the sixth order  and thus as $\varepsilon\rightarrow 0$,
\begin{align*}
\hat{L}_{n}(\tilde{\bm{\theta}}_{1}, \bm{\theta}_{2}^{*})&=\hat{L}_{n}({\bm{\theta}}^{*}_{1}, \bm{\theta}_{2}^{*})+\frac{1}{2!}\frac{d^{2}\hat{L}_{n}(\bm{\theta}^{*})}{d^{2}a_{1}}(\Delta a_{1})^{2}\\
&\quad+\frac{1}{2}\Delta a_{1}\Delta \bm{w}_{1}^{\top}\frac{d}{da_{1}}\left[\bm{D}_{\bm{w}_{1}}^{2}{ \hat{L}_{n}(\bm{\theta}^{*};p)}\right]\Delta \bm{w}_{1} + o(|a_{1}|^{2})+o(|a_{1}|\|\bm{w}_{1}\|^{2}_{2})+o(\|\Delta \bm{w}_{1}\|_{2}^{5})\\
&=\hat{L}_{n}({\bm{\theta}}^{*}_{1}, \bm{\theta}_{2}^{*})+\frac{1}{2!}\frac{d^{2}\hat{L}_{n}(\bm{\theta}^{*})}{d^{2}a_{1}}\varepsilon^{9/2}+\frac{1}{2}\text{sgn}(a_{1}) \varepsilon^{9/4+2}\sum_{i=1}^{n}\ell_{p}'(-y_{i}f(x_{i};\bm{\theta}))\sigma''\left({\bm{w}_{1}^{*}}^{\top}x_{i}\right)(\bm{u}_{1}^{\top}x_{i})^{2}\\
&\quad+o(\varepsilon^{9/2})+o(\varepsilon^{9/4+2})+o(\varepsilon^{5})\\
&=\hat{L}_{n}({\bm{\theta}}^{*}_{1}, \bm{\theta}_{2}^{*})+\frac{1}{2}\text{sgn}(a_{1})\varepsilon^{17/4}\sum_{i=1}^{n}\ell_{p}'(-y_{i}f(x_{i};\bm{\theta}))(-y_{i})\sigma''\left({\bm{w}_{1}^{*}}^{\top}x_{i}\right)(\bm{u}_{1}^{\top}x_{i})^{2}+o(\varepsilon^{17/4})
\end{align*}
Since $\varepsilon>0$ and $\hat{L}_{n}(\tilde{\bm{\theta}}_{1},\bm{\theta}^{*}_{2};p)\ge \loss$ holds for any $\bm{u}_{1}:\|\bm{u}_{1}\|_{2}=1$ and any $\sgn(a_{1})\in\{-1, 1\}$, then 
\begin{equation}\label{eq::prop11-part1-cond}\sum_{i=1}^{n}\ell_{p}'(-y_{i}f(x_{i};\bm{\theta}))(-y_{i})\sigma''\left({\bm{w}_{1}^{*}}^{\top}x_{i}\right)(\bm{u}^{\top}x_{i})^{2}=0, \quad\text{for any } \bm{u}\in\mathbb{R}^{d}.\end{equation}
Therefore, 
\begin{equation*}
\sum_{i=1}^{n}\ell_{p}'(-y_{i}f(x_{i};\bm{\theta}))(-y_{i})\sigma''\left({\bm{w}_{1}^{*}}^{\top}x_{i}\right)x_{i}x_{i}^{\top}=\bm{0}_{d\times d}.
\end{equation*}
Since $\sigma''(z)=2$ for all $z$, then 
\begin{equation}\label{eq::prop-quadratic-2}
\sum_{i=1}^{n}\ell_{p}'(-y_{i}f(x_{i};\bm{\theta}))(-y_{i})x_{i}x_{i}^{\top}=\bm{0}_{d\times d}.
\end{equation}
Furthermore, since $\theta^{*}$ is a critical point, then 
\begin{equation}\label{eq::prop-quadratic-1}
\frac{\partial \hat{L}_{n}(\bm{\theta};p)}{\partial a_{0}}=\frac{1}{n}\sum_{i=1}^{n}\ell'(-y_{i}f(x_{i};\bm{\theta}^{*}))(-y_{i})=0.
\end{equation}
Now we assume that $\hat{R}_{n}(\bm{\theta}^{*})>0$. This means that there exists a index $i$ such that $y_{i}f(x_{i};\bm{\theta}^{*})<0$ or $\ell'(-y_{i}f(x_{i};\bm{\theta}^{*}))>0$. 
Furthermore, since $\ell'(z)\ge 0$, then by setting $\lambda_{i}=\ell'(-y_{i}f(x_{i};\bm{\theta}^{*}))$, we have that there exists a sequence $\{\lambda_{i}\ge 0\}_{i=1}^{n}$ satisfying $\sum_{i:y_{i}=1}\lambda_{i}=\sum_{i:y_{i}=-1}\lambda_{i}>0$, where the equality follows from Equation~\eqref{eq::prop-quadratic-1} and the positiveness comes from the assumption that $\ell'(-y_{i}f(x_{i};\bm{\theta}^{*}))>0$ for some $i$, such that 
$$\sum_{i=1}^{n}\lambda_{i}y_{i}x_{i}x_{i}^{\top}=\bm{0}_{d\times d},$$
where the equality follows from Equation~\eqref{eq::prop-quadratic-2}. This leads to the contradiction with our assumption that the matrix $\sum_{i=1}^{n}\lambda_{i}y_{i}x_{i}x_{i}^{\top}$ should be indefinite for all sequences $\{\lambda_{i}\ge 0\}_{i=1}^{n}$ satisfying $\sum_{i:y_{i}=1}\lambda_{i}=\sum_{i:y_{i}=-1}\lambda_{i}>0$. Therefore, this indicates that $\hat{R}_{n}(\bm{\theta}^{*})=0.$

\textbf{(b) Proof of Claim 2:} To prove the claim 2, we first show that if $M>d$, then there exists coefficients $\alpha_{1},...,\alpha_{M}$, not all zero, such that $$\left(\alpha_{1}\bm{w}_{1}^{*}+...+\alpha_{M}\bm{w}_{M}^{*}\right)^{\top}x_{i}=0,\quad \text{for all }i\in[n].$$
Clearly, if $M>r$, then there exists coefficients $\alpha_{1},...,\alpha_{M}$, not all zero, such that $$(\alpha_{1}\bm{w}_{1}^{*}+...+\alpha_{M}\bm{w}_{M}^{*})=\bm{0}_{d},\quad \text{for all }i\in[n].$$
Now we prove the claim 2. First, we consider the Hessian matrix $H(\bm{w}_{1}^{*},...,\bm{w}_{M}^{*})$. Since $\bm{\theta}^{*}$ is a local minima, then 
\begin{equation*}
F(\bm{u}_{1},...,\bm{u}_{M})=\sum_{j=1}^{M}\sum_{k=1}^{M}\bm{u}_{j}^{\top}\nabla^{2}_{\bm{w}_{j},\bm{w}_{k}}\loss \bm{u}_{k}\ge 0
\end{equation*}
holds for any vectors $\bm{u}_{1},...,\bm{u}_{M}\in\mathbb{R}^{d}$. 
Since $\sigma''(z)=2$ and $\sigma'(z)=2z$ for all $z\in\mathbb{R}$, then
\begin{align*}
\nabla_{\bm{w}_{j}}^{2}\loss&=a_{j}^{*}\sum_{i=1}^{n}\ell_{p}'(-y_{i}f(x_{i};\bm{\theta}^{*}))(-y_{i})\sigma''\left({\bm{w}_{j}^{*}}^{\top}x_{i}\right)x_{i}x_{i}^{\top}\\
&\quad +{a_{j}^{*}}^{2}\sum_{i=1}^{n}\ell_{p}''(-y_{i}f(x_{i};\bm{\theta}^{*}))\left[\sigma'\left({\bm{w}_{j}^{*}}^{\top}x_{i}\right)\right]^{2}x_{i}x_{i}^{\top}\\
&=-2a_{j}^{*}\sum_{i=1}^{n}\ell_{p}'(-y_{i}f(x_{i};\bm{\theta}^{*}))y_{i}x_{i}x_{i}^{\top}+4{a_{j}^{*}}^{2}\sum_{i=1}^{n}\ell_{p}''(-y_{i}f(x_{i};\bm{\theta}^{*}))\left({\bm{w}_{j}^{*}}^{\top}x_{i}\right)^{2}x_{i}x_{i}^{\top},
\end{align*}
and 
\begin{align*}
\nabla_{\bm{w}_{j},\bm{w}_{k}}^{2}\loss&={a_{j}^{*}}a_{k}^{*}\sum_{i=1}^{n}\ell_{p}''(-y_{i}f(x_{i};\bm{\theta}^{*}))\left[\sigma'\left({\bm{w}_{j}^{*}}^{\top}x_{i}\right)\right]\left[\sigma'\left({\bm{w}_{k}^{*}}^{\top}x_{i}\right)\right]x_{i}x_{i}^{\top}\\
&=4{a_{j}^{*}}a_{k}^{*}\sum_{i=1}^{n}\ell_{p}''(-y_{i}f(x_{i};\bm{\theta}^{*}))\left({\bm{w}_{j}^{*}}^{\top}x_{i}\right)\left({\bm{w}_{k}^{*}}^{\top}x_{i}\right)x_{i}x_{i}^{\top}.
\end{align*}
Thus, we have 
\begin{align*}
F(\bm{u}_{1},...,\bm{u}_{M})&=-2\sum_{j=1}^{M}\left[a_{j}^{*}\sum_{i=1}^{n}\ell_{p}'(-y_{i}f(x_{i};\bm{\theta}^{*}))y_{i}\left(\bm{u}_{j}^{\top}x_{i}\right)^{2}\right]\\
&\quad +4\sum_{j=1}^{M}\sum_{k=1}^{M}\left[{a_{j}^{*}}a_{k}^{*}\sum_{i=1}^{n}\ell_{p}''(-y_{i}f(x_{i};\bm{\theta}^{*}))\left({\bm{w}_{j}^{*}}^{\top}x_{i}\right)\left({\bm{w}_{k}^{*}}^{\top}x_{i}\right)\left(\bm{u}_{j}^{\top}x_{i}\right)\left(\bm{u}_{k}^{\top}x_{i}\right)\right]\\
&=-2\sum_{j=1}^{M}\left[a_{j}^{*}\sum_{i=1}^{n}\ell_{p}'(-y_{i}f(x_{i};\bm{\theta}^{*}))y_{i}\left(\bm{u}_{j}^{\top}x_{i}\right)^{2}\right]\\
&\quad +4\sum_{i=1}^{n}\left[\ell_{p}''(-y_{i}f(x_{i};\bm{\theta}^{*}))\left(\sum_{j=1}^{M}a_{j}^{*}\left({\bm{w}_{j}^{*}}^{\top}x_{i}\right)\left(\bm{u}_{j}^{\top}x_{i}\right)\right)^{2}\right].
\end{align*}
Since there exists coefficients $\alpha_{1},...,\alpha_{M}$, not all zero, such that $(\alpha_{1}\bm{w}_{1}^{*}+...+\alpha_{M}\bm{w}_{M}^{*})^{\top}x_{i}=0$, for all $i\in[n],$ and $a_{j}^{*}\neq 0$ for all $j\in[M]$ then by setting $\bm{u}_{j}=\alpha_{j}\bm{u}/a_{j}^{*}$ for all $j\in[M]$, we have that the inequality
\begin{align*}
F(\bm{u}_{1},...,\bm{u}_{M})&=-2\sum_{j=1}^{M}\left[a_{j}^{*}\sum_{i=1}^{n}\ell_{p}'(-y_{i}f(x_{i};\bm{\theta}^{*}))y_{i}\left(\alpha_{j}/a_{j}^{*}\right)^{2}\left(\bm{u}^{\top}x_{i}\right)^{2}\right]\\
&\quad +4\sum_{i=1}^{n}\left[\ell_{p}''(-y_{i}f(x_{i};\bm{\theta}^{*}))\left(\sum_{j=1}^{M}\alpha_{j}\left({\bm{w}_{j}^{*}}^{\top}x_{i}\right)\left(\bm{u}^{\top}x_{i}\right)\right)^{2}\right]\\
&=-2\sum_{j=1}^{M}\left[a_{j}^{*}\sum_{i=1}^{n}\ell_{p}'(-y_{i}f(x_{i};\bm{\theta}^{*}))y_{i}\left(\alpha_{j}/a_{j}^{*}\right)^{2}\left(\bm{u}^{\top}x_{i}\right)^{2}\right]\\
&\quad +4\sum_{i=1}^{n}\left[\ell_{p}''(-y_{i}f(x_{i};\bm{\theta}^{*}))\left(\left(\sum_{j=1}^{M}\alpha_{j}{\bm{w}_{j}^{*}}\right)^{\top}x_{i}\right)^{2}\left(\bm{u}^{\top}x_{i}\right)^{2}\right]\\
&=-2\sum_{j=1}^{M}\left(\alpha_{j}^{2}/a_{j}^{*}\right)\cdot\sum_{i=1}^{n}\ell_{p}'(-y_{i}f(x_{i};\bm{\theta}^{*}))y_{i}\left(\bm{u}^{\top}x_{i}\right)^{2}\ge 0
\end{align*}
holds for any $\bm{u}\in\mathbb{R}^{d}$.

Next we consider the following two cases: (1) $\sum_{j=1}^{M}\left(\alpha_{j}^{2}/a_{j}^{*}\right)\neq 0$; (2) $\sum_{j=1}^{M}\left(\alpha_{j}^{2}/a_{j}^{*}\right)=0$. 

\textbf{Case 1: } If $\sum_{j=1}^{M}\left(\alpha_{j}^{2}/a_{j}^{*}\right)\neq 0$, then without loss of generality, we assume that $\sum_{j=1}^{M}\left(\alpha_{j}^{2}/a_{j}^{*}\right)<0$. This indicates that 
\begin{equation}\label{eq::prop-quadratic-4}\sum_{i=1}^{n}\ell_{p}'(-y_{i}f(x_{i};\bm{\theta}^{*}))y_{i}\left(\bm{u}^{\top}x_{i}\right)^{2}\ge 0,\quad \text{for all }\bm{u}\in\mathbb{R}^{d}.\end{equation}
Since $\bm{\theta}^{*}$ is a critical point, then 
\begin{equation}\label{eq::prop-quadratic-3}
\frac{\partial \hat{L}_{n}(\bm{\theta}^{*};p)}{\partial a_{0}}=\frac{1}{n}\sum_{i=1}^{n}\ell'(-y_{i}f(x_{i};\bm{\theta}^{*}))(-y_{i})=0.
\end{equation}
Now we assume that $\hat{R}_{n}(\bm{\theta}^{*})>0$. This means that there exists a index $i$ such that $y_{i}f(x_{i};\bm{\theta}^{*})<0$ or $\ell'(-y_{i}f(x_{i};\bm{\theta}^{*}))>0$. 
Furthermore, since $\ell'(z)\ge 0$, then by setting $\lambda_{i}=\ell'(-y_{i}f(x_{i};\bm{\theta}^{*}))$, we have that there exists a sequence $\{\lambda_{i}\ge 0\}_{i=1}^{n}$ satisfying $\sum_{i:y_{i}=1}\lambda_{i}=\sum_{i:y_{i}=-1}\lambda_{i}>0$, where the equality follows from Equation~\eqref{eq::prop-quadratic-1} and the positiveness comes from the assumption that $\ell'(-y_{i}f(x_{i};\bm{\theta}^{*}))>0$ for some $i$, such that 
$$\sum_{i=1}^{n}\lambda_{i}y_{i}x_{i}x_{i}^{\top}\succeq 0,$$
where the positive semi-definiteness follows from the inequality~\eqref{eq::prop-quadratic-4}. This leads to the contradiction with our assumption that the matrix $\sum_{i=1}^{n}\lambda_{i}y_{i}x_{i}x_{i}^{\top}$ should be indefinite for all sequences $\{\lambda_{i}\ge 0\}_{i=1}^{n}$ satisfying $\sum_{i:y_{i}=1}\lambda_{i}=\sum_{i:y_{i}=-1}\lambda_{i}>0$. Therefore, this indicates that $\hat{R}_{n}(\bm{\theta}^{*})=0.$

\textbf{{Case 2}:} If $\sum_{j=1}^{M}\left(\alpha_{j}^{2}/a_{j}^{*}\right)= 0$,  then by setting $\bm{u}_{j}=(\alpha_{j}/a_{j}^{*}+v\sgn(\alpha_{j}))\bm{u}$ for some scalar $v$ and vector $\bm{u}\in\mathbb{R}^{d}$, we have 
\begin{align*}
F(v,\bm{u})&=-2\sum_{j=1}^{M}\left[a_{j}^{*}\sum_{i=1}^{n}\ell_{p}'(-y_{i}f(x_{i};\bm{\theta}^{*}))y_{i}\left((\alpha_{j}/a_{j}^{*}+v\sgn(\alpha_{j}))\bm{u}^{\top}x_{i}\right)^{2}\right]\\
&\quad +4\sum_{i=1}^{n}\left[\ell_{p}''(-y_{i}f(x_{i};\bm{\theta}^{*}))\left(\sum_{j=1}^{M}a_{j}^{*}\left({\bm{w}_{j}^{*}}^{\top}x_{i}\right)\left((\alpha_{j}/a_{j}^{*}+v\sgn(\alpha_{j}))\bm{u}^{\top}x_{i}\right)\right)^{2}\right]\\
&=-2\sum_{j=1}^{M}\left[a_{j}^{*}\sum_{i=1}^{n}\ell_{p}'(-y_{i}f(x_{i};\bm{\theta}^{*}))y_{i}\left((\alpha_{j}/a_{j}^{*}+v\sgn(\alpha_{j}))\bm{u}^{\top}x_{i}\right)^{2}\right]\\
&\quad +4\sum_{i=1}^{n}\left[\ell_{p}''(-y_{i}f(x_{i};\bm{\theta}^{*}))\left(\left(\sum_{j=1}^{M}(\alpha_{j}+v\sgn(\alpha_{j})a^{*}_{j})\bm{w}_{j}^{*}\right)^{\top}x_{i}\right)\left(\bm{u}^{\top}x_{i}\right)^{2}\right]\\
&=-2\sum_{j=1}^{M}\left[a_{j}^{*}\sum_{i=1}^{n}\ell_{p}'(-y_{i}f(x_{i};\bm{\theta}^{*}))y_{i}\left((\alpha_{j}/a_{j}^{*}+v\sgn(\alpha_{j}))\bm{u}^{\top}x_{i}\right)^{2}\right]\\
&\quad +4v^{2}\sum_{i=1}^{n}\left[\ell_{p}''(-y_{i}f(x_{i};\bm{\theta}^{*}))\left(\left(\sum_{j=1}^{M}\sgn(\alpha_{j})a_{j}^{*}\bm{w}_{j}^{*}\right)^{\top}x_{i}\right)^{2}\left(\bm{u}^{\top}x_{i}\right)^{2}\right]\\
&\triangleq  -2\sum_{j=1}^{M}\left[a_{j}^{*}\sum_{i=1}^{n}\ell_{p}'(-y_{i}f(x_{i};\bm{\theta}^{*}))y_{i}\left((\alpha_{j}/a_{j}^{*}+v\sgn(\alpha_{j}))\bm{u}^{\top}x_{i}\right)^{2}\right]+v^{2}R(\bm{u}),
\end{align*}
where we define 
$$R(\bm{u})=4\sum_{i=1}^{n}\left[\ell_{p}''(-y_{i}f(x_{i};\bm{\theta}^{*}))\left(\left(\sum_{j=1}^{M}\sgn(\alpha_{j})a_{j}^{*}\bm{w}_{j}^{*}\right)^{\top}x_{i}\right)^{2}\left(\bm{u}^{\top}x_{i}\right)^{2}\right].$$
In addition, we have 
\begin{align*}
\sum_{j=1}^{M}&\left[a_{j}^{*}\sum_{i=1}^{n}\ell_{p}'(-y_{i}f(x_{i};\bm{\theta}^{*}))y_{i}\left((\alpha_{j}/a_{j}^{*}+v\sgn(\alpha_{j}))\bm{u}^{\top}x_{i}\right)^{2}\right]\\
&=\sum_{i=1}^{n}\ell'_{p}(-y_{i}f(x_{i};\bm{\theta}))y_{i}(\bm{u}^{\top}x_{i})^{2}\cdot\left[\sum_{j=1}^{M}(\alpha_{j}^{2}/a_{j}^{*}+2v\sgn(\alpha_{j})\alpha_{j}+v^{2}a_{j}^{*})\right]\\
&=\sum_{i=1}^{n}\ell'_{p}(-y_{i}f(x_{i};\bm{\theta}))y_{i}(\bm{u}^{\top}x_{i})^{2}\cdot\left[\sum_{j=1}^{M}(2v\sgn(\alpha_{j})\alpha_{j}+v^{2}a_{j}^{*})\right]\\
&=2v\left[\sum_{j=1}^{M}|\alpha_{j}|\right]\sum_{i=1}^{n}\ell'_{p}(-y_{i}f(x_{i};\bm{\theta}))y_{i}(\bm{u}^{\top}x_{i})^{2}+v^{2}\left[\sum_{j=1}^{M}a_{j}^{*}\right]\sum_{i=1}^{n}\ell'_{p}(-y_{i}f(x_{i};\bm{\theta}))y_{i}(\bm{u}^{\top}x_{i})^{2}.
\end{align*}
Therefore, we can rewrite $F(v, \bm{u})$ as 
\begin{align*}
F(v,\bm{u})&=2v\sum_{j=1}^{M}|\alpha_{j}|\sum_{i=1}^{n}\ell'_{p}(-y_{i}f(x_{i};\bm{\theta}))y_{i}(\bm{u}^{\top}x_{i})^{2}+v^{2}\sum_{j=1}^{M}a_{j}^{*}\cdot\sum_{i=1}^{n}\ell'_{p}(-y_{i}f(x_{i};\bm{\theta}))y_{i}(\bm{u}^{\top}x_{i})^{2}+v^{2}R(\bm{u})\\
&\triangleq2v\sum_{j=1}^{M}|\alpha_{j}|\sum_{i=1}^{n}\ell'_{p}(-y_{i}f(x_{i};\bm{\theta}))y_{i}(\bm{u}^{\top}x_{i})^{2}+v^{2}\hat{R}(\bm{u})
\end{align*}
Since $F(\bm{v},\bm{u})\ge 0$ holds for any scalar $v$ and vector $\bm{u}\in\mathbb{R}^{d}$, then we should have 
$$\sum_{j=1}^{M}|\alpha_{j}|\sum_{i=1}^{n}\ell'_{p}(-y_{i}f(x_{i};\bm{\theta}))y_{i}(\bm{u}^{\top}x_{i})^{2}=0,\quad \text{ for any }\bm{u}\in\mathbb{R}^{d}. $$
Since the coefficient $\alpha_{1},...,\alpha_{M}$ are not all zero, then for any $\bm{u}\in\mathbb{R}^{d}$, we have 
$$\sum_{i=1}^{n}\ell'_{p}(-y_{i}f(x_{i};\bm{\theta}))y_{i}(\bm{u}^{\top}x_{i})^{2}=0. $$
Applying the same analysis shown earlier, we have  $\error =0$.  

\textbf{Proof of ``only if'':} 
We prove the necessary condition by proving the following claim.
\begin{claim}
If  there exists a sequence $\{\lambda_{i}\ge 0\}_{i=1}^{n}$ satisfying $\sum_{i:y_{i}=1}\lambda_{i}=\sum_{i:y_{i}=-1}\lambda_{i}>0$ such that the matrix 
$\sum_{i=1}^{n}\lambda_{i}y_{i}x_{i}x_{i}^{\top}$ is positive or negative positive semi-definite, 
then there exists a multilayer neural architecture $f_{D}$ such that the empirical loss function $\hat{L}_{n}(\bm{\theta}_{S},\bm{\theta}_{D};p),p\ge 6$ has a local minimum with a non-zero training error.  
\end{claim}
\begin{proof}
Let $\mathcal{D}=\{(x_{i},y_{i})\}_{i=1}^{n}$ denote a dataset consisting of $n$ samples. We rewrite the sample $x$ as  $x=\left(x^{(1)},...,x^{(d)}\right)$.
Consider the following network,
$$f(x;\bm{\theta})=f_{S}(x;\bm{\theta}_{S})+f_{D}(x;\bm{\theta}_{D}),$$
where $$f_{S}(x;\bm{\theta}_{S})=a_{0}+\sum_{j=1}^{M}a_{j}\sigma(\bm{w}_{j}^{\top}x_{i}+b_{j}),$$
and the multilayer network is defined as follows,
\begin{equation}
f_{D}(x;\bm{\theta}_{D})=f_{D}(x;\theta_{1},...,\theta_{d})=\sum_{i=1}^{n}\mu_{i}\prod_{k=1}^{d}\bm{1}\left\{x^{(k)}\in\left[x_{i}^{(k)}-\theta_{k},x_{i}^{(k)}+\theta_{k}\right]\right\}.
\end{equation}
We note here that $\mu_{1},...,\mu_{n}$ are not parameters and later we will show that this function can be implemented by a multilayer network consisted of threshold units. A useful property of the function $f_{D}(x;\bm{\theta}_{D})$ is that if all parameters $\theta_{i}$s are positive and sufficiently smalls, then for each sample $(x_{i},y_{i})$ in the dataset, 
$$f_{D}(x_{i};\bm{\theta}_{D})=\mu_{i}.$$
Furthermore, if we slightly perturb all parameters, the output of the function $f_{D}$ on all samples remain the same. In the proof, we use these two properties to construct the local minimum with a non-zero training error. 

By assumption, there exists a sequence $\{\lambda_{i}\ge 0\}_{i=1}^{n}$ satisfying $\sum_{i:y_{i}=1}\lambda_{i}=\sum_{i:y_{i}=-1}\lambda_{i}>0$ such that the matrix 
$\sum_{i=1}^{n}\lambda_{i}y_{i}x_{i}x_{i}^{\top}$ is positive or negative semi-definite.
Without loss of generality, we assume that the matrix is positive semi-definite.
Now we construct a local minimum $\bm{\theta}^{*}$. Let $a_{0}^{*}=a_{1}^{*}=...=a_{M}^{*}=-1$, $\bm{w}^{*}_{1}=...=\bm{w}^{*}_{M}=\bm{0}_{d}$ and $b_{1}^{*}=...=b_{M}^{*}=0$. Now we set $\theta^{*}_{1},...,\theta^{*}_{d}$ to be positive and sufficiently small such that for two different samples in the dataset, e.g., $x_{i}\neq x_{j}$, the following equations holds,
$$\prod_{k=1}^{d}\bm{1}\left\{x_{j}^{(k)}\in\left[x_{i}^{(k)}-2\theta^{*}_{k},x_{i}^{(k)}+2\theta^{*}_{k}\right]\right\}=0,\quad\prod_{k=1}^{d}\bm{1}\left\{x_{i}^{(k)}\in\left[x_{j}^{(k)}-2\theta^{*}_{k},x_{j}^{(k)}+2\theta^{*}_{k}\right]\right\}=0.$$
%Furthermore, since the loss function satisfies $\ell'_{p}(z)\ge 0$ and $\ell'_{p}(z)=0$ iff $z\le -z_{0}$, then for  $\forall z>-z_{0}$, $\ell_{p}'(z)>0$. This indicates that the loss function $\ell$ is monotonically increasing on the interval $(-z_{0},\infty)$. Furthermore, since $\ell_{p}(z)\ge \mathbb{I}\{z> 0\}$, then $\ell_{p}(z)>1$ for all $z>0$.
Now we choose $\mu_{1},...,\mu_{n}$ as follows. The output of the neural network on sample $x_{i}$ in the dataset is $f(x_{i};\bm{\theta}^{*})=\mu_{i}-M\sigma(0)$.

We need to choose $\mu_{1},...,\mu_{n}$ to satisfy all conditions shown as follows:
\begin{itemize}
\item[(1)] There exists $i\in[n]$ such that $y_{i}(\mu_{i}-M\sigma(0))<0$.
\item[(2)] For all $i:y_{i}=1$ and all $k:y_{k}=-1$, $$\frac{\ell'(-y_{i}(\mu_{i}-M\sigma(0)))}{\sum_{j:j=1}\ell'(-y_{i}(\mu_{i}-M\sigma(0)))}=\frac{\lambda_{i}}{\sum_{j:j=1}\lambda_{j}},\quad \frac{\ell'(-y_{k}(\mu_{k}-M\sigma(0)))}{\sum_{j:j=-1}\ell'(-y_{i}(\mu_{i}-M\sigma(0)))}=\frac{\lambda_{k}}{\sum_{j:j=-1}\lambda_{j}},$$
and 
$$\sum_{j:j=1}\ell'(-y_{i}(\mu_{i}-M\sigma(0)))=\sum_{j:j=-1}\ell'(-y_{i}(\mu_{i}-M\sigma(0))).$$
\end{itemize}  
Now we start from the largest element in the sequence $\{\lambda_{i}\}_{i=1}^{n}$. Since $\sum_{i=1}^{n}\lambda_{i}>0$, the define the index $i_{\max}$ as the index of the largest element, i.e., $$i_{\max}=\arg\max_{i}\lambda_{i}.$$
Let $\lambda_{\max}=\lambda_{i_{\max}}$.
Now we choose $\mu_{i_{\max}}$ such that
$$y_{i_{\max}}(\mu_{i_{\max}}-M\sigma(0))=-1.$$
Thus, the index $i_{\max}$ satisfy the first condition. Then for $i\neq i_{\max}$, we choose $\mu_{i}$ such that 
\begin{equation}\label{eq::prop-necc-1}\ell'(-y_{i}(\mu_{i}-M\sigma(0)))=\frac{\lambda_{i}}{\lambda_{\max}}\ell(-y_{i_{\max}}(\mu_{i\max}-M\sigma(0)))=\frac{\lambda_{i}}{\lambda_{\max}}\ell'(1)\le\ell'(1).\end{equation}
We note here that for each $i\in[n]$, there always exists a $\mu_{i}$ solving the above equation. This can be seen by the fact that $\ell'$ is continuous, $\ell'_{p}(z)\ge 0$ and $\ell'_{p}(z)=0$ iff $z\le -z_{0}$. This indicates that for  $\forall z>-z_{0}$, $\ell_{p}'(z)>0$, i.e., $\ell'(1)>0$ and that $\ell'(-z_{0})=0$. Since $\ell'(z)$ is continuous, then for $\forall r\in[0,\ell'(1)]$, there always exists $z\in\mathbb{R}$ such that $\ell'(z)=r$, which further indicates that for $\forall i\in[n]$, there always exists $\mu_{i}\in\mathbb{R}$ solving the Equation~\eqref{eq::prop-necc-1}. Under this construction, it is easy to show that the second condition is satisfied as well. 

Now we only need to show that $\bm{\theta}^{*}$ is local minimum. We first show that $\bm{\theta}^{*}$ is a critical point of the empirical loss function. 
Since for $\forall j\in[M]$,
\begin{align*}
\frac{\partial \hat{L}_{n}(\bm{\theta}^{*})}{\partial a_{j}}&=\sum_{i=1}^{n}\ell'(-y_{i}(\mu_{i}-M\sigma(0)))(-y_{i})\sigma(0)\\
&=\sigma(0)\sum_{i=1}^{n}\frac{\lambda_{i}}{\lambda_{\max}}\ell'(1)(-y_{i})=-\frac{\sigma(0)\ell'(1)}{\lambda_{\max}}\sum_{i=1}^{n}y_{i}\lambda_{i}\\
&=0&&\text{by }\sum_{i:y_{i}=1}\lambda_{i}=\sum_{i:y_{i}=-1}\lambda_{i}\\
\nabla_{\bm{w}_{j}}\hat{L}_{n}(\bm{\theta}^{*})&=\sum_{i=1}^{n}\ell'(-y_{i}(\mu_{i}-M\sigma(0)))(-y_{i})\sigma'(0)x_{i}\\
&=-\sigma'(0)\sum_{i=1}^{n}\frac{\lambda_{i}}{\lambda_{\max}}\ell'(1)y_{i}x_{i}=-\frac{\sigma'(0)\ell'(1)}{\lambda_{\max}}\sum_{i=1}^{n}\lambda_{i}y_{i}x_{i}\\
&=\bm{0}_{d}&&\text{by }\sigma'(0)=0
\end{align*}
and 
$$\frac{\partial \hat{L}_{n}(\bm{\theta}^{*})}{\partial a_{0}}=\sum_{i=1}^{n}\ell'(-y_{i}(\mu_{i}-M\sigma(0)))(-y_{i})=-\frac{\ell'(1)}{\lambda_{\max}}\sum_{i=1}^{n}y_{i}\lambda_{i}=0.$$
In addition, we have stated earlier, if we slightly perturb the parameter $\theta_{k}^{*}$ in the interval $[\theta_{k}^{*}/2,3\theta_{k}^{*}/2]$, the output of the function $f_{D}(x_{i};\bm{\theta}_{D})$ does not change for all $i\in[n]$, then $\bm{\theta}^{*}$ is a critical point. 

Now we show that $\bm{\theta}^{*}$ is local minimum. Consider any perturbation $\Delta a_{1},...,\Delta a_{M}:|\Delta a_{j}|<\frac{1}{2}$ for all $j\in[M]$, $\Delta \bm{w}_{1},...,\Delta \bm{w}_{M}\in\mathbb{R}^{d}$, $\Delta a_{0}\in\mathbb{R}$, $\Delta\theta_{k}:|\Delta \theta_{k}|\le\theta_{k}/2$ for all $k\in[n]$. Define $$\tilde{\bm{\theta}}=(a_{0}^{*}+\Delta a_{0},...,a_{M}^{*}+\Delta a_{M},\bm{w}_{1}^{*}+\Delta \bm{w}_{1},...,\bm{w}_{M}^{*}+\Delta \bm{w}_{M},\theta_{1}^{*}+\Delta\theta_{1}^{*},..., \theta_{d}^{*}+\Delta\theta_{d}^{*}).$$

Then 
\begin{align*}
\sum_{i=1}^{n}\ell(-y_{i}f(x_{i};\tilde{\bm{\theta}}))-\sum_{i=1}^{n}\ell(-y_{i}f(x_{i};\bm{\theta}^{*}))&=\sum_{i=1}^{n}\left[\ell(-y_{i}f(x_{i};\tilde{\bm{\theta}}))-\ell(-y_{i}f(x_{i};\bm{\theta}^{*}))\right]\\
&\ge\sum_{i=1}^{n}\ell'(-y_{i}f(x_{i};\bm{\theta}^{*}))(-y_{i})[f(x_{i};\tilde{\bm{\theta}})-f(x_{i};{\bm{\theta}}^{*})].
\end{align*}
Since for each sample $x_{i}$ in the dataset, 
\begin{align*}
f(x_{i};\tilde{\bm{\theta}})-f(x_{i};{\bm{\theta}}^{*})&=\Delta a_{0}+\sum_{j=1}^{M}(a^{*}_{j}+\Delta a_{j})\sigma(\Delta \bm{w}_{j}^{\top}x_{i})+\mu_{i}-\mu_{i}\\
&=\Delta a_{0}+\sum_{j=1}^{M}(a^{*}_{j}+\Delta a_{j})\sigma(\Delta \bm{w}_{j}^{\top}x_{i}),
\end{align*}
then
 
\begin{align*}
\sum_{i=1}^{n}\ell(-y_{i}f(x_{i};\tilde{\bm{\theta}}))&-\sum_{i=1}^{n}\ell(-y_{i}f(x_{i};\bm{\theta}^{*}))\\
&\ge\sum_{i=1}^{n}\ell'(-y_{i}f(x_{i};\bm{\theta}^{*}))(-y_{i})[f(x_{i};\tilde{\bm{\theta}})-f(x_{i};{\bm{\theta}}^{*})]\\
&=\sum_{i=1}^{n}\ell'(-y_{i}(\mu_{i}-M\sigma(0)))(-y_{i})\left[\sum_{j=1}^{M}(a_{j}^{*}+\Delta a_{j})\sigma\left(\Delta\bm{w}_{j}^{\top}x_{i}\right)+\Delta a_{0}\right]\\
&=\sum_{i=1}^{n}\frac{\lambda_{i}\ell'(1)}{\lambda_{\max}}(-y_{i})\left[\sum_{j=1}^{M}(a_{j}^{*}+\Delta a_{j})\sigma\left(\Delta\bm{w}_{j}^{\top}x_{i}\right)\right]\\
&=\frac{\ell'(1)}{\lambda_{\max}}\sum_{j=1}^{M}-(a_{j}^{*}+\Delta a_{j})\left[\sum_{i=1}^{n}\lambda_{i}y_{i}\left(\Delta\bm{w}_{j}^{\top}x_{i}\right)^{2}\right].
\end{align*}

Since by assumption that the matrix $\sum_{i=1}^{n}\lambda_{i}y_{i}x_{i}x_{i}^{\top}$ is positive semi-definite, then for any $\Delta \bm{w}_{j}^{\top}\in\mathbb{R}^{d}$, 
$$\sum_{i=1}^{n}\lambda_{i}y_{i}\left(\Delta\bm{w}_{j}^{\top}x_{i}\right)^{2}\ge0.$$

In addition, since  $a_{j}^{*}=-1$, $|\Delta a_{j}|<\frac{1}{2}$, then for all $\Delta\bm{w}_{j}\in\mathbb{R}^{d}$,
\begin{align*}
\sum_{i=1}^{n}\ell(-y_{i}f(x_{i};\tilde{\bm{\theta}}))&-\sum_{i=1}^{n}\ell(-y_{i}f(x_{i};\bm{\theta}^{*}))\ge 0.
\end{align*}
Thus, $\bm{\theta}^{*}$ is a local minima of the empirical loss function with $f(x_{i};\bm{\theta}^{*})=\mu_{i}-M\sigma(0)$. Since there exists a $\mu_{i_{\max}}$ such that $y_{i_{\max}}(\mu_{i_{\max}}-M\sigma(0))=1$, then this means that the neural network makes an incorrect prediction on the sample $x_{i_{\max}}$. This indicates that this local minimum has a non-zero training error.

Finally, we present the way we construct the neural network $f_{D}$. Since
\begin{equation*}
f_{D}(x;\bm{\theta}_{D})=f_{D}(x;\theta_{1},...,\theta_{d})=\sum_{i=1}^{n}\mu_{i}\prod_{k=1}^{d}\bm{1}\left\{x^{(k)}\in\left[x_{i}^{(k)}-\theta_{k},x_{i}^{(k)}+\theta_{k}\right]\right\}.
\end{equation*}
Let $\thres$ denote the threshold unit, where $\thres(z)=1$ if $z\ge 0$ and $\thres(z)=0$, otherwise. Therefore, the indicator function can be represented as follows:
$$\bm{1}\left\{x^{(k)}\in\left[x_{i}^{(k)}-\theta_{k},x_{i}^{(k)}+\theta_{k}\right]\right\}=\thres\left(x^{(k)}-x_{i}^{(k)}+\theta_{k}\right)-\thres\left(x^{(k)}-x_{i}^{(k)}-\theta_{k}\right)$$
Therefore, 
\begin{align*}
\prod_{k=1}^{d}&\bm{1}\left\{x^{(k)}\in\left[x_{i}^{(k)}-\theta_{k},x_{i}^{(k)}+\theta_{k}\right]\right\}\\
&=\thres\left(\sum_{k=1}^{d}\left[\thres\left(x^{(k)}-x_{i}^{(k)}+\theta_{k}\right)-\thres\left(x^{(k)}-x_{i}^{(k)}-\theta_{k}\right)\right]-d+\frac{1}{2}\right)\\
\end{align*}
Therefore, we have 
$$f_{D}(x;\bm{\theta}_{D})=\sum_{i=1}^{n}\mu_{i}\thres\left(\sum_{k=1}^{d}\left[\thres\left(x^{(k)}-x_{i}^{(k)}+\theta_{k}\right)-\thres\left(x^{(k)}-x_{i}^{(k)}-\theta_{k}\right)\right]-d+\frac{1}{2}\right).$$
It is very easy to see that this is a two layer network consisted of threshold units. 

Furthermore, we note here that, in the proof shown above, we assume the only parameters in the network $f_{D}$ are $\bm{\theta}_{1},...,\bm{\theta}_{d}$. In fact, we can prove a more general statement where the $f_{D}$ is of the form
$$f_{D}(x;\bm{\theta}_{D})=\sum_{i=1}^{n}\mu_{i}\thres\left(\sum_{k=1}^{d}\left[a_{ik}\thres\left(x^{(k)}+u_{ik}\right)+b_{ik}\thres\left(x^{(k)}+v_{ik}\right)\right]+c_{i}\right),$$
where $a_{ik},b_{ik},u_{ik},v_{ik}, c_{i}$, $i\in[n], k\in[d]$ are all parameters. We can show that the neural network 
$$f_{D}(x;\bm{\theta}_{D})=\sum_{i=1}^{n}\mu_{i}\thres\left(\sum_{k=1}^{d}\left[\thres\left(x^{(k)}-x_{i}^{(k)}+\theta_{k}\right)-\thres\left(x^{(k)}-x_{i}^{(k)}-\theta_{k}\right)\right]-d+\frac{1}{2}\right),$$
denotes a local minimum, since any slight perturbations on parameters $a_{ik},b_{ik},u_{ik},v_{ik}, c_{i}$, $i\in[n], k\in[d]$ do not change the output of the neural network on the samples in the dataset $\mathcal{D}$.
\end{proof}

\end{proof}

\end{appendix}
\end{document}